\pdfoutput=1

\documentclass[11pt]{article}
\usepackage{bbm}

\usepackage{eqnarray,amsmath,amsfonts,amsthm,mathrsfs}
\usepackage{color}
\usepackage{bm}
\usepackage{amssymb}
\usepackage{epsfig}
\usepackage{epsf}
\usepackage{graphicx}
\usepackage{float}
\usepackage{subfigure}
\usepackage{epstopdf}
\usepackage{amsfonts,amsmath,amsthm,amssymb,graphicx,float,fancyhdr,multirow,hyperref}
\usepackage{booktabs,longtable,authblk}
\usepackage{mathrsfs,hhline}
\usepackage{algorithm}
\usepackage{algorithmic}

\usepackage{caption}
\usepackage{fancyhdr}
\usepackage[top=2.5cm, bottom=2.5cm, left=3cm, right=3cm]{geometry}
\usepackage{graphics}
\newtheorem{theorem}{Theorem}
\newtheorem{assumption}{Assumption}

\newtheorem{lemma}{Lemma}

\newtheorem{remark}{Remark}
\newtheorem{example}{Example}

\allowdisplaybreaks[4]

\DeclareMathOperator*{\argmin}{arg\,min}
\numberwithin{equation}{section}

\title{Diffusion-based Semi-supervised Spectral Algorithm for Regression on Manifolds
$^\dag$\footnotetext{\dag~The work described in this paper is supported by the National Natural Science Foundation of China (Grants No.12171039) and Shanghai Science and Technology Program [Project No. 21JC1400600]. Email addresses: 22210180107@m.fudan.edu.cn (W. Xia), jxjiang20@fudan.edu.cn (J. Jiang), leishi@fudan.edu.cn (L. Shi). The corresponding author is Lei Shi.}}

\author[]{Weichun Xia}
\author[]{Jiaxin Jiang}
\author[]{Lei Shi}
\affil[]{School of Mathematical Sciences and Shanghai Key Laboratory for
Contemporary Applied Mathematics, Fudan University, Shanghai, 200433, China.}
\date{}

\begin{document}
	\maketitle
\begin{abstract}

We introduce a novel diffusion-based spectral algorithm to tackle  regression analysis on high-dimensional data, particularly data embedded within lower-dimensional manifolds. Traditional spectral algorithms often fall short in such contexts, primarily due to the reliance on predetermined kernel functions, which inadequately address the complex structures inherent in manifold-based data. By employing graph Laplacian approximation, our method uses the local estimation property of heat kernel, offering an adaptive, data-driven approach to overcome this obstacle. Another distinct advantage of our algorithm lies in its semi-supervised learning framework, enabling it to fully use the additional unlabeled data. This ability enhances the performance by allowing the algorithm to dig the spectrum and curvature of the data manifold, providing a more comprehensive understanding of the dataset. Moreover, our algorithm performs in an entirely data-driven manner, operating directly within the intrinsic manifold structure of the data, without requiring any predefined manifold information. We provide a convergence analysis of our algorithm. Our findings reveal that the algorithm achieves a convergence rate that depends solely on the intrinsic dimension of the underlying manifold, thereby avoiding the curse of dimensionality associated with the higher ambient dimension.

\end{abstract}
	
	
\section{Introduction}\label{section: introduction}

High-dimensional data are increasingly prevalent across various fields within modern machine learning, presenting unique challenges. Specifically, the obstacle people often encounter is the small sample size compared to the high dimensionality of data and the ``curse of dimensionality", a concept introduced by \cite{Bellman+1961}, resulting in poor performance. This complex data structure necessitates novel approaches and theoretical proceedings in statistical analysis and inference. 
Recent literature \cite{frank2006three,carlsson2008local,chikuse2012statistics} has illuminated that, in many instances, the practical input data tends to reside within a lower-dimensional, nonlinear manifold. This suggests that the input data space $ X $, can be conceptualized as a compact and connected Riemannian manifold $ \mathcal{M} $ embedding into ambient Euclidean space $ \mathbb{R}^D $, where the manifold dimension $ d $  is substantially less than $ D $. This notion is in harmony with the widely acknowledged low-dimensional manifold hypothesis, which has gained significant attention in {contemporary research, e.g. \cite{Calandra2014ManifoldGP,guhaniyogi2016compressed,hamm2021adaptive}.} This geometric perspective paves the way for addressing the intrinsic complexities posed by high-dimensional data.

In the domain of data analysis and pattern recognition, regression emerges as a fundamental problem discerning relationships between input (explanatory) and output (response) variables.
In this paper, we primarily study the non-parametric regression problem \eqref{definition of regression problem}. The input variable $ x $ originates from a lower-dimensional manifold $ \mathcal{M} $, the output variable $ y \in \mathbb{R} $ is real, and they are governed by an unknown distribution $ \mathcal{P} $ over the product space $ \mathcal{M}\times\mathbb{R} $. Given a dataset $ D=\{(x_i,y_i)\}_{i=1}^n $ that is independently and identically distributed according to $ \mathcal{P} $, our goal is to identify or approximate the optimal estimator $ f^*:\mathcal{M}\to\mathbb{R} $ that minimizes the mean-squared error, in a specified function space $ \mathcal{F} $, e.g. the space of measurable functions:
\begin{equation}\label{definition of regression problem}
	f^* = \argmin\limits_{f\in\mathcal{F}}\int_{\mathcal{M}\times\mathbb{R}} (y-f(x))^2\text{d}\mathcal{P}(x,y).
\end{equation}
Referencing the classical literature \cite{gyorfi2002distribution}, with the knowledge of $ \mathcal{P} $, an explicit calculation finds the optimal regression function $ f^* $ in the whole measurable space as
\begin{equation*}
	f^*(x)=\int_{\mathbb{R}}y\;\text{d}\mathcal{P}(y|x),
\end{equation*}
where $ \mathcal{P}(\cdot|x) $ denotes the conditional distribution of $ \mathcal{P} $ over $ \mathbb{R} $ given $ x $.

The regression problem \eqref{definition of regression problem} has been extensively explored, leading to diverse methodologies. For an in-depth exploration, one can refer to {comprehensive monographs \cite{cucker2002mathematical,gyorfi2002distribution}.}
Among the numerous strategies, spectral algorithms, as suggested by \cite{bauer2007regularization}, stand out due to their effective implementation, which involves applying a filter function to the spectra of a finite-dimensional kernel matrix. The efficacy of these algorithms is largely ascribed to the regularization achieved via a carefully selected set of filter functions. Moreover, the rich choice of regularization function allows spectral algorithms to conclude a broad range of popular regression algorithms, which is specified in Subsection \ref{subsection: regularization family and spectral algorithm}. 
This adaptability is particularly noteworthy, emphasizing the significance of spectral algorithms in regression analysis. By leveraging the strengths of different regularization strategies, spectral algorithms offer an important framework for tackling complex regression problems. For more detailed analyses and discussions of spectral algorithms, we refer to {recent studies such as \cite{blanchard2018optimal,lin2020optimal,celisse2021analyzing}.}

Although the spectral algorithm has gained widespread application, its efficacy is occasionally hampered by several issues, among which the selection of the kernel function stands out as a particularly daunting challenge. As detailed in Subsection \ref{subsection: regularization family and spectral algorithm}, the computation of the spectral algorithm estimator is fundamentally dependent on the kernel matrix $ (K(x_i,x_j))_{i,j=1}^n $, where $ K(x,x') $ is the kernel function defined over the input manifold. Traditional approaches often employ kernels such as radial basis kernels and polynomial kernels, which we refer to \cite{scholkopf2002learning,steinwart2008support} for more details. Nevertheless, these conventional kernels tend to underperform in scenarios involving manifold input data, primarily because they overlook the intrinsic structure of data. 
To address the complexities presented by the manifold structure of data, recent studies have shifted towards utilizing manifold-based bandlimited kernels and heat kernels, as seen in \cite{mcrae2020sample,xia2024spectral}. These kernels are designed with the input manifold's intrinsic structure in mind, notably leveraging the spectrum of the Laplace-Beltrami operator and curvatures of the data manifold. This focus on the manifold's inherent features allows for remarkable theoretical convergence properties, characterized by a rapid convergence rate that depends solely on the manifold's dimension $ d $, and remains unaffected by the higher ambient dimension $ D $. This presents a substantial improvement over traditional kernels, which often suffer from the curse of dimensionality. The theoretical benefits, including the enhanced convergence rate, of these manifold-based kernels, have also been validated in \cite{mcrae2020sample,xia2024spectral}, highlighting their efficiency in advancing the spectral algorithm's applicability in high-dimensional data analysis.

While manifold-based kernels offer theoretical advantages, their practical application is hindered by the challenge of calculating these functions, which is requisite in the algorithms. Our work is inspired by graph Laplacian methods, well-regarded in the literature of manifold learning ({see e.g. \cite{belkin2003laplacian,coifman2006diffusion,von2007tutorial,melacci2011laplacian}}) for their effectiveness in capturing the local geometric and topological structures of input manifolds. We adopt these methods to estimate the manifold heat kernel, introducing a data-driven, adaptive kernel matrix as an alternative to the difficult-to-calculate manifold-based kernel matrix. This approach not only effectively and efficiently overcomes the implementation hurdle but also preserves the superior qualities of manifold-based kernels by leveraging the local approximation property of the heat kernel on data manifold. To our knowledge, the specific application of graph Laplacian approximation for manifold heat kernel estimation is still relatively uncharted, with few literature venturing into this domain \cite{dunson2019diffusion,dunson2021spectral}. Our paper integrates this estimation method with spectral algorithms to introduce a novel diffusion-based spectral algorithm. See Algorithm \ref{Alg: diffusion-based spectral algorithm}. This development expands the capabilities of spectral algorithms in processing complex manifold data structures.

One of the superioritys of our algorithm is its semi-supervised learning nature. In traditional supervised learning scenarios, labeling data necessitates considerable effort, such as transcribing for speech recognition tasks, or might require specialized expertise, like determining the health status of a brain scan. Conversely, in many fields, including computer vision, natural language processing, and speech recognition, acquiring unlabeled data is often much easier and can be done on a larger scale with less effort. Carrying out the principles of classical semi-supervised learning, our algorithm enables the use of additional unlabeled manifold samples. 
These samples inherently reflect the geometric and topological structure of the input manifold, bearing a potential enhancement in the performance of the spectral algorithm beyond what is achievable with solely labeled data. This ability to use both labeled and unlabeled data for improving learning outcomes presents a substantial enhancement over traditional spectral algorithms, which rely exclusively on labeled data. Additionally, unlike methods that pre-process data through dimension reduction to simplify regression tasks, our algorithm tackles the complexity head-on by operating directly within the intrinsic manifold structure of data. This approach offers a direct solution to the intricate problem of regression on manifolds, requiring no predefined manifold information. Our algorithm operates entirely on the available data, sticking to a fully data-driven methodology. These characteristics showcase our algorithm's capability to navigate the manifold's inherent complexities efficiently and effectively.

In this paper, we further provide a rigorous convergence rate analysis of our novel algorithm. Detailed theoretical results are elaborated in Subsection \ref{subsection: convergence analysis}. Drawing from our main Theorem \ref{Thm: convergence analysis}, we demonstrate that, for a selection of filter functions commonly used in traditional spectral algorithms, including kernel ridge regression and kernel principal component regression, our diffusion-based spectral algorithm achieves a convergence rate of{
\begin{equation*}
	\sup_{1 \leq j \leq N}|\tilde{f}_{D,\lambda}(j)-f^*(x_j)|=\tilde{O}\left( m^{-\frac{1}{2}+\alpha}+m^{\frac{3}{2r+1}+1}\left(\frac{1}{K}+\left(\frac{\log N}{N}\right)^\frac{1}{2d+8}\right)\right).
\end{equation*}}
In this expression, the notation $ \tilde{O}(\cdot) $ omits lower-order terms compared to the primary term. The symbol $ \tilde{f}_{D,\lambda} $ denotes the estimator derived from our diffusion-based spectral algorithm, $ m $ represents the number of labeled inputs, $ N $ signifies the total number of inputs (both labeled and unlabeled), $ m\ll K<N $ refers to a truncation hyperparameter introduced in subsection \ref{subsection: graph Laplacian and heat kernel estimation}, {$r>1$ is a constant introduced in Assumption \ref{Assumption: source condition on regression function}}, and $ 0<\alpha<1/2 $ can be chosen arbitrarily small. This pointwise convergence rate is applicable across all $ 1\leq j\leq N $, particularly within the unlabeled dataset, which is solely dependent on the intrinsic dimension $ d $ of the input manifold, irrespective of the considerably higher ambient dimension $ D $. This property reflects the benefits observed in classical spectral algorithms that employ manifold-based bandlimited kernels or heat kernels, further demonstrating the effectiveness and efficiency of our algorithm in dealing with high-dimensional data through a semi-supervised learning framework.

The remainder of this paper is organized as follows. In Section \ref{section:framework and notation}, we introduce key concepts and notations in our paper. In Section \ref{section: main results}, we present the diffusion-based spectral algorithm with a comprehensive convergence analysis. In Section \ref{section: numerical experiments}, we illustrate the validity of the algorithm through numerical experiments. In Section \ref{section: proof of convergence result}, we rigorously prove our main theoretical theorem. We provide additional details and analyses that complement the main text in the Appendix.

\section{Notations and Theoretical Backgrounds}\label{section:framework and notation}

Before delving into a detailed discussion, we introduce several notations. As mentioned in Section \ref{section: introduction}, our input data space is defined as a Riemannian manifold $ \mathcal{M} $ with dimension $ d $, and the output space as the real line $ \mathbb{R} $. We consider an unknown distribution $ \mathcal{P} $ on the product space $ \mathcal{M}\times\mathbb{R} $, denoting its marginal distribution on $ \mathcal{M} $ as $ \nu $, and the conditional distribution given $ x\in\mathcal{M} $ on $ \mathbb{R} $ as $ \mathcal{P}(\cdot|x) $. Furthermore, we assume this marginal distribution $\nu$ to be uniform, which means that it differs from the Riemannian volume measure $\mu_g$ (defined in \eqref{Definition of Riemannian volume measure} below) only by a constant factor $ p = \frac{1}{\text{vol}\mathcal{M}} $. For simplicity, in subsequent discussions, we will use the notation $L^2(\mu_g)$ and $ L^2(\nu) $ to refer to the $L^2$-space with respect to the Riemannian volume measure $\mu_g$ and the uniform measure $\nu$ on $ \mathcal{M} $  respectively.

\subsection{Manifold and Heat Kernel}\label{subsection: manifold and heat kernel}

In this subsection, we introduce the concepts of Riemannian manifolds and the heat kernel. We consider a compact, connected Riemannian manifold $ \mathcal{M} $ of dimension $ d $, embedded into a higher-dimensional ambient Euclidean space $ \mathbb{R}^D $, where $ d\ll D $. The fundamental geometry of this manifold, the Riemannian metric $ g $ and the Levi-Civita connection $ \nabla $, are derived from the manifold's embedding in $ \mathbb{R}^D $. 
The Riemannian metric $ g $ gives rise to a unique Radon measure $ \mu_g $ on $ \mathcal{M} $, commonly referred to as the Riemannian volume measure. This measure is defined as
\begin{equation}\label{Definition of Riemannian volume measure}
	\int_{\mathcal{M}}fd{\mu_g}=\sum_j\int_{\varphi_j(U_j)}\left(\psi_jf\sqrt{{\rm det}(g_{ij})}\right)\circ\varphi_j^{-1}dx,
\end{equation} 
where $ dx $ denotes the Lebesgue measure on $ \mathbb{R}^d $, $ \{U_j,\varphi_j\} $ represents a smooth atlas of $ \mathcal{M} $, and $ \{\psi_j\} $ signifies the partition of unity subordinate to the atlas covering $ \{U_j\} $. If $\mathcal{M}$ is additionally orientable, we can locally define a smooth, non-zero volume form on $\mathcal{M}$ as $ \omega=\sqrt{{\rm det}(g_{ij})}dx^1\land\cdots\land dx^m $, given a positively oriented orthonormal basis of $ T_u(\mathcal{M}) $. This allows the Riemannian volume measure $ \mu_g $ to be elegantly represented as $ \int_{\mathcal{M}}fd_{\mu_g}=\int_{\mathcal{M}}f\omega $. For a more thorough exploration of integration on Riemannian manifolds, one may refer to standard literature such as  \cite{chavel1984eigenvalues,sakai1996riemannian,petersen2006riemannian}.

We introduce the Laplace-Beltrami operator on a Riemannian manifold $\mathcal{M}$. It is often abbreviated as the Laplacian, a concept fundamental to the study of differential geometry with applications in manifold learning. This operator is determined by the Levi-Civita connection $ \nabla $ and its Riemannian metric $ g $, as detailed in the textbooks \cite{petersen2006riemannian,lee2018introduction}. For any smooth function $ f\in C^\infty(\mathcal{M}) $, the Laplacian $ \Delta f $ is defined as the negative divergence of the gradient vector field of $ f $, 
\begin{equation*}
	\Delta f = -{\rm div}(\nabla f).
\end{equation*}
And in local coordinates, this operator has an explicit expression as
\begin{equation*}
	\Delta f = -\frac{1}{\sqrt{{\rm det}g}}\partial_i\left(g^{ij}\sqrt{{\rm det}g}\partial_jf\right),
\end{equation*}
where we have adopted the Einstein summation convention. Choosing a negative sign in definition ensures that the Laplacian remains a positive operator, aligning with the conventions of mathematical physics and differential geometry. For the classical case where the manifold $\mathcal{M}$ is the Euclidean space $ \mathbb{R}^d $ , the Laplacian corresponds to the familiar second-order derivative operator $-\sum_{i=1}^d\partial^2_i$. Initially, the Laplacian acts as a self-adjoint, positive, and semi-definite operator on the space of smooth functions defined on $ \mathcal{M} $. However, it can be uniquely extended to a self-adjoint, positive, and semi-definite operator on the space $ L^2(\mu_g) $ through Friedrich's extension theorem, a process thoroughly examined in the analytical studies by \cite{strichartz1983analysis,reed2003methods}. This extension of the Laplacian preserves its fundamental properties and enables a broader application in the analysis of functions on Riemannian manifold, and we also refer to it as Laplacian.

Consider the heat equation on a Riemannian manifold $\mathcal{M}$, formulated as:
\begin{equation*}
	\begin{cases}
	\frac{\partial u}{\partial t}+\Delta u=0,\quad t>0\\
	u|_{t=0}=f
	\end{cases}
\end{equation*}
with a smooth function $ f:\mathcal{M}\to\mathbb{R} $. Given that $ \mathcal{M} $ is both smooth and compact, this equation guarantees a unique smooth solution $ u=e^{-t\Delta}f $. Drawing on classical literature such as \cite{davies1989heat,grigoryan2009heat}, for the heat semi-group $ e^{-t\Delta} $, there exists a unique function $ H:\mathcal{M}\times\mathcal{M}\times(0,\infty)\to\mathbb{R} $ such that:
\begin{equation}\label{definition of heat kernel}
	\left(e^{-t\Delta}f\right)(u)=\int_{\mathcal{M}}H(u,u',t)f(u')d\mu_g(u').
\end{equation}
Here, $ H_t(u,u')=H(u,u',t) $ is called the heat kernel, characterized by its diffusion time $ t $. Specifically, for $ \mathcal{M}=\mathbb{R}^d $, $ H_t $ reduces to the well-known Gaussian kernel $ H_t(x,x')=e^{-\frac{|x-x'|^2}{4t}}/(4\pi t)^{\frac{d}{2}} $. Gaussian-type functions could be used to approximate the manifold heat kernel, especially for closely positioned $ u,u' $ and small $ t $, exhibiting an approximation error
\begin{equation}\label{local approxiamtion property of heat kernel}
	\left|H_t(u,u')-\frac{1}{(4\pi t)^{\frac{d}{2}}}e^{-\frac{d_\mathcal{M}(u,u')^2}{4t}}(u_0+tu_1)\right|\leq O(t),
\end{equation}
where $ u_0,u_1 $ are smooth functions locally close to $ 1 $, and $ d_\mathcal{M}(u,u') $ denotes the geodesic distance between $ u $ and $ u' $ on $ \mathcal{M} $. This local approximation property is crucial for estimating the Laplacian's eigensystem using graph Laplacian techniques, which will be further explored in Subsection \ref{subsection: graph Laplacian and heat kernel estimation}. A classical assumption about the heat kernel is its boundedness, namely
\begin{equation}\label{boundness condition on Heat Kernel}
	\sup\limits_{u\in\mathcal{M}}H_t(u,u)\leq\kappa^2
\end{equation}
with some constant $\kappa\geq1$. Given the compactness of $\mathcal{M}$ and inherently bounded sectional curvature and Ricci curvature, this boundedness condition is naturally satisfied by the heat kernel comparison theorem. Further details can be found in \cite{hsu2002stochastic}.

Given that the Laplacian $\Delta$ is a self-adjoint, positive, and semi-definite operator on $L^2(\mu_g)$, the celebrated Sturm-Liouville theorem enables us to identify an orthonormal basis $\{\psi_k\}_{k\in\mathbb{N}}$ for $L^2(\mu_g)$. Each $\psi_k$ in this basis is smooth on $ \mathcal{M} $ and adheres to 
\begin{equation}\label{eigen-system of Laplacian}
	\Delta \psi_k=\mu_k\psi_k
\end{equation}
for all $k$. The sequence of eigenvalues 
\begin{equation*}
	0=\mu_1\leq\mu_2\leq\dots\leq\mu_k\to+\infty
\end{equation*}
corresponding to $ \{\psi_k\} $, are aligned in an unbounded increase. The rate at which these eigenvalues grow directly relates to the manifold's dimension. This relationship is elaborated in the following Lemma \ref{Lemma: estimation of eigen-system for laplacian}. The following literature \cite{hormander1968spectral,li1983schrodinger,chavel1984eigenvalues} offers a deeper exploration of these foundational results.

\begin{lemma}\label{Lemma: estimation of eigen-system for laplacian}
	Suppose that $ \mathcal{M} $ is a compact, connected Riemannian manifold with dimension $ d $. Then, for the eigensystem of Laplacian $ \Delta $ on $ \mathcal{M} $ as in \eqref{eigen-system of Laplacian}, for any $ k\in\mathbb{N} $, we have the following estimations:
	\begin{equation}\label{population bounds for eigenvalue of laplacian}
		C_{low}k^\frac{2}{d}\leq\mu_k\leq C_{up}k^\frac{2}{d},
	\end{equation}
	\begin{equation}\label{population bounds for eigenfunction of laplacian}
		\|\psi_k\|_{L^\infty(\mathcal{M})}\leq D_1\mu_k^\frac{d-1}{4},
	\end{equation}
	where $C_{low}$, $C_{up}$, and $D_1$ are absolute constants only rely on $ \mathcal{M} $. 
\end{lemma}

\noindent
Building on the eigensystem of the Laplacian as previously discussed in \eqref{eigen-system of Laplacian}, the heat semi-group $ e^{-t\Delta} $ can be characterized through its eigen-decomposition:
\begin{equation}\label{spectral decomposition of heat semi-group}
	e^{-t\Delta}\psi_k=e^{-t\mu_k}\psi_k.
\end{equation}
This directly leads to the spectral decomposition of the heat kernel $ H_t $, which, according to Mercer's theorem (see e.g. \cite{wainwright2019high}), is expressed as
\begin{equation}\label{spectral decomposition of heat kernel}
	H_t(u,u')=\sum_{k=1}^\infty e^{-t\mu_k}\psi_k(u)\psi_k(u'),
\end{equation}
ensuring a uniformly and absolutely convergent series.

\subsection{Diffusion Space and Integral Operator}\label{subsection: diffusion space and integral operators}

In this subsection, we clarify the concept of diffusion space on a Riemannian manifold $ \mathcal{M} $ and the related integral operators. Given any fix $t>0$, the diffusion space, denoted as $\mathcal{H}_t$ (with diffusion time $ t $), is defined as
\begin{equation}\label{definition of diffusion space}
	\mathcal{H}_t=e^{-\frac{t}{2}\Delta}L^2(\mu_g).
\end{equation}
$\mathcal{H}_t$ is structured as a Hilbert space equipped with an inner product
\begin{equation}\label{definition of inner product on diffusion space}
	\langle f,g\rangle_{\mathcal{H}_t}=\langle e^{\frac{t}{2}\Delta}f, e^{\frac{t}{2}\Delta}g\rangle_{L^2(\mu_g)}.
\end{equation}
Further, endowed with this inner product \eqref{definition of inner product on diffusion space}, $\mathcal{H}_t$ constitutes a reproducing kernel Hilbert space (RKHS), with the heat kernel $H_t(u,u')$ serving as its reproducing kernel. This RKHS framework facilitates various significant attributes, as detailed in \cite{de2021reproducing}. A key feature of the diffusion space is the embedding property:
\begin{equation}\label{embedding property of diffusion space}
	\mathcal{H}_t\hookrightarrow\mathcal{H}_{t'}\hookrightarrow W^s(\mathcal{M}),
\end{equation}
for all  $0<t'<t$ and $s>0$, with $W^s(\mathcal{M})$ representing the $s$-order Sobolev space on $\mathcal{M}$ under the $L^2(\mu_g)$-norm. Given the compact embedding of the Sobolev space $W^s(\mathcal{M})$ into $C(\mathcal{M})$ for $s>\frac{d}{2}$, as noted by \cite{hebey2000nonlinear,triebel2010theory}, and considering the compactness of $\mathcal{M}$, it follows that any diffusion space $\mathcal{H}_t$ can be compactly embedded into $L^\infty(\mathcal{M})$. Consequently, for any $t>0$, there exists a constant $A$ (dependent on $ t $) such that
\begin{equation*}
	\|i:\mathcal{H}_t\hookrightarrow L^\infty\|_{op}\leq A,
\end{equation*}
where $\|\cdot\|_{op}$ denotes the operator norm. This embedding property, as explored in recent literature \cite{fischer2020sobolev,zhang2024optimality,xia2024spectral}, plays a vital role in the convergence analysis.

Considering the inclusion map $I_\nu$ from $\mathcal{H}_t$ to $L^2(\nu)$, we introduce its adjoint operator, $I_\nu^*$, which is an integral operator mapping $L^2(\nu)$ back to $\mathcal{H}_t$:
\begin{equation}\label{definition of integral operator I_nu}
	(I_\nu^*f)(x)=\int_{\mathcal{M}}H_t(x,z)f(z)d\nu(z).
\end{equation}
A straightforward examination reveals that both $I_\nu$ and $I_\nu^*$ qualify as Hilbert-Schmidt operators, hence they are inherently compact. The Hilbert-Schmidt norm of these operators meets the following relation
\begin{equation*}
	\|I_\nu^*\|_{\mathscr{L}^2(L^2(\nu),\mathcal{H}_t)}=
	\|I_\nu\|_{\mathscr{L}^2(\mathcal{H}_t,L^2(\nu))}=
	\|H_t\|_{L^2(\nu)}=\left(\int_\mathcal{M}H_t(x,x)d\nu(x)\right)^\frac{1}{2}
	\leq\kappa.
\end{equation*}
Further exploration of the integral operator $I_\nu^*$ involves composition with the inclusion map $I_\nu$:
\begin{equation}\label{definition of integral operator L_nu}
	L_\nu=I_\nu I_\nu^*:L^2(\nu)\to L^2(\nu),
	\end{equation}
	\begin{equation}\label{definition of integral operator T_nu}
	T_\nu=I_\nu^* I_\nu:\mathcal{H}_t\to \mathcal{H}_t.
\end{equation}
The fact that these operators $L_\nu$ and $T_\nu$ are self-adjoint, positive-definite, and fall within the trace class, indicates that they are also Hilbert-Schmidt and compact by nature. The trace norm of these operators are
\begin{equation*}
	\|T_\nu\|_{\mathscr{L}^1(\mathcal{H}_t)}=
	\|L_\nu\|_{\mathscr{L}^1(L^2(\nu))}=
	\|I_\nu^*\|^2_{\mathscr{L}^2(L^2(\nu),\mathcal{H}_t)}=
	\|I_\nu\|^2_{\mathscr{L}^2(\mathcal{H}_t,L^2(\nu))}.
\end{equation*}

In conclusion, $L_\nu$ represents the integral operator associated with the heat kernel $H_t$ on $ L^2(\nu) $, with respect to the uniform distribution $\nu$. Let us compare it with the heat kernel integral operator on $ L^2(\mu_g) $ with respect to the Riemannian volume measure $\mu_g$, as in \eqref{definition of heat kernel}, which is the heat semi-group $e^{-t\Delta}$. Given that $\nu$ and $\mu_g$ differ only by a constant factor $ p $, the eigen-decomposition of $L_\nu$ can be derived from \eqref{spectral decomposition of heat semi-group} as
\begin{equation*}
	L_\nu\psi_k=pe^{-t\mu_k}\psi_k.
\end{equation*}

Given that $ \{\psi_k\}_{k\in\mathbb{N}} $ constitutes an orthonormal basis of $ L^2(\mu_g) $, it follows that $ \{\varphi_k = p^{-\frac{1}{2}}\psi_k\}_{k\in\mathbb{N}} $ forms an orthonormal basis of $ L^2(\nu) $. Furthermore, based on the inner product definition \eqref{definition of inner product on diffusion space} on $ \mathcal{H}_t $, $ \{p^\frac{1}{2}e^{-\frac{t\mu_k}{2}}\varphi_k\}_{k\in\mathbb{N}} $ constitutes an orthonormal basis of $ \mathcal{H}_t $.
Since both $L_\nu$ and $T_\nu$ are self-adjoint and compact operators, the classical spectral theorem (see e.g. \cite{reed2003methods}) allows us to write down their spectral decompositions:
\begin{equation}\label{spectral decomposition of L_nu}
	L_\nu=\sum_{k=1}^\infty pe^{-t\mu_k}\langle \cdot, \varphi_k\rangle_{L^2(\nu)} \varphi_k,
\end{equation}
\begin{equation}\label{spectral decomposition of T_nu}
	T_\nu=\sum_{k=1}^\infty pe^{-t\mu_k}\langle \cdot, p^\frac{1}{2}e^{-\frac{t\mu_k}{2}}\varphi_k\rangle_{\mathcal{H}_t} \cdot p^\frac{1}{2}e^{-\frac{t\mu_k}{2}}\varphi_k.
\end{equation}
Additionally, the spectral decomposition for $I_\nu^*$ is given by:
\begin{equation}\label{spectral decomposition of I_nu}
	I_\nu^*=\sum_{k=1}^\infty p^\frac{1}{2}e^{-\frac{t\mu_k}{2}}\langle \cdot, \varphi_k\rangle_{L^2(\nu)} \cdot p^\frac{1}{2}e^{-\frac{t\mu_k}{2}}\varphi_k.
\end{equation}

Building on the integral operator $L_\nu$, we explore the concept of $\alpha$-power spaces within the diffusion space framework. These spaces, denoted as $ \mathcal{H}_t^{\alpha} $ for any $ \alpha>0 $, are defined as the range of a power of $ L_\nu $, i.e., $\mathcal{H}_t^{\alpha}=Ran(L_\nu^\frac{\alpha}{2})$. Leveraging the spectral decomposition of $L_\nu$, we obtain an explicit representation for $\mathcal{H}_t^{\alpha}$ as
\begin{equation}\label{definition of alpha-power space}
	\mathcal{H}_t^{\alpha}=\left\{
	\sum_{k=1}^\infty a_k p^{\frac{\alpha}{2}}e^{-\frac{\alpha\mu_k}{2}t}\varphi_k:\{a_k\}_{k=1}^\infty\in \ell^2(\mathbb{N})
\right\}.
\end{equation}
Furthermore, $\mathcal{H}_t^{\alpha}$ is equipped with an $\alpha$-power norm, defined as
\begin{equation}\label{definition of alpha-power norm}
	\left\|\sum_{k=1}^\infty a_k p^{\frac{\alpha}{2}}e^{-\frac{\alpha\mu_k}{2}t}f_k\right\|_{\alpha}:=\left\|\sum_{k=1}^\infty a_k p^{\frac{\alpha}{2}}e^{-\frac{\alpha\mu_k}{2}t}f_k\right\|_{\mathcal{H}_t^{\alpha}}
	=\|(a_k)\|_{\ell^2(\mathbb{N})}
	=\left(\sum_{k=1}^\infty a_k^2\right)^\frac{1}{2}.
\end{equation}
In discussions that follow, this $\alpha$-power norm will be referred to using the notation $\|\cdot\|_{\alpha}$ for ease of reference. As a result of this formulation, we have
\begin{equation}\nonumber
	\left\|L_\nu^{\frac{\alpha}{2}}(f)\right\|_\alpha=\|f\|_{L^2(\nu)}.
\end{equation}
This framework establishes that $ L_\nu^{\frac{\alpha}{2}} $ acts as an isometry from $ L^2(\nu) $ to $ \mathcal{H}_t^{\alpha} $, thereby confirming $\mathcal{H}_t^{\alpha}$ as a Hilbert space. This space is characterized by an orthonormal basis $\{p^{\frac{\alpha}{2}}e^{-\frac{\alpha\mu_k}{2}t}\varphi_k\}_{k\in\mathbb{N}}$. Notably, $ \mathcal{H}_t^{\alpha} $ differs with $ \mathcal{H}_{\alpha t} $ by a multiplication of $p^{\frac{1}{2}}$, and when $\alpha=0$, it equals $ L^2(\nu) $.

\subsection{Regularization Family and Spectral Algorithm}\label{subsection: regularization family and spectral algorithm}

In this subsection, we introduce the regularization family represented by different filter functions and its application in spectral algorithm estimators. Consider a labeled dataset $D^l=\{(x_i,y_i)\}_{i=1}^m$, independently and identically distributed from a distribution $\mathcal{P}$ on $ \mathcal{M}\times\mathbb{R} $. Our objective is to identify the optimal estimator $ \hat{f}\in\mathcal{H}_t $ that minimizes the empirical risk
\begin{equation}\label{definition of empirical risk}
	\hat{f}=\argmin\limits_{f\in\mathcal{H}_t}\frac{1}{m}\sum_{i=1}^m\left(y_i-f(x_i)\right)^2.
\end{equation}
The key to this regression problem is a finite rank sampling operator $H_{t,x}$ , which maps $ y $ to $ yH_t(x,\cdot) $ from $\mathbb{R}$ to $\mathcal{H}_t$. 
The adjoint of this operator, $ H_{t,x}^* $, is indeed the evaluation operator from $ \mathcal{H}_t $ to $ \mathbb{R} $ at a specific point $ x\in\mathcal{M} $, defined as $H_{t,x}^*:f\mapsto f(x)$.
Additionally, we introduce the sample covariance operator $T_\delta:\mathcal{H}_t\to\mathcal{H}_t$:
\begin{equation}\label{definition of empirical integral operator T_delta}
	T_\delta = \frac{1}{m}\sum_{i=1}^mH_{t,x_i}H_{t,x_i}^*,
\end{equation}
and define the sample basis function as
\begin{equation}\label{defintion of g_D}
	g_D=\frac{1}{m}\sum_{i=1}^my_iH_t(x_i,\cdot)\in\mathcal{H}_t.
\end{equation}
With these notions at hand, the regression problem \eqref{definition of empirical risk} turns into the equation
\begin{equation}\label{direct solution of empirical risk}
	T_\delta\hat{f}=g_D.
\end{equation}
We refer to \cite{wainwright2019high} for more details.

Due to the potential non-invertibility of the operator $T_\delta$, we turn to a regularization approach to tackle problem \eqref{direct solution of empirical risk}. A set of functions, $\{g_\lambda:\mathbb{R^+\to\mathbb{R}^+}\}_{\lambda>0}$, known as a regularization family or filter function, are defined to satisfy specific conditions
%
\begin{equation}\label{property of regularization family}
	\begin{aligned}
	&\sup\limits_{0<t\leq\kappa^2}|tg_\lambda(t)|<1,\\
	&\sup\limits_{0<t\leq\kappa^2}|1-tg_\lambda(t)|<1,\\
	&\sup\limits_{0<t\leq\kappa^2}|g_\lambda(t)|<\lambda^{-1}.\\
	\end{aligned}
\end{equation}
These functions essentially aim to mimic the behavior of $ t\to\frac{1}{t} $ while maintaining controlled behavior near zero. This approach ensures that they remain bounded by $ \lambda^{-1} $, and effectively deal with the issue of non-invertibility in $ T_\delta $. For a chosen regularization family $g_\lambda$, we further define its qualification $\xi$, which indicates the adaptability of the regularization family, as the supremum of the following set:
\begin{equation}\label{qualification of regularization family}
	\xi = \sup \left\{\alpha>0:\sup\limits_{0<s\leq\kappa^2}|1-sg_\lambda(s)|s^\alpha<\lambda^\alpha\right\}.
\end{equation}
Leveraging the regularization family defined above, we propose the classical spectral algorithm estimator as a regularized solution to the regression problem \eqref{definition of empirical risk} as 
\begin{equation}\label{definition of spectal alogrithm estimator f_D,lambda}
	f_{D,\lambda} = g_\lambda(T_\delta)g_D.
\end{equation}
In this setup, $g_\lambda(T_\delta)$ denotes the action of the function $g_\lambda$ on $T_\delta$ through functional calculus, a concept elaborated in literature such as \cite{reed1980methods}. This methodology provides a feasible solution to the regression problem by circumventing the issue of non-invertibility and introduces a framework for applying spectral methods effectively in the context of empirical risk minimization.

Before delving deeper into the discussion of the spectral algorithm estimator, let us briefly illustrate a few examples. These instances serve as a primer to the broader and more detailed discussion in seminal works on the topic, including \cite{rosasco2005spectral,bauer2007regularization,gerfo2008spectral}.

\begin{example}
	\rm{\textbf{(Kernel ridge regression)}}
	We choose the following regularization family
	\begin{equation*}
		g_\lambda(t)=\frac{1}{\lambda+t},
	\end{equation*}
	which satisfies \eqref{qualification of regularization family} with $ \xi=1 $. In this situation, the spectral algorithm coincides with the kernel ridge regression (or Tikhonov regularization), which corresponds to the classical Hilbert-norm regularization of the original regression problem \eqref{definition of empirical risk} as 
	\begin{equation*}
		f_{D,\lambda}=\argmin\limits_{f\in\mathcal{H}_t}\left\{\frac{1}{m}\sum_{i=1}^m\left(y_i-f(x_i)\right)^2+\lambda\|f\|_{\mathcal{H}_t}^2\right\}.
	\end{equation*}
\end{example}

\begin{example}
	\rm{\textbf{(Kernel principal component regularization)}}
	We choose the following regularization family
	\begin{equation*}
		g_\lambda(t)=\frac{1}{t}\mathbbm{1}_{\{t\geq\lambda\}},
	\end{equation*}
	which satisfies \eqref{qualification of regularization family} with $ \xi=\infty $. In this situation, the spectral algorithm coincides with kernel principal component regularization (or spectral cut-off).
\end{example}

\begin{example}
	\rm{\textbf{(Gradient flow)}}
	We choose the following regularization family
	\begin{equation*}
		g_\lambda(t)=\frac{1-e^{-\frac{t}{\lambda}}}{t},
	\end{equation*}
	which satisfies \eqref{qualification of regularization family} with $ \xi=\infty $. In this situation, the spectral algorithm coincides with gradient flow.
\end{example}

It is noteworthy that the operator $T_\delta$, as defined in \eqref{definition of empirical integral operator T_delta}, aligns with the integral operator defined in \eqref{definition of integral operator T_nu} about the empirical marginal distribution $\delta=\frac{1}{m}\sum_{i=1}^m\delta_{x_i}$ on $ \mathcal{M} $. 
From this perspective we call $ T_\delta $ the empirical integral operator. By definition, $T_\delta$ is inherently self-adjoint, positive semi-definite, and compact on $\mathcal{H}_t$. These properties facilitate an expression of $T_\delta$'s eigensystem as
\begin{equation}\label{eigen-system of T_delta}
	T_\delta\phi_k = \lambda_k\phi_k,
\end{equation}
where the eigenvalues are organized in a non-increasing order $\lambda_1\geq\cdots\geq\lambda_m\geq\lambda_{m+1}=0=\cdots$, with the fact that $T_\delta$ has a rank of at most $m$, and the eigenfunctions $\{\phi_k\}_{k\in\mathbb{N}}$ establish an orthonormal basis for $\mathcal{H}_t$.
To further our understanding of the data set, we introduce the concept of the heat kernel matrix $ H_t\in\mathbb{R}^{m\times m} $, characterized by:
\begin{equation}\label{definition of heat kernel matrix}
	H_t(i,j)\triangleq \frac{1}{m}H_t(x_i,x_j).
\end{equation}
This matrix, comprised of heat kernel evaluations between data points, is crucial in the practical computation of spectral algorithm estimators. 
Given that the heat kernel matrix $ H_t $ is symmetric and positive-definite, it naturally supports an eigen-decomposition
\begin{equation}\label{eigen-system of heat kernel matrix}
	H_t\hat{u}_k=\hat{\lambda}_k\hat{u}_k.
\end{equation}
In this decomposition, eigenvalues are also organized in non-increasing order $\hat{\lambda}_1\geq\cdots\geq\hat{\lambda}_m>0$, with the corresponding eigenvectors $\{\hat{u}_k\}_{k=1}^m$ forming an orthonormal basis of $\mathbb{R}^m$. The eigenvalues and eigenvectors of $T_\delta$ and $H_t$ exhibit a close relationship, since the operator $ H_{t,x_j}^*H_{t,x_i}:y\mapsto yH_t(x_i,x_j)$ from $\mathbb{R}$ to $\mathbb{R}$ has a matrix representation $ H_t(x_i,x_j) $, whose matrix trace is scalar to $T_\delta$.
Specifically, we can explicitly express the relation between the eigensystem:
\begin{equation}\label{relation between eigen-system of T_delta and heat kernel matrix}
	\begin{aligned}
	&\lambda_k=\hat{\lambda}_k\\
	&\phi_k=\frac{1}{\sqrt{m\hat{\lambda}_k}}\sum_{i=1}^m\hat{u}_k(i)H_t(x_i,\cdot)
	\end{aligned}
\end{equation}
with $ \hat{u}_k(i) $ denoting the $ i $-th entry of the eigenvector $ \hat{u}_k $. Further insight into this connection is provided in \cite{guo2012empirical}.
Leveraging the eigen-decomposition of $T_\delta$ as in \eqref{eigen-system of T_delta}, the spectral algorithm estimator $f_{D,\lambda}$ can be unfolded as
\begin{equation}\label{spectral decomposition of f_D,lambda under empirical eigensystem}
	f_{D,\lambda}=\sum_{k=1}^\infty g_\lambda(\lambda_k)\langle g_D,\phi_k\rangle_{\mathcal{H}_t}\cdot\phi_k.
\end{equation}
This representation, informed by the relationship \eqref{relation between eigen-system of T_delta and heat kernel matrix}, is instrumental in the formulation and analysis of our algorithm.

\subsection{Graph Laplacian and Heat Kernel Estimation}\label{subsection: graph Laplacian and heat kernel estimation}

The representations of the spectral algorithm estimator $ f_{D,\lambda} $, either through its definition \eqref{definition of spectal alogrithm estimator f_D,lambda} or its spectral decomposition  \eqref{spectral decomposition of f_D,lambda under empirical eigensystem}, depend significantly on the explicit value of the heat kernel $ H_t$ on the data points. However, directly accessing or computing the heat kernel for a general manifold $\mathcal{M}$ poses practical challenges. To address this obstacle, we turn to the graph Laplacian approximation, a powerful tool for approximating manifold structures using discrete data points.

Suppose that we are given a dataset $ \{x_i\}_{i=1}^N $ independently and identically distributed from some distribution $ p $ on manifold $ \mathcal{M} $. We may, for the sake of simplicity, consider $ p $ as the uniform distribution $ \nu $ on $ \mathcal{M} $. To analyze the manifold structure using this dataset, we construct a graph affinity matrix $ W\in\mathbb{R}^{N\times N} $ and a diagonal degree matrix $ D\in\mathbb{R}^{N\times N} $ with entries
\begin{equation}\label{definition of graph matrix W and degree matrix D}
	\begin{aligned}
		&W(i,j)=K_\epsilon(x_i,x_j),\\
		&D(i,i)=\sum_{j=1}^N W(i,j).\\
	\end{aligned}
\end{equation}
Here, $K_\epsilon$ is defined as the Gaussian kernel
\begin{equation}\nonumber
	K_\epsilon(x,x')=\epsilon^{-\frac{d}{2}}\frac{1}{(4\pi)^\frac{d}{2}}\exp\left(-\frac{\|x-x'\|^2}{4\epsilon}\right),
\end{equation}
where $ \|\cdot\| $ denotes the usual Euclidean distance. This Gaussian kernel is selected for its ability to locally approximate the heat kernel, as exampled in \eqref{local approxiamtion property of heat kernel} for the Euclidean case, a property that could enhance the eigen-convergence rate of the graph Laplacian. Specifically, a recent studies \cite{cheng2022eigen} has revealed that such a local approximation result allows for an elaborate heat kernel interpolation method, thereby resulting in better convergence performance. The parameter $ \epsilon>0 $ serves as a proxy for the diffusion time, which is also adopted in \cite{singer2006graph,cheng2022eigen}. It is worth noting that in other literature, such as \cite{calder2022improved,dunson2021spectral}, the kernel parameters are expressed in terms of $ \sqrt{\epsilon}>0 $, reflecting the local distance (or kernel bandwidth).
Further, we define the un-normalized graph Laplacian $L_{un}\in\mathbb{R}^{N\times N}$ as
\begin{equation}\label{definition of un-normalized graph Laplacian}
	L_{un}= \frac{1}{p\epsilon N}(D-W).
\end{equation}
This definition includes a constant normalization factor to ensure that $ L_{un} $ converges to the Laplacian $\Delta$, and may not be involved in practical algorithms. The matrix $ L_{un} $ is characterized by its symmetry and positive semi-definiteness, with its smallest eigenvalue being zero. Moreover, as specified in Theorem \ref{Thm: L-infty eigen-system approxiamtion of graph Laplacian}, the eigensystem of $ L_{un} $ provides an effective approximation to that of $ \Delta $, demonstrating the utility of the graph Laplacian in approximating manifold structures from discrete data points. Detailed proofs of this theorem can be found in Appendix \ref{appendix: proof of L-infty eigen-system approxiamtion of graph Laplacian}. This approach bridges the gap between discrete data analysis and continuous manifold structures, offering a robust framework for understanding and exploring the intrinsic geometry of data manifolds. Recall that  $ \mathcal{M} $ is a compact, connected Riemannian manifold with dimension $ d $ and $ \nu $ is a uniform distribution on  $ \mathcal{M} $.

\begin{theorem}\label{Thm: L-infty eigen-system approxiamtion of graph Laplacian}
For any fixed $ K\in\mathbb{N} $, assume that the eigenvalues $ \mu_k $ of $ \Delta $ are all of single multiplicity for every $ k\leq K+1 $. Consider the first $ K $ eigenvalues $ \{\tilde{\mu}_k\}_{k=1}^K $  and associated eigenvectors $ \{\tilde{v}_k\}_{k=1}^K $ of $ L_{un} $ with appropriate normalizations as
	\begin{equation}\nonumber
		\begin{aligned}
			&0=\tilde{\mu}_1\leq\cdots\leq\tilde{\mu}_K,\\
			&L_{un}\tilde{\mu}_k=\tilde{\mu}_k\tilde{v}_k,\\
			&\tilde{v}_k^T\tilde{v}_k = pN,\quad\tilde{v}_k^T\tilde{v}_l=0,\;\forall\;k\neq l.
		\end{aligned}
	\end{equation}
	Then, as $ N\to\infty $ and $ \epsilon\to0+$ with $ \epsilon\sim\left(\frac{\log N}{N}\right)^\frac{1}{\frac{d}{2}+2}$, for sufficiently large $ N $, there exists an event $ E_1 $ with probability larger than $ 1-cN^{-8} $ such that, on $ E_1 $ it holds 
	\begin{equation}\label{graph laplacian eigenvalue approximation error}
		|\tilde{\mu}_k-\mu_k|\leq C_1\left(\frac{\log N}{N}\right)^\frac{1}{\frac{d}{2}+2},
	\end{equation}
	\begin{equation}\label{graph laplacian L^infty eigenvector approximation error}
		\|\tilde{v}_k-\rho_X(\psi_k)\|_{\infty}\leq C_2\left(\frac{\log N}{N}\right)^\frac{1}{d+4}
	\end{equation}
	for all $ 1\leq k\leq K $, where $ \rho_X:C^\infty(\mathcal{M})\to\mathbb{R}^N $ is the sampling operator defined as $ \rho_X(f)=(f(x_1),\cdots,f(x_N))^T $ and $c,C_1,C_2 $ are constants independent of $ N, \epsilon $. 
\end{theorem}

\begin{remark}
The assumption of single multiplicity for the eigenvalues of the Laplacian $ \Delta $ in Theorem \ref{Thm: L-infty eigen-system approxiamtion of graph Laplacian} is introduced primarily for simplicity. This assumption facilitates the derivation of specific eigenvalue and eigenvector approximation errors, as referenced in \cite{cheng2022eigen} and further elaborated in the Appendix \ref{appendix: proof of L-infty eigen-system approxiamtion of graph Laplacian}. 
	However, it is important to note that such an assumption is not necessary for the theorem's validity. Following the same approach in \cite{cheng2022eigen} could bypass this assumption. Generally, the eigen-space of $ \Delta $ may encompass dimensions greater than one, implying that our eigenvector approximation error specified in \eqref{graph laplacian L^infty eigenvector approximation error} may not directly apply to a predetermined $ \{\psi_k\} $ within the eigen-system of $ \Delta $. Therefore, the theorem should be understood as asserting the existence of an appropriate orthonormal eigenfunction basis $ \{\psi_k\} $ for which the specified eigenvector approximation error \eqref{graph laplacian L^infty eigenvector approximation error} holds. Such an existence result is sufficient for our purpose.
\end{remark}

Drawing from the eigen-consistency results presented in Theorem \ref{Thm: L-infty eigen-system approxiamtion of graph Laplacian} and leveraging the spectral decomposition \eqref{spectral decomposition of heat kernel} of the heat kernel, we can craft an estimator for the heat kernel $ H_t $ using the following formula:
\begin{equation}\label{definition of heat kernel estimator H_t,K}
	\tilde{H}_{t,K}=\sum_{i=1}^K e^{-\tilde{\mu}_it}\cdot\tilde{v}_i\tilde{v}_i^T.
\end{equation}
This formulation yields an estimator, $ \tilde{H}_{t,K} $, that is essentially an $ N\times N $ matrix, approximated across the sample points $ \{x_i\}_{i=1}^N $. As highlighted in Theorem \ref{Thm: heat kernel estimation}, each entry of $ \tilde{H}_{t,K} $ closely approximates the true heat kernel values at the sampled points, ensuring the estimator's efficacy. The proof of Theorem \ref{Thm: heat kernel estimation} is reserved for the Appendix \ref{appendix: proof of heat kernel estimation}. Similar approaches could be found in \cite{dunson2019diffusion,dunson2021spectral}.

\begin{theorem}\label{Thm: heat kernel estimation}
	Suppose that $ K, \epsilon, N $ satisfy following conditions: 
	\begin{equation}\label{conditions on heat kernel estimation Thm}
		\begin{cases}
			&K\geq\left(\frac{d}{C_{low}t}\right)^\frac{d}{2}-1\\
			&(K+1)e^{-C_{low}(K+1)^\frac{2}{d}t}\leq\frac{1}{K}\\
			&\epsilon^{\frac{1}{4}}\leq \frac{1}{C_1C_2K\mu_K^\frac{d-1}{2}}\\
			&\left(\frac{\log N}{N}\right)^\frac{1}{\frac{d}{2}+2}\sim\epsilon
		\end{cases}.
	\end{equation}
	Then, on the event $ E_1 $ proposed in Theorem \ref{Thm: L-infty eigen-system approxiamtion of graph Laplacian}, it holds
	\begin{equation}\label{point-wise error of heat kernel estimator}
		\sup\limits_{i,j\in\{1,\cdots,N\}}\left|H_t(x_i,x_j)-\tilde{H}_{t,K}(i,j)\right|\leq \frac{C_a}{K}+C_b\epsilon^{\frac{1}{4}}
	\end{equation}
	with constants $ C_a,C_b $ independent of $ K,\epsilon, N $.
\end{theorem}

\begin{remark}
	The conditions concerning $ K $, $ \epsilon $, and $ N $ , namely the assumption that $ K $ is large enough, $\epsilon>0 $ is sufficiently small, and $ N $ is adequately large, is to ensure the effectiveness of our heat kernel approximation. It's important to note that, those constants $ C_1 $ and $ C_2 $, derived from Theorem \ref{Thm: L-infty eigen-system approxiamtion of graph Laplacian}, are not absolute but depend on $ K $. This implies a sequential approach to parameter selection in a practical algorithm: initially determining  $ K $, followed by $ \epsilon $, and finally $ N $. 
	However, as \cite{dunson2021spectral} have pointed out, the practical application of these theoretical guidelines faces some challenges. The constants tied to the general characteristics of the manifold $ \mathcal{M} $ are often difficult to ascertain or control in real-world scenarios. Hence, the practical selection of these hyperparameters—$ K $, $ \epsilon $, and $ N $—typically relies on a trial-and-error approach, rather than a direct application of the theoretical conditions. How to fill the gap between theoretical idealizations and practical feasibility  will be a focus of subsequent work.
\end{remark}

\section{Main Results}\label{section: main results}

\subsection{Diffusion-Based Spectral Algorithm}\label{subsection: diffusion-based spectral algorithm}

In this subsection, we introduce our diffusion-based spectral algorithm. Given a labeled dataset $ D^l=\{(x_i,y_i)\}_{i=1}^m $ independently and identically distributed from $ \mathcal{P} $ on $ \mathcal{M}\times\mathbb{R} $ and an unlabeled dataset $ D^{ul}=\{x_i\}_{i=m+1}^{m+n} $ i.i.d from $ \nu $ on $ \mathcal{M} $, we commence by constructing a heat kernel estimator $ \tilde{H}_{t,K} $ as in Subsection \ref{subsection: graph Laplacian and heat kernel estimation} based on the comprehensive input dataset $ \{x_i\}_{i=1}^{m+n} $ with $ N=m+n $. The estimator $ \tilde{H}_{t,K} $ is an $ N\times N $ block matrix
\begin{equation}\nonumber
	\tilde{H}_{t,K}=\begin{pmatrix}
		\tilde{H}_{t,K}^{m} & \Sigma_{12}\\
		\Sigma_{21} & \Sigma_{22}
	\end{pmatrix}.
\end{equation}
In this formulation, the first block $ \tilde{H}_{t,K}^{m} $ is an $ m\times m $ matrix that effectively approximates the heat kernel for the labeled data points $ \{x_i\}_{i=1}^m $, as delineated in Theorem \ref{Thm: heat kernel estimation}. To be more precise, we have
\begin{equation}\nonumber
	\sup\limits_{i,j\in\{1,\cdots,m\}}\left|H_t(x_i,x_j)-\tilde{H}_{t,K}^{m}(i,j)\right|\leq \frac{C_a}{K}+C_b\epsilon^{\frac{1}{4}}.
\end{equation}
Let $ \tilde{H}_{t,m}=\frac{1}{m}\tilde{H}_{t,K}^{m} $. Leveraging the aforementioned pointwise approximation, an $l^ \infty $-norm error bound between $ \tilde{H}_{t,m} $ and the heat kernel matrix $ H_t $ as defined previously in \eqref{definition of heat kernel matrix} can be established:
\begin{equation}\label{point-wise error of heat kernel matrxi estimator}
	\left\|\tilde{H}_{t,m}- H_t\right\|_{\infty}\leq \frac{1}{m}\left(\frac{C_a}{K}+C_b\epsilon^{\frac{1}{4}}\right).
\end{equation}
This $\ell^ \infty $-norm error bound paves the way for a consistency result between the eigen-system of $ \tilde{H}_{t,m} $ and that of $ H_t $. Notably, since the matrix $ \tilde{H}_{t,m} $ is symmetric and positive semi-definite, its eigenvalues can be sorted in a non-increasing order $\tilde{\lambda}_1\geq\cdots\geq\tilde{\lambda}_m\geq0$. Their corresponding eigenvectors, $\{\tilde{u}_k\}_{k=1}^m$, constitute an orthonormal basis of $\mathbb{R}^m$. The relationship of those eigens is further explored in Theorem \ref{Thm: eigen-system consistency of heat kernel matrix}, with a proof available in the Appendix \ref{appendix: proof of eigen-system consistency of heat kernel matrix}.

\begin{theorem}\label{Thm: eigen-system consistency of heat kernel matrix}
	Suppose that condition \eqref{conditions on heat kernel estimation Thm} is satisfied, with $ E_1 $ being the event in Theorem \ref{Thm: L-infty eigen-system approxiamtion of graph Laplacian}. Denote
	\begin{equation}\label{definition of truncated q}
		q = \left(\frac{\log m}{C_{up}t(2r+1)}\right)^\frac{d}{2}
	\end{equation}
	with $ r>1 $ a fixed constant. Then, for any $ \tau\geq1 $, there exists an event $ E_2 $ with probability larger than $ 1-2e^{-\tau} $ such that, on $ E_1\bigcap E_2 $, for any $ 1\leq k\leq q $, there hold
	\begin{align}
		&|\tilde{\lambda}_k-\hat{\lambda}_k|\leq \frac{C_a}{K}+C_b\epsilon^{\frac{1}{4}},
		\label{eigenvalue consistency of heat kernel matrix}\\
		&\|\tilde{u}_k-\hat{u}_k\|_\infty\leq C_c m^{\frac{1}{2r+1}}\left(\frac{1}{K}+\epsilon^{\frac{1}{4}}\right),
		\label{L-infty eigenvector consistency of heat kernel matrix}
	\end{align}
	where constants $ C_a,C_b$ are the same as in Theorem \ref{Thm: heat kernel estimation} and $ C_c $ is independent of $ K,\epsilon,m,$ or $ N $.
\end{theorem}

Theorem \ref{Thm: eigen-system consistency of heat kernel matrix} suggests that the eigen-pairs of $ H_t $ can be effectively replaced by those of $ \tilde{H}_{t,m} $ when appropriate hyperparameters are selected. This substitution is reinforced by the relationship \eqref{relation between eigen-system of T_delta and heat kernel matrix} between eigens of $ T_\delta $ and that of $ H_t $, enabling us to define the following data-based empirical eigens 
\begin{equation}\label{estimator of eigen-system of T_delta}
	\begin{split}
		&\tilde{\lambda}_k=\tilde{\lambda}_k,\\
		&\tilde{\phi}_k=\frac{1}{\sqrt{m\tilde{\lambda}_k}}\sum_{i=1}^m\tilde{u}_k(i)\tilde{H}_{t,K}(i,\cdot)
	\end{split}	
\end{equation}
as alternatives to the eigenfunctions of the empirical integral operator $ T_\delta $. In this definition, $ \tilde{H}_{t,K}(i,\cdot) $ denotes the $ i $-th row of the matrix $ \tilde{H}_{t,K}$, which is an $ N $-dimensional vector. For clarity, the notation for the eigenvalues of matrix $ \tilde{H}_{t,m} $ and the empirical eigenvalue estimators for $ T_\delta $ remains consistent, reflecting their equivalence. The congruence between $ \tilde{\lambda}_k $ and $ \lambda_k $ has been previously established in Theorem \ref{Thm: eigen-system consistency of heat kernel matrix}, since $ \lambda_k=\hat{\lambda}_k$ from the relation \eqref{relation between eigen-system of T_delta and heat kernel matrix}.
We proceed to demonstrate that the data-based empirical eigenfunction estimator $ \tilde{\phi}_k $, represented as an $ N $-dimensional vector, serves as a good approximation of the empirical eigenfunction $ \phi_k $ of $ T_\delta $ across the full input dataset $ \{x_i\}_{i=1}^N $. The accuracy of the approximation is outlined in the upcoming Lemma \ref{Lemma: eigenfunction estimation}. A proof of this lemma is presented in Appendix \ref{appendix: proof of eigenfunction estimation}.

\begin{lemma}\label{Lemma: eigenfunction estimation}
	Suppose that the conditions of Theorem \ref{Thm: eigen-system consistency of heat kernel matrix} are satisfied, with $ E_1 $ being the event in Theorem \ref{Thm: L-infty eigen-system approxiamtion of graph Laplacian} and $ E_2 $ being the event in Theorem \ref{Thm: eigen-system consistency of heat kernel matrix}. Then, on $ E_1\bigcap E_2 $, for any $ 1\leq k\leq q $, it holds that
	\begin{equation}\label{estimation error of empirical eigenfunction estimator}
		\sup\limits_{1\leq i\leq N}|\phi_k(x_i)-\tilde{\phi}_k(i)|\leq C_d m^{\frac{r+2}{2r+1}}\left(\frac{1}{K}+\epsilon^{\frac{1}{4}}\right),
	\end{equation}
	where $ C_d $ is a constant independent of $ K,\epsilon,m,N $.
\end{lemma}

Returning to the spectral decomposition representation \eqref{spectral decomposition of f_D,lambda under empirical eigensystem} of the spectral algorithm estimator $ f_{D,\lambda} $, it is remarkable that
\begin{equation}\nonumber
	\langle g_D,\phi_k\rangle_{\mathcal{H}_t}=\frac{1}{m}\sum_{i=1}^m y_i\langle \phi_k,H_t(x_i,\cdot)\rangle_{\mathcal{H}_t}=\frac{1}{m}\sum_{i=1}^m y_i\phi_k(x_i),
\end{equation}
which is given by the reproducing property. Consequently, to facilitate our algroithm, we introduce
\begin{equation}\label{estimator of gD}
	\tilde{g}_{D,k}=\frac{1}{m}\sum_{i=1}^m y_i\tilde{\phi}_k(i)
\end{equation}
as a substitution of $ \langle g_D,\phi_k\rangle_{\mathcal{H}_t} $, which is also an $ N $-dimensional vector.
With the foundational elements established, we present our diffusion-based spectral algorithm estimator as an approximated-truncated version of the original spectral algorithm:
\begin{equation}\label{definition of our estimator}
	\tilde{f}_{D,\lambda}=\sum_{k=1}^q g_\lambda(\tilde{\lambda}_k)\tilde{g}_{D,k}\tilde{\phi}_k\in\mathbb{R}^N.
\end{equation}
This formulation results in an $ N $-dimensional vector, where the $ i $-th entry provides an estimation of the regression function $ f^*(x_i) $ at the $ i $-th data point $ x_i $ for all $ 1\leq i \leq N $. Emphasizing its capacity as a semi-supervised learning algorithm, our primary interest lies in its performance on the unlabeled input dataset $ \{x_i\}_{i=m+1}^{m+n} $, which will be elaborated upon subsequently.

As a comparison, the classical spectral algorithm is typically computed (see e.g. \cite{rosasco2005spectral,bauer2007regularization}) using the formula 
\begin{equation*}
	f_{D,\lambda}=\frac{1}{m}\sum_{i=1}^m \hat{\gamma}_iH_t(x_i,\cdot),
\end{equation*}
where $ \hat{\gamma}=(\hat{\gamma}_1,\dots,\hat{\gamma}_m)\in\mathbb{R}^m $ is determined by
\begin{equation*}
	\hat{\gamma} = g_\lambda(H_t)y.
\end{equation*}
In this context, $ H_t $ represents the heat kernel matrix and $ y=(y_1,\dots,y_m)\in\mathbb{R}^m $ denotes the response variables. The computation of $ \hat{\gamma} $ necessitates an eigen-decomposition of the entire kernel matrix $ H_t $, which is of dimension $ m $. Our diffusion-based spectral algorithm estimator advances this process by substituting the incalculable heat kernel terms with data-based, calculable estimators and reduces the number of heat kernel matrix eigen-pairs needed from $ m $ to $ q $, where $ q $ is on the $ \log m $ scale. This approach result in a reduction of computational cost while theoretically maintaining learning rates. For insights into similar truncation methods within kernel regression algorithms, references such as \cite{guo2012empirical,guo2017thresholded} are valuable.
We conclude this section by presenting our algorithm in Algorithm \ref{Alg: diffusion-based spectral algorithm} below.

\begin{algorithm} 
	\caption{Diffusion-based spectral algorithm}
	\label{Alg: diffusion-based spectral algorithm} 
	\begin{algorithmic}
        \renewcommand{\algorithmicrequire}{\textbf{Input:}}
        \renewcommand{\algorithmicensure}{\textbf{Output:}}
        \REQUIRE labeled dataset $ D^l=\{(x_i,y_i)\}_{i=1}^m $, unlabeled dataset $ D^{ul}=\{x_i\}_{i=m+1}^{m+n} $, regularization family $ \{g_\lambda\} $, hyperparameters $ N=m+n $, $ \epsilon $, $ K $, $ q $, $ \lambda $, $ t $ 
		\ENSURE estimator $\tilde{f}_{D,\lambda}$
		\STATE \textbf{Step 1.} Compute the un-normalized graph Laplacian $ L_{un} $ on $ \{x_i\}_{i=1}^{N} $ by \eqref{definition of un-normalized graph Laplacian} with diffusion time $ \epsilon $.
		\STATE \textbf{Step 2.} Compute the eigenvalues $ \tilde{\mu}_k$ and eigenvectors $ \tilde{v}_k $ of matrix $ L_{un} $ with normalization stated in Theorem \ref{Thm: L-infty eigen-system approxiamtion of graph Laplacian}.
		\STATE \textbf{Step 3.} Compute the heat kernel estimator $ \tilde{H}_{t,K} $ by \eqref{definition of heat kernel estimator H_t,K} with truncation number $ K $.
		\STATE \textbf{Step 4.} Compute matrix $ \tilde{H}_{t,m} $ by selecting first $ m\times m $ elements of $ \tilde{H}_{t,K} $ with a normalization factor $ \frac{1}{m} $. 
		\STATE \textbf{Step 5.} Compute the eigenvalues $ \tilde{\lambda}_k$ and eigenvectors $ \tilde{u}_k $ of matrix $ \tilde{H}_{t,m} $ with normalization stated in Section \ref{subsection: diffusion-based spectral algorithm}. 
		\STATE \textbf{Step 6.} Compute the data-based empirical eigenvalue estimators $ \tilde{\lambda}_k $ and eigenfunction estimators $ \tilde{\phi}_k $ by \eqref{estimator of eigen-system of T_delta}.
		\STATE \textbf{Step 7.} Compute the diffusion-based spectral algorithm estimator $\tilde{f}_{D,\lambda} $ by \eqref{estimator of gD} and \eqref{definition of our estimator} with regularization family $ \{g_\lambda\} $, normalization parameter $ \lambda $ and truncation number $ q $.
	\end{algorithmic} 
\end{algorithm}

\subsection{Convergence Results}\label{subsection: convergence analysis}

In this subsection, we delve into the convergence results of our diffusion-based spectral algorithm. This examination begins by introducing several pivotal assumptions essential for establishing a robust analytical foundation.

\begin{assumption}\label{Assumption: source condition on regression function}
	There exists a constant $ r>1 $ and a function $ g\in\mathcal{H}_t $ such that
	\begin{equation}\label{souce condition assumption}
		f^* = T_\nu^r(g).
	\end{equation}
\end{assumption}

In the convergence analysis of our diffusion-based spectral algorithm, we employ this source condition, a fundamental assumption widely used across various studies in kernel methods and spectral algorithms (see, for instance, \cite{bauer2007regularization, caponnetto2007optimal, caponnetto2010cross, lin2018optimal}). In other words, we adapt the notion of $ \alpha $-power spaces previously introduced in \eqref{definition of alpha-power space}, to assert that the target regression function $ f^*\in\mathcal{H}_t^\beta $ for some $ \beta>3 $.
This adaptation of the source condition strengthens the classical source condition, which typically posits $ f^*\in\mathcal{H}_t $. The motivation for this enhanced requirement lies in addressing technical considerations within our analysis framework. Nonetheless, it is important to note that this condition could be relaxed back to the classical scenario--or even extended to address harder learning scenarios where $ \beta<1$--by applying a more intricate approach. The exploration of such scenarios is earmarked for future investigation.

\begin{assumption}\label{Assumption: boundedness condition on regression function}
	There exists a constant $ M>0 $ such that, for any sample pair $(x,y) $ distributed from $ \mathcal{P} $ on $ \mathcal{M}\times\mathbb{R} $, it holds
	\begin{equation}\label{boundedness condition on conditional distribution}
		|y|\leq M.
	\end{equation}
\end{assumption}

This assumption pertains to the boundedness of the conditional distribution of $ y $ given $ x $, a premise frequently invoked in statistical learning theory. This boundedness condition, as detailed in \cite{steinwart2009optimal,guo2017thresholded,pillaud2018statistical,jun2019kernel}, provides a foundational constraint that facilitates the development and analysis of learning algorithms by bounding the outcomes within a predictable range.
In contrast, some researchers, as seen in \cite{fischer2020sobolev,zhang2024optimality}, opt for a moment condition approach to manage the tail probabilities of the noise term $ \epsilon=y-f^*(x) $. This moment condition, which is generally a weaker constraint than the direct boundedness assumption, focuses on controlling the statistical behavior of the noise rather than simply bounding its magnitude.
For the purposes of our discussion, the boundedness condition is adopted to simplify the analysis. Although moment condition constraints are weaker, they can still facilitate the derivation of similar results with a little extra effort.

\begin{assumption}\label{Assumption: qualification condition on filter function}
	The regularization family $ \{g_\lambda\} $ has a qualification of $ \xi\geq 1 $. 
\end{assumption}

The qualification condition imposed on the regularization family is relatively moderate and can be readily satisfied by various methods. For instance, the examples highlighted in the subsection \ref{subsection: regularization family and spectral algorithm} align with this requirement. Additionally, several other notable algorithms, including Landweber iteration, accelerated Landweber iteration, and iterated Tikhonov regularization, also conform to this condition. For further elaboration on these methods and their compliance with the qualification condition, one can refer to \cite{bauer2007regularization,gerfo2008spectral}, where these aspects are discussed in detail.

We now present the convergence analysis for our diffusion-based spectral algorithm in the following Theorem \ref{Thm: convergence analysis}.

\begin{theorem}\label{Thm: convergence analysis}
	Suppose that Assumption \ref{Assumption: source condition on regression function}, Assumption \ref{Assumption: boundedness condition on regression function}, and Assumption \ref{Assumption: qualification condition on filter function} hold. Suppose that condition \eqref{conditions on heat kernel estimation Thm} is satisfied with $ m $ sufficiently large, $ N>K\gg m $ sufficiently large, and $ \epsilon\ll 1 $ sufficiently small. Denote $ E_1 $ the event given in Theorem \ref{Thm: L-infty eigen-system approxiamtion of graph Laplacian} and $ E_2 $ the event given in Theorem \ref{Thm: eigen-system consistency of heat kernel matrix}. Moreover, we choose $ q $ as in \eqref{definition of truncated q} and $ \lambda $ as 
	\begin{equation}\label{definition of lambda}
		\lambda \sim \left(\frac{(\log m)^\frac{d}{2}}{m}\right)^\frac{1}{2}.
	\end{equation}
	Then, for any $ \tau\geq1 $, there exists an event $ E_3 $ with probability larger than $ 1-20e^{-\tau} $ such that, on $ E_1\bigcap E_2\bigcap E_3 $, for any $ \alpha>0 $ and $ m+1\leq j\leq m+n $, 
	it holds
	\begin{equation}\label{convergence rate of diffusion-based spectral algorithm}
		|\tilde{f}_{D,\lambda}(j)-f^*(x_j)|\leq C m^{-\theta_1}\cdot(\log m)^{\theta_2}+C' m^{\frac{3}{2r+1}+1}(\log m)^\frac{d}{2}\cdot\left(\frac{1}{K}+\epsilon^{\frac{1}{4}}\right),
	\end{equation}
	where
	\begin{equation}\label{definition of theta_1 and theta_2}
		\begin{aligned}
			&\theta_1 = \min\left\{\frac{1}{2}\frac{C_{low}}{C_{up}},\;\frac{\xi}{2}-\frac{\xi+\frac{1}{2}}{2r+1}\right\}-\frac{\alpha}{4},\\
			&\theta_2 = 
			\left\{
				\begin{aligned}
					\frac{d\xi}{4}-\frac{d\alpha}{8},\quad&\text{if}\quad \frac{1}{2}\frac{C_{low}}{C_{up}}\geq\frac{\xi}{2}-\frac{\xi+\frac{1}{2}}{2r+1}\\
					\frac{1}{2}+\frac{d}{4}-\frac{d\alpha}{8},\quad&\text{if}\quad \frac{1}{2}\frac{C_{low}}{C_{up}}<\frac{\xi}{2}-\frac{\xi+\frac{1}{2}}{2r+1}\\
				\end{aligned}
			\right.,
		\end{aligned}
	\end{equation} 
	and $ C,C' $ are constants independent of $ m,K,N,\epsilon,\alpha$.
\end{theorem}

To elucidate the complex convergence result presented, let us examine two specific cases of the regularization family $ \{g_\lambda\} $. For the first example, we select $ \{g_\lambda\} $ with a qualification of $ \xi=\infty $, a condition met by algorithms such as gradient flow, kernel principal component regularization, and Landweber iteration. In this scenario, according to \eqref{definition of theta_1 and theta_2}, we find that $ \theta_1=\frac{C_{low}}{2C_{up}}-\frac{\alpha}{4} $, where $ C_{low}$ and $C_{up} $ are absolute constants related to the eigenvalues of the Laplacian from \eqref{population bounds for eigenvalue of laplacian}.
From the proof detailed in section \ref{section: proof of convergence result}, these constants primarily feature in estimating the $ q $-th eigenvalue, $ \mu_q $, with $ q\to\infty $ as defined in \eqref{definition of truncated q}, thus they can be considered as functions of $ q $ that guide the upper and lower bounds for $ \mu_q $. Classical Weyl's law suggests that $ \mu_q\sim q^\frac{2}{d} $, leading to $ \frac{C_{low}}{2C_{up}}\to\frac{1}{2} $ as $ m\to\infty $. Consequently, the upper bound \eqref{convergence rate of diffusion-based spectral algorithm} becomes{
\begin{equation}\label{convergence rate of diffusion-based spectral algorithm, special case}
	\sup_{1 \leq j \leq N}|\tilde{f}_{D,\lambda}(j)-f^*(x_j)|=\tilde{O}\left( m^{-\frac{1}{2}+\alpha}+m^{\frac{3}{2r+1}+1}\left(\frac{1}{K}+\left(\frac{\log N}{N}\right)^\frac{1}{2d+8}\right)\right),
\end{equation}}
with $ \alpha>0 $ being arbitrarily small for sufficiently large $ m\ll K<N $, focusing only on the main terms.
As a second example, consider $ g_\lambda(t)=\frac{1}{t+\lambda} $, corresponding to Tikhonov regularization. This choice yields $ \xi=1 $, where our upper bounds \eqref{convergence rate of diffusion-based spectral algorithm} might perform poorly if the regression function $ f^* $ lacks sufficient smoothness. However, the Lipschitz continuity of this filter function—a characteristic not common to most filters—allows for a modified proof approach in our Theorem \ref{Thm: convergence analysis}, achieving a comparable convergence rate as in \eqref{convergence rate of diffusion-based spectral algorithm, special case}. The detailed proof for Tikhonov regularization's convergence rate is deferred to the Appendix \ref{appendix: convergence rate for Tikhonov regularization}.

\section{Numerical Experiments}\label{section: numerical experiments}

In this section, numerical experiments are conducted to showcase the efficacy of the diffusion-based spectral algorithm on a synthetically generated manifold dataset. This dataset is uniformly sampled from the unit sphere $ S^2 $ in $ \mathbb{R}^3 $, comprising $ 1600 $ data points. These points are divided into a labeled dataset $ D^l $ with $ 140 $ data points and an unlabeled dataset $ D^{ul} $ with $ 1460 $ data points. We first consider an intrinsic unit variable regression function, defined as $ f^*(\theta)=20\sin(\theta)+24\cos(\theta) $.
For the labeled dataset $ D^l $, the response variable $ y $ is generated according to
\begin{equation*}
	y = f^*(x)+\epsilon,
\end{equation*}
where $ \epsilon\sim\mathcal{N}(0,1) $ denotes Gaussian white noise with unit variance. The experiment evaluates three different filters corresponding to kernel ridge regression, kernel principal component regularization, and gradient flow. Hyperparameters are set by $ \epsilon=10^{-1.4}, \lambda=10^{-4}, K=200, q=80$, and $ t=0.4 $ for each filter. The outcomes of these simulations are illustrated in Figure \ref{Fig1}.
Additionally, a more complex regression function, $ f^*(\theta,\phi)=20\sin(\theta)\phi $, is examined. This function represents a product of two intrinsic variables, posing a non-trivial regression scenario. Utilizing the same algorithms and parameter settings as before, the results of these simulations are depicted in Figure \ref{Fig2}.

\begin{figure}[htbp]
	\centering  
	\subfigure[kernel ridge regression]{
	\label{Fig1.krr}
	\includegraphics[width=0.45\textwidth]{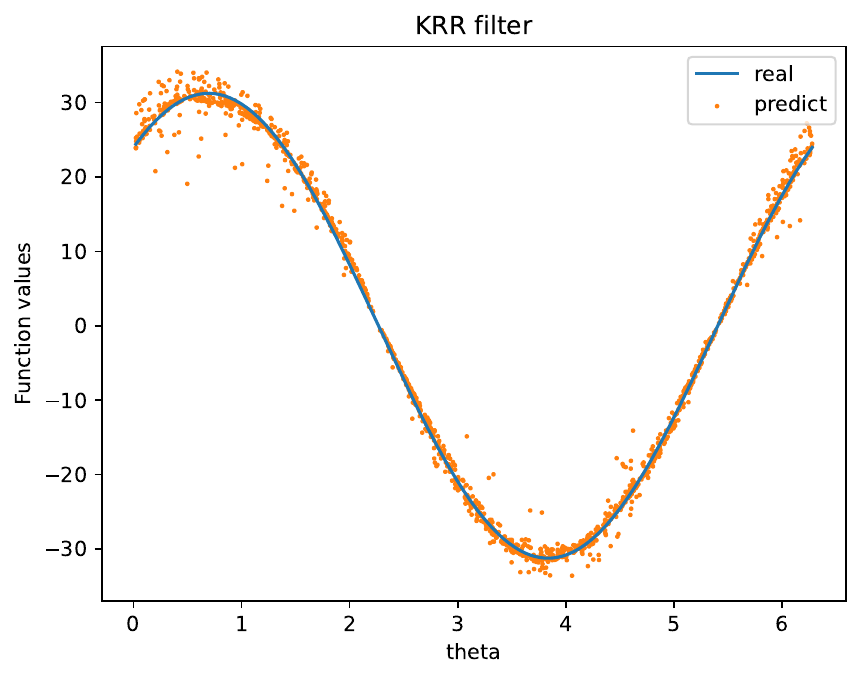}}
	\subfigure[kernel principal component regularization]{
	\label{Fig1.kpcr}
	\includegraphics[width=0.45\textwidth]{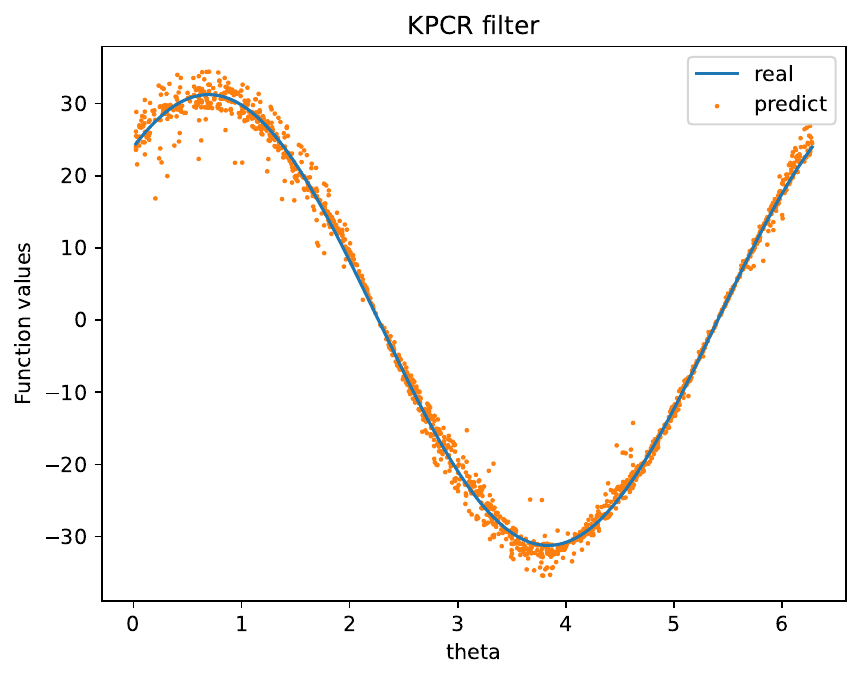}}
	\subfigure[gradient flow]{
	\label{Fig1.gf}
	\includegraphics[width=0.45\textwidth]{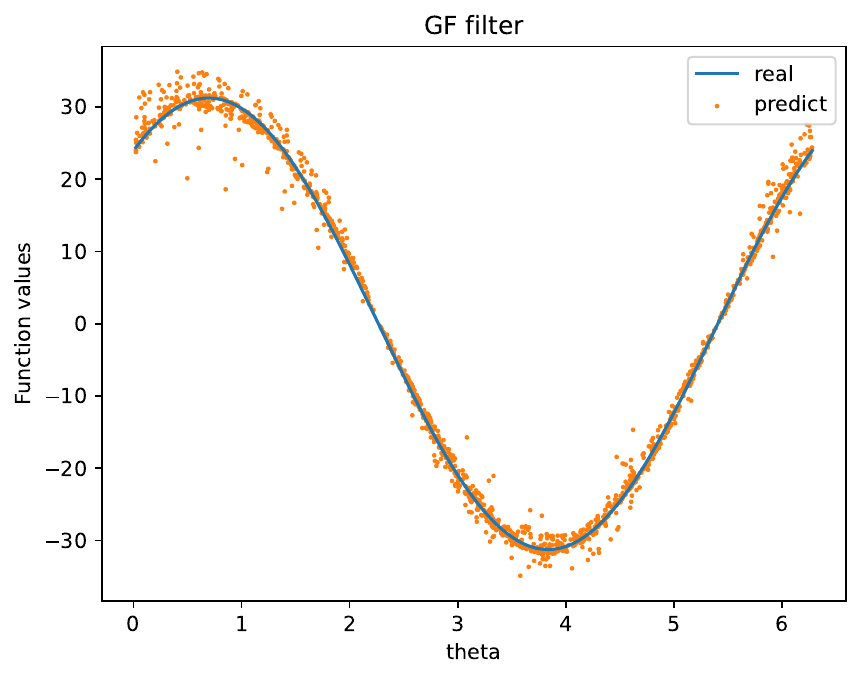}}
	\caption{Results for regression function $ f^*(\theta)=20\sin(\theta)+24\cos(\theta) $ on $ S^2 $ over a labeled dataset with $ 140 $ data points and an unlabeled dataset with $ 1460 $ data points. Figure \ref{Fig1.krr} illustrates the outcomes of kernel ridge regression, while Figure \ref{Fig1.kpcr} displays the results for kernel principal component regularization, and Figure \ref{Fig1.gf} showcases the application of gradient flow. In each depiction, the actual values of $ f^* $  are represented by a continuous blue curve, contrasting with the orange scatter points that signify the algorithm's predictions.}
	\label{Fig1}
\end{figure}

\begin{figure}[htbp]
	\centering  
	\subfigure[real value of $ f^* $ ]{
	\label{Fig2.real}
	\includegraphics[width=0.45\textwidth]{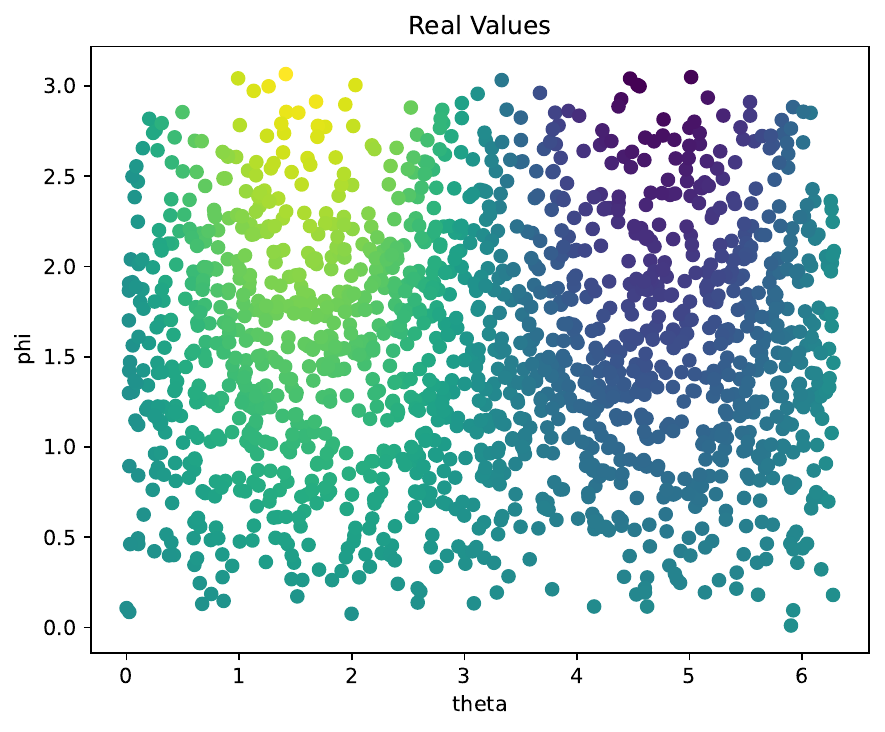}}
	\subfigure[kernel ridge regression]{
	\label{Fig2.krr}
	\includegraphics[width=0.45\textwidth]{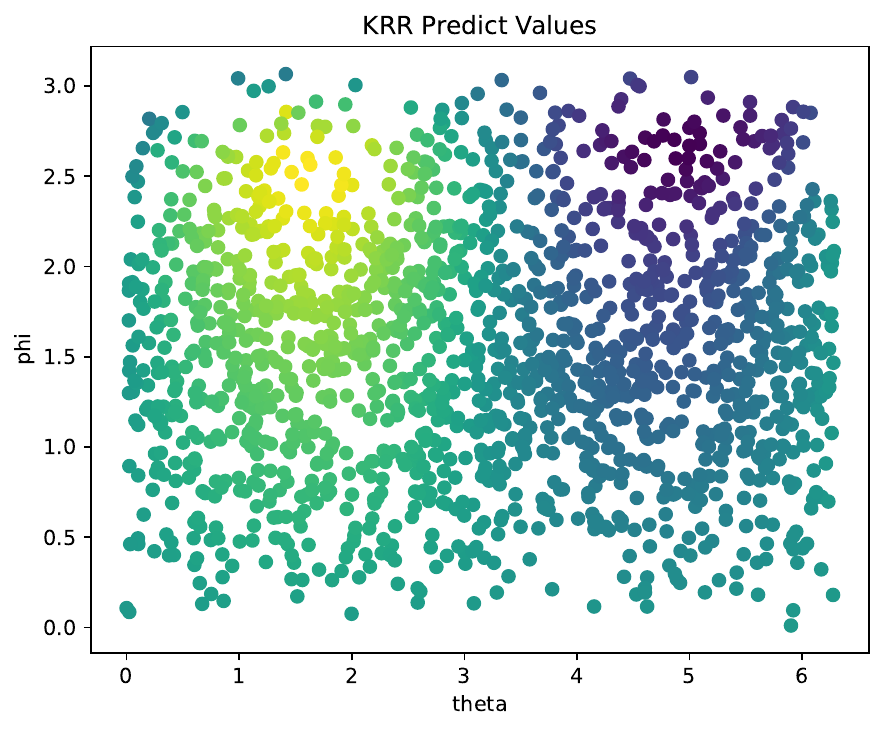}}
	\subfigure[kernel principal component regularization]{
	\label{Fig2.kpcr}
	\includegraphics[width=0.45\textwidth]{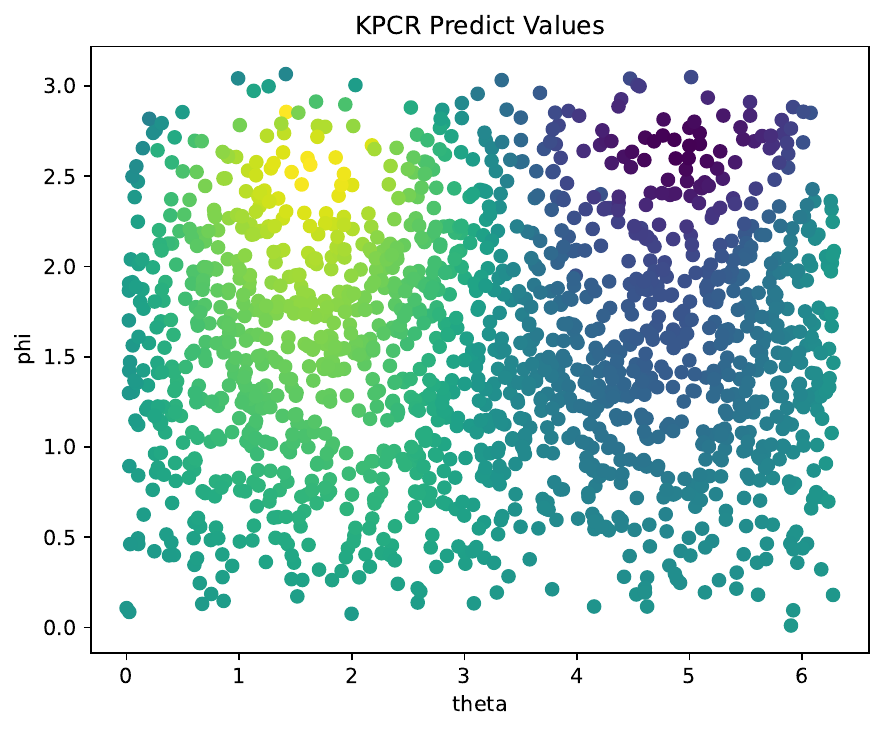}}
	\subfigure[gradient flow]{
	\label{Fig2.gf}
	\includegraphics[width=0.45\textwidth]{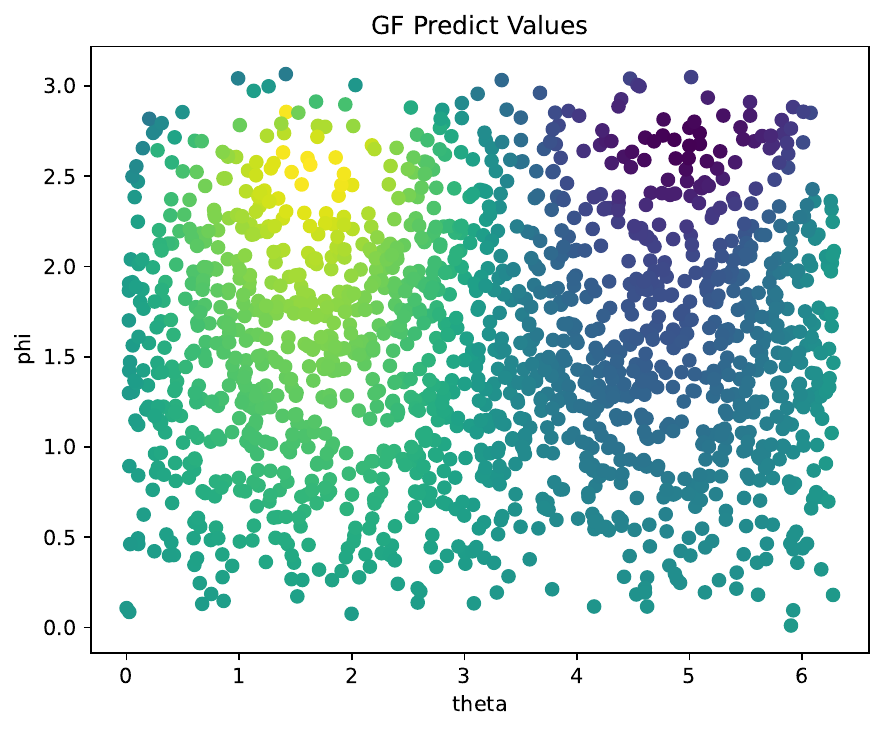}}
	\caption{Results for regression function $ f^*(\theta,\phi)=20\sin(\theta)\phi $ on $ S^2 $ over the same dataset with $ 140 $ labeled points and $ 1460 $ unlabeled points. Figure \ref{Fig2.real} visualizes the true distribution of $ f^* $, serving as a benchmark for comparison. Figures \ref{Fig2.krr}, \ref{Fig2.kpcr}, and \ref{Fig2.gf} respectively illustrate the predictions made by kernel ridge regression, kernel principal component regularization, and gradient flow. In these figures, the color assigned to each sample point denotes its response value.}
	\label{Fig2}
\end{figure}

The outcomes of our numerical experiments demonstrate that our diffusion-based spectral algorithm can accurately estimate the regression target function $ f^* $  across the entire dataset, comprising both the labeled dataset $ D^l $ and the unlabeled dataset $ D^{ul} $, even in scenarios where the labeled dataset is significantly smaller than the unlabeled one, aligning with the predictions made by our theoretical analysis. These findings underscore the algorithm's efficacy as a semi-supervised learning tool, showing its resilience across various regression functions.
Furthermore, when comparing the performances of kernel ridge regression (KRR) filters, which possess a qualification of $ \xi=1 $, against kernel principal component regularization (KPCR) or gradient flow (GF) filters, which boast a qualification of $ \xi=\infty $, we observe negligible differences in their effectiveness. This parity in performance echoes the insights shared in section \ref{subsection: convergence analysis}, suggesting that our algorithm maintains its robustness across a diverse range of filters. Those results confirm the practical applicability of diffusion-based spectral algorithms in addressing different types of regression challenges in semi-supervised learning contexts.

\section{Convergence Analysis}\label{section: proof of convergence result}

To establish the convergence behaviors for our estimator $ \tilde{f}_{D,\lambda} $, which integrates a truncation of the original spectral algorithm estimator $ f_{D,\lambda} $ with an approximation of the eigen-system of the empirical integral operator $ T_\delta $, we will approach the proof in a segmented manner. First, we will explore the impact of truncating the spectral algorithm estimator in Subsection \ref{subsection: truncation error analysis}, followed by examining the approximation error in Subsection \ref{subsection: approximation error analysis}. Finally, we will synthesize these analyses to establish the desired convergence result in subsection \ref{subsection: proof of convergence analysis thm}.

\subsection{Truncation Error Analysis}\label{subsection: truncation error analysis}

The truncated spectral algorithm estimator $ f_{D,\lambda}^T $ is defined as
\begin{equation}\label{definition of truncated spectral algorithm estimator}
	\begin{aligned}		
		f_{D,\lambda}^T
		&=\sum_{k=1}^q g_\lambda(\lambda_k)\langle g_D,\phi_k\rangle_{\mathcal{H}_t}\cdot\phi_k\\
		&=g_\lambda(T_\delta)g_D^T,
	\end{aligned}	
\end{equation}
where
\begin{equation}\label{defintion of g_D^T}
	g_D^T=\sum_{k=1}^q \langle g_D,\phi_k\rangle_{\mathcal{H}_t}\phi_k
\end{equation}
represents the truncation of $ g_D $. Given any $ \alpha>0 $, we focus on the $ \alpha $-power norm error $ \|f_{D,\lambda}^T-f^*\|_\alpha $. Recall that
\begin{equation}\nonumber
	f_{D,\lambda} = g_\lambda(T_\delta)g_D.
\end{equation}
We introduce
\begin{equation}\label{definition of hat f_D,lambda}
	\hat{f}_{D,\lambda}=g_\lambda(T_\delta)T_\delta f^*.
\end{equation}
This allows us to break down the original truncation error into three distinct components:
\begin{equation}\nonumber
	\begin{aligned}
		\|f_{D,\lambda}^T-f^*\|_\alpha
		&\leq\|f_{D,\lambda}^T-\hat{f}_{D,\lambda}\|_\alpha+\|\hat{f}_{D,\lambda}-f_{D,\lambda}\|_\alpha+\|f_{D,\lambda}-f^*\|_\alpha\\
		&\triangleq I_1+I_2+I_3.
	\end{aligned}
\end{equation}
Next, we will address these three terms individually.

\begin{lemma}\label{Lemma: upper bound for I_1 in truncation error analysis}
	Suppose that $ 0<\alpha\leq1 $, $ \tau\geq1 $. We choose $ q,\lambda $ as in \eqref{definition of truncated q} and \eqref{definition of lambda} seperately and $ m $ sufficiently large. Then, with probability larger than $ 1-10e^{-\tau} $, it holds
	\begin{equation}\nonumber
			I_1\leq K_1m^{-\frac{C_{low}}{C_{up}}\cdot\frac{1}{2}+\frac{\alpha}{4}}(\log m)^{\frac{1}{2}+\frac{d}{4}-\frac{\alpha d}{8}}
	\end{equation}
	with $K_1 $ a constant independent of $ m,\alpha $.
\end{lemma}

\begin{proof}
	Drawing from \cite{fischer2020sobolev}, we have that for any $ 0<\alpha\leq 1 $ and for any function $ f\in\mathcal{H}_t $,
	\begin{equation}\nonumber
		\|f\|_\alpha = \|T_\nu^\frac{1-\alpha}{2}f\|_{\mathcal{H}_t}.
	\end{equation}
	This leads to 
	\begin{equation}\nonumber
		\begin{aligned}
			\|f_{D,\lambda}^T-\hat{f}_{D,\lambda}\|_\alpha
			&=\|T_\nu^\frac{1-\alpha}{2}(f_{D,\lambda}^T-\hat{f}_{D,\lambda})\|_{\mathcal{H}_t}\\
			&\leq \|T_\nu^\frac{1-\alpha}{2}(T_\nu+\lambda I)^{-\frac{1}{2}}\|_{\mathscr{B}(\mathcal{H}_t)}\cdot\|(T_\nu+\lambda I)^{\frac{1}{2}}(f_{D,\lambda}^T-\hat{f}_{D,\lambda})\|_{\mathcal{H}_t}.\\
		\end{aligned}
	\end{equation}
	We provide an explanation for the symbols adopted here. Let $(\mathcal{H}, \langle\cdot,\cdot\rangle_{\mathcal{H}})$ and $(\mathcal{H}', \langle\cdot,\cdot\rangle_{\mathcal{H}'})$ be two Hilbert spaces. Therefore, the collection of bounded linear operators from $\mathcal{H}$ to $\mathcal{H}'$ forms a Banach space, symbolized as $\mathscr{B}(\mathcal{H}, \mathcal{H}')$ with operator norm $\|A\|_{\mathscr{B}(\mathcal{H}, \mathcal{H}')} = \sup_{\|f\|_{\mathcal{H}}=1} \|Af\|_{\mathcal{H}'}$. In the special case $\mathcal{H} = \mathcal{H}'$, we denote it as $\mathscr{B}(\mathcal{H})$ with the corresponding operator norm given by $\|A\|_{\mathscr{B}(\mathcal{H})}$.
	Furthermore,
	\begin{equation}\label{upper bound of alpha-to-H_t norm, proof I_1 upper bound lemma}
		\|T_\nu^\frac{1-\alpha}{2}(T_\nu+\lambda I)^{-\frac{1}{2}}\|_{\mathscr{B}(\mathcal{H}_t)}=\sup\limits_{k\geq1}\left(\frac{p^{1-\alpha}e^{-t\mu_k(1-\alpha)}}{pe^{-t\mu_k}+\lambda}
		\right)^\frac{1}{2}\leq\lambda^{-\frac{\alpha}{2}}.
	\end{equation}
	The last inequality in this sequence is derived from a straightforward calculation involving the function $t\mapsto t^{1-\alpha}/(\lambda+t)$.
	For the remaining term, we further decompose it as
	\begin{equation}\label{third term in lemma3}
	  \begin{split}
	   \|(T_\nu+\lambda I)^{\frac{1}{2}}(f_{D,\lambda}^T-\hat{f}_{D,\lambda})\|_{\mathcal{H}_t}
		&=\|(T_\nu+\lambda I)^{\frac{1}{2}}g_\lambda(T_\delta)(g_D^T-T_\delta f^*)\|_{\mathcal{H}_t}\\
		&\leq\|(T_\nu+\lambda I)^{\frac{1}{2}}(T_\delta+\lambda I)^{-\frac{1}{2}}\|_{\mathscr{B}(\mathcal{H}_t)}\\
		&\quad\cdot\|(T_\delta+\lambda I)^{\frac{1}{2}}g_\lambda(T_\delta)(T_\delta+\lambda I)^{\frac{1}{2}}\|_{\mathscr{B}(\mathcal{H}_t)}\\
		&\quad\cdot\|(T_\delta+\lambda I)^{-\frac{1}{2}}(g_D^T-T_\delta f^*)\|_{\mathcal{H}_t}.
	  \end{split}
	\end{equation}

	Also drawing from \cite{fischer2020sobolev}, it is established that, with probability larger than $ 1-2e^{-\tau} $, for sufficiently large $ m $,
	\begin{equation}\label{upper bound of empirical norm between T_nu and T_delta, proof I_1 upper bound lemma}
		\|(T_\nu+\lambda I)^{\frac{1}{2}}(T_\delta+\lambda I)^{-\frac{1}{2}}\|_{\mathscr{B}(\mathcal{H}_t)}\leq\sqrt{3}.
	\end{equation}
	Additionally, by leveraging the characteristics of the filter function $ g_\lambda $, as described in \eqref{property of regularization family}, we have
	\begin{equation}\label{upper bound of g_lambda(T_delta), proof I_1 upper bound lemma}
		\|(T_\delta+\lambda I)^{\frac{1}{2}}g_\lambda(T_\delta)(T_\delta+\lambda I)^{\frac{1}{2}}\|_{\mathscr{B}(\mathcal{H}_t)}=\|(T_\delta+\lambda I)g_\lambda(T_\delta)\|_{\mathscr{B}(\mathcal{H}_t)}\leq2.
	\end{equation}
	As for the last term \eqref{third term in lemma3}, we exploit the following decomposition 
	\begin{equation}\nonumber
		\begin{aligned}
			\|(T_\delta+\lambda I)^{-\frac{1}{2}}(g_D^T-T_\delta f^*)\|_{\mathcal{H}_t}
			&=\left\|\sum_{k=1}^q (\lambda_k+\lambda)^{-\frac{1}{2}}\left(\langle g_D,\phi_k\rangle_{\mathcal{H}_t}-\langle T_\delta f^*,\phi_k\rangle_{\mathcal{H}_t}\right)\phi_k\right\|_{\mathcal{H}_t}\\
			&\quad+\left\|\sum_{k=q+1}^\infty (\lambda_k+\lambda)^{-\frac{1}{2}}\langle T_\delta f^*,\phi_k\rangle_{\mathcal{H}_t}\phi_k\right\|_{\mathcal{H}_t}\\
			&=\left(
				\sum_{k=1}^q (\lambda_k+\lambda)^{-1}\left(\langle g_D,\phi_k\rangle_{\mathcal{H}_t}-\langle T_\delta f^*,\phi_k\rangle_{\mathcal{H}_t}\right)^2
			\right)^\frac{1}{2}\\
			&\quad+\left(\sum_{k=q+1}^\infty (\lambda_k+\lambda)^{-1}\lambda_k^2\langle f^*,\phi_k\rangle_{\mathcal{H}_t}^2
			\right)^\frac{1}{2}\\
			&\triangleq J_1+J_2,
		\end{aligned}
	\end{equation}
	where the second equation follows from $ T_\delta f^*=\sum\lambda_l\langle f^*,\phi_l\rangle\phi_l $. We consider $ J_1 $ first. Notice that
	\begin{equation}\nonumber
		\begin{aligned}
			\left|\langle g_D,\phi_k\rangle_{\mathcal{H}_t}-\langle T_\delta f^*,\phi_k\rangle_{\mathcal{H}_t}\right|
			&=\left|\langle g_D,\phi_k\rangle_{\mathcal{H}_t}-\lambda_k\langle f^*,\phi_k\rangle_{\mathcal{H}_t}\right|\\
			&=\sqrt{\lambda_k}\sqrt{\lambda_k}\left|\left\langle \frac{1}{\lambda_k}g_D,\phi_k\right\rangle_{\mathcal{H}_t}-\langle f^*,\phi_k\rangle_{\mathcal{H}_t}\right|.
		\end{aligned}
	\end{equation} 
	From \cite{guo2017thresholded} we have, with probability larger than $ 1-2e^{-\tau} $, for any $ k\in\mathbb{N} $,
	\begin{equation}\nonumber
		\sqrt{\lambda_k}\left|\left\langle \frac{1}{\lambda_k}g_D,\phi_k\right\rangle_{\mathcal{H}_t}-\langle f^*,\phi_k\rangle_{\mathcal{H}_t}\right|\leq3M\sqrt{\frac{\log m}{m}}.
	\end{equation}
	Therefore, a combination with \eqref{definition of truncated q} yields
	\begin{equation}\label{upper bound for J_1, proof I_1 upper bound lemma}
	  \begin{split}
		J_1
		&\leq\left(\sum_{k=1}^q(\lambda_k+\lambda)^{-1}\lambda_k\cdot9M^2\frac{\log m}{m}\right)^\frac{1}{2}\\
		&\leq\sqrt{q}\cdot3M\sqrt{\frac{\log m}{m}}\\
		&\lesssim m^{-\frac{1}{2}}(\log m)^{\frac{1}{2}+\frac{d}{4}}.
	\end{split}
\end{equation}
	Here, $ A\lesssim B $ implies that $A$ is less than or equal to $B$ up to a multiplication of constant.
	We now turn to $ J_2 $. According to \cite{guo2017thresholded} again, we have, with probability larger than $ 1-6e^{-\tau} $,
	\begin{equation}\nonumber
		J_2\leq C_J
		\left\{\min\left\{\frac{\Lambda_{m,q}}{\sqrt{\lambda}},\Lambda_{m,q}^\frac{1}{2}\right\}\gamma_q^r+\frac{\gamma_q}{(2\lambda+\gamma_q)^\frac{1}{2}}\frac{1}{\sqrt{m}}+2^r\Lambda_{m,q}^\frac{1}{2}\left[\left(\sum_{k=q+1}^\infty\gamma_k^{2r}\right)^\frac{1}{2}+\frac{8}{\sqrt{m}}\right]\right\}.
	\end{equation}
	Here, $ \Lambda_{m,q}=\max\left\{\frac{\gamma_q}{\kappa^2},\frac{1}{\sqrt{m}}\right\} $, $ \gamma_k=pe^{-t\mu_k} $ is an abbreviation, and $ C_J $ is a constant independent of $ m $. In the following discussions, we will estimate above three terms separately. Recall that 
	\begin{equation}\nonumber
		q = \left(\frac{\log m}{C_{up}t(2r+1)}\right)^\frac{d}{2},\;
		\lambda \sim \left(\frac{(\log m)^\frac{d}{2}}{m}\right)^\frac{1}{2}.
	\end{equation}
	Hence we have
	\begin{equation}\nonumber
		\Lambda_{m,q}\leq\max\left\{\frac{1}{\sqrt{m}},\frac{p}{\kappa^2}\left(\frac{1}{m}\right)^{\frac{C_{low}}{C_{up}}\frac{1}{2r+1}}\right\}\lesssim m^{-\frac{C_{low}}{C_{up}}\frac{1}{2r+1}},
	\end{equation}
	and
	\begin{equation}\nonumber
		\frac{\Lambda_{m,q}}{\sqrt{\lambda}}\lesssim (\log m)^{-\frac{d}{8}}\cdot m^{-\frac{C_{low}}{C_{up}}\frac{1}{2r+1}+\frac{1}{4}}.
	\end{equation}
	Since $ \frac{C_{low}}{C_{up}}\frac{1}{2r+1}<\frac{1}{2} $, we derive that
	\begin{equation}\nonumber
		\min\left\{\frac{\Lambda_{m,q}}{\sqrt{\lambda}},\Lambda_{m,q}^\frac{1}{2}\right\}\lesssim m^{-\frac{C_{low}}{C_{up}}\frac{1}{2r+1}\cdot\frac{1}{2}}.
	\end{equation}
	Therefore,
	\begin{equation}\label{upper bound for first term in J_2, proof I_1 upper bound lemma}
		\min\left\{\frac{\Lambda_{m,q}}{\sqrt{\lambda}},\Lambda_{m,q}^\frac{1}{2}\right\}\gamma_q^r
		\lesssim m^{-\frac{C_{low}}{C_{up}}\frac{1}{2r+1}\cdot\frac{1}{2}}\cdot m^{-\frac{C_{low}}{C_{up}}\frac{r}{2r+1}}\lesssim m^{-\frac{C_{low}}{C_{up}}\cdot\frac{1}{2}}.
	\end{equation}
	Further, from $ r>1 $ we have
	\begin{equation}\nonumber
		2\lambda+\gamma_q\gtrsim\gamma_q.
	\end{equation} 
	Hence
	\begin{equation}\label{upper bound for second term in J_2, proof I_1 upper bound lemma}
		\frac{\gamma_q}{(2\lambda+\gamma_q)^\frac{1}{2}}\frac{1}{\sqrt{m}}\lesssim \gamma_q^\frac{1}{2}\frac{1}{\sqrt{m}}\lesssim m^{-\frac{C_{low}}{C_{up}}\frac{1}{2r+1}\cdot\frac{1}{2}-\frac{1}{2}}.
	\end{equation}
	For the last term, similar to the calculation \eqref{caculation of Lambda_>l, proof of heat kernel matrix consistency Thm} for $ \Lambda_{>l} $ in the Appendix \ref{appendix: proof of eigen-system consistency of heat kernel matrix}, we have 
	\begin{equation}\nonumber
		\sum_{k=q+1}^\infty\gamma_k^{2r}=\sum_{k=q+1}^\infty (pe^{-\mu_kt})^{2r}\lesssim q (pe^{-\mu_qt})^{2r}\lesssim m^{-\frac{C_{low}}{C_{up}}\frac{2r}{2r+1}}(\log m)^\frac{d}{2},
	\end{equation}
	which results in 
	\begin{equation}\nonumber
		\left(\sum_{k=q+1}^\infty\gamma_k^{2r}\right)^\frac{1}{2}+\frac{8}{\sqrt{m}}\lesssim m^{-\frac{C_{low}}{C_{up}}\frac{r}{2r+1}}(\log m)^\frac{d}{4}+m^{-\frac{1}{2}}\lesssim m^{-\frac{C_{low}}{C_{up}}\frac{r}{2r+1}}(\log m)^\frac{d}{4}.
	\end{equation}
	Therefore,
	\begin{equation}\label{upper bound for third term in J_2, proof I_1 upper bound lemma}
		\begin{split}
			2^r\Lambda_{m,q}^\frac{1}{2}\left[\left(\sum_{k=q+1}^\infty\gamma_k^{2r}\right)^\frac{1}{2}+\frac{8}{\sqrt{m}}\right]
			&\lesssim m^{-\frac{C_{low}}{C_{up}}\frac{\frac{1}{2}}{2r+1}}\cdot m^{-\frac{C_{low}}{C_{up}}\frac{r}{2r+1}}(\log m)^\frac{d}{4}\\
			&\lesssim m^{-\frac{C_{low}}{C_{up}}\cdot\frac{1}{2}}(\log m)^\frac{d}{4}.
		\end{split}
	\end{equation}
	Therefore, a combination of \eqref{upper bound for first term in J_2, proof I_1 upper bound lemma}, \eqref{upper bound for second term in J_2, proof I_1 upper bound lemma}, and \eqref{upper bound for third term in J_2, proof I_1 upper bound lemma} yields
	\begin{equation}\label{upper bound for J_2, proof I_1 upper bound lemma}
		J_2\lesssim m^{-\frac{C_{low}}{C_{up}}\cdot\frac{1}{2}}(\log m)^\frac{d}{4}.
	\end{equation}
	Finally, combining \eqref{upper bound for J_1, proof I_1 upper bound lemma}, \eqref{upper bound for J_2, proof I_1 upper bound lemma} with \eqref{upper bound of alpha-to-H_t norm, proof I_1 upper bound lemma}, \eqref{upper bound of empirical norm between T_nu and T_delta, proof I_1 upper bound lemma}, and \eqref{upper bound of g_lambda(T_delta), proof I_1 upper bound lemma} we derive 
	\begin{equation}\nonumber
		\begin{aligned}
			\|f_{D,\lambda}^T-\hat{f}_{D,\lambda}\|_\alpha\
			&\lesssim \left(\frac{(\log m)^\frac{d}{2}}{m}\right)^{\frac{1}{2}\cdot{-\frac{\alpha}{2}}}\cdot\left(m^{-\frac{1}{2}}(\log m)^{\frac{1}{2}+\frac{d}{4}}+m^{-\frac{C_{low}}{C_{up}}\cdot\frac{1}{2}}(\log m)^\frac{d}{4}\right)\\
			&\lesssim m^{-\frac{C_{low}}{C_{up}}\cdot\frac{1}{2}+\frac{\alpha}{4}}(\log m)^{\frac{1}{2}+\frac{d}{4}-\frac{\alpha d}{8}},
		\end{aligned}
	\end{equation}
	which completes the proof.
\end{proof}

\begin{lemma}\label{Lemma: upper bound for I_2 in truncation error analysis}
	Suppose that $ 0<\alpha\leq1 $, $ \tau\geq1 $, $\lambda $ chosen as in \eqref{definition of lambda}, and $ m $ sufficiently large. Then, with probability at least $ 1-6e^{-\tau} $, it holds
	\begin{equation}\nonumber
		I_2\leq K_2m^{-\frac{1}{2}+\frac{\alpha}{4}}(\log m)^{\frac{d}{4}-\frac{\alpha d}{8}},
	\end{equation}
	where $K_2 $ is a constant independent of $ m$ or $\alpha $.
\end{lemma}

\begin{proof}
	From \eqref{definition of spectal alogrithm estimator f_D,lambda} and \eqref{definition of hat f_D,lambda} we have
	\begin{equation}\nonumber
		f_{D,\lambda}-\hat{f}_{D,\lambda}=g_\lambda(T_\delta)(g_D-T_\delta f^*).
	\end{equation}
	Therefore,
	\begin{equation}\nonumber
		\begin{aligned}
			\|f_{D,\lambda}-\hat{f}_{D,\lambda}\|_\alpha
			&=\|T_\nu^\frac{1-\alpha}{2}(f_{D,\lambda}-\hat{f}_{D,\lambda})\|_{\mathcal{H}_t}\\
			&\leq\|T_\nu^\frac{1-\alpha}{2}(T_\nu+\lambda I)^{-\frac{1}{2}}\|_{\mathscr{B}(\mathcal{H}_t)}\cdot\|(T_\nu+\lambda I)^{\frac{1}{2}}(T_\delta+\lambda I)^{-\frac{1}{2}}\|_{\mathscr{B}(\mathcal{H}_t)}\\
			&\quad\cdot\|(T_\delta+\lambda I)^{\frac{1}{2}}(f_{D,\lambda}-\hat{f}_{D,\lambda})\|_{\mathcal{H}_t}.
		\end{aligned}
	\end{equation}
	It is remarkable that the first two terms have already been bounded in \eqref{upper bound of alpha-to-H_t norm, proof I_1 upper bound lemma} and \eqref{upper bound of empirical norm between T_nu and T_delta, proof I_1 upper bound lemma}. For the remaining term, we further decompose it as
	\begin{equation}\nonumber
		\begin{aligned}
			\|(T_\delta+\lambda I)^{\frac{1}{2}}(f_{D,\lambda}-\hat{f}_{D,\lambda})\|_{\mathcal{H}_t}
			&=\|(T_\delta+\lambda I)^{\frac{1}{2}}g_\lambda(T_\delta)(g_D-T_\delta f^*)\|_{\mathcal{H}_t}\\
			&\leq\|(T_\delta+\lambda I)^{\frac{1}{2}}g_\lambda(T_\delta)(T_\delta+\lambda I)^{\frac{1}{2}}\|_{\mathscr{B}(\mathcal{H}_t)}\\
			&\quad\cdot\|(T_\delta+\lambda I)^{-\frac{1}{2}}(T_\nu+\lambda I)^{\frac{1}{2}}\|_{\mathscr{B}(\mathcal{H}_t)}\\
			&\quad\cdot\|(T_\nu+\lambda I)^{-\frac{1}{2}}(g_D-T_\delta f^*)\|_{\mathcal{H}_t}\\
		\end{aligned}
	\end{equation}
	with the first two terms also discussed in \eqref{upper bound of g_lambda(T_delta), proof I_1 upper bound lemma} and \eqref{upper bound of empirical norm between T_nu and T_delta, proof I_1 upper bound lemma}. To facilitate the third term above, we introduce a continuous version of $ f_{D,\lambda} $ as
	\begin{equation}\label{definition of f_P,lambda}
		f_{P,\lambda}= g_\lambda(T_\nu)g_P.
	\end{equation}
	Here,
	\begin{equation}\nonumber
		g_P\triangleq I_\nu^* f^*=\int_{\mathcal{M}\times\mathbb{R}}yH_t(x,\cdot)d\mathcal{P}(x,y).
	\end{equation}
	As a result, we further control it by
	\begin{equation}\nonumber
		\begin{aligned}
			\|(T_\nu+\lambda I)^{-\frac{1}{2}}(g_D-T_\delta f^*)\|_{\mathcal{H}_t}
			&\leq\|(T_\nu+\lambda I)^{-\frac{1}{2}}(g_D-T_\delta f_{P,\lambda})\|_{\mathcal{H}_t}\\
			&\quad+\|(T_\nu+\lambda I)^{-\frac{1}{2}}T_\delta(f_{P,\lambda}-f^*)\|_{\mathcal{H}_t}.
		\end{aligned}
	\end{equation}
	Drawing from Lemma 10 of \cite{xia2024spectral} with $ \beta=2 $, it is shown that with probability larger than $ 1-2e^{-\tau} $, 
	\begin{equation}\label{cited lemma 10, proof of I_2 upper bound lemma}
		\left\|(T_\nu+\lambda I)^{-\frac{1}{2}}\left(g_D-T_\delta f_{P,\lambda}\right)\right\|_{\mathcal{H}_t}^2\lesssim\frac{1}{m}\left(N_\nu(\lambda)+\lambda^{2-\alpha}+\frac{L^2_\lambda}{m\lambda^\alpha}\right)+\lambda^2.
	\end{equation}
	Here, $ N_\nu(\lambda)=\text{tr}\left((T_\nu+\lambda I)^{-1}T_\nu\right)$ and $ L_\lambda=\max\{M,\|f_{P,\lambda}-f^*\|_{\infty}\} $. Notably, while the original Lemma 10 in \cite{xia2024spectral} is formulated for $ \beta\leq1 $, it can be adapted for $ \beta=2 $, an extension that will be detailed in the Appendix \ref{appendix: complementary discussion of proof in main theorem}. Employing methodologies akin to those in the proof of Theorem 1 in \cite{xia2024spectral}, with $ \lambda\sim \left(\frac{(\log m)^\frac{d}{2}}{m}\right)^\frac{1}{2} $ corresponding to $ \beta=2 $ (Again note that the original theorem addressed $ \beta\leq1 $, which yet can be extended to $ \beta=2 $. We will formulate this in Appendix \ref{appendix: complementary discussion of proof in main theorem}), which now results in
	\begin{equation}\label{upper bound for the first term, proof of I_2 upper bound lemma}
		\left\|(T_\nu+\lambda I)^{-\frac{1}{2}}\left(g_D-T_\delta f_{P,\lambda}\right)\right\|_{\mathcal{H}_t}\lesssim \lambda.
	\end{equation}
	We decompose the remaining term as
	\begin{equation}\nonumber
		\begin{aligned}
			\|(T_\nu+\lambda I)^{-\frac{1}{2}}T_\delta(f_{P,\lambda}-f^*)\|_{\mathcal{H}_t}
			&\leq\|(T_\nu+\lambda I)^{-\frac{1}{2}}(T_\delta-T_\nu)(f_{P,\lambda}-f^*)\|_{\mathcal{H}_t}\\
			&\quad+\|(T_\nu+\lambda I)^{-\frac{1}{2}}T_\nu(f_{P,\lambda}-f^*)\|_{\mathcal{H}_t}\\
			&\triangleq J_1+J_2.
		\end{aligned}
	\end{equation}
	For $ J_1 $, we have
	\begin{equation}\nonumber
		J_1\leq\|(T_\nu+\lambda I)^{-\frac{1}{2}}\|_{\mathscr{B}(\mathcal{H}_t)}\cdot\|T_\delta-T_\nu\|_{\mathscr{B}(\mathcal{H}_t)}\cdot\|f_{P,\lambda}-f^*\|_{\mathcal{H}_t}.
	\end{equation}
	A direct computation yields
	\begin{equation}\nonumber
		\|(T_\nu+\lambda I)^{-\frac{1}{2}}\|_{\mathscr{B}(\mathcal{H}_t)}=\sup\limits_{k\geq1}\left(\frac{1}{pe^{-t\mu_k}+\lambda}\right)^\frac{1}{2}\leq\lambda^{-\frac{1}{2}}.
	\end{equation}
	Drawing from \cite{guo2017thresholded} we have, with probability larger than $ 1-2e^{-\tau} $,
	\begin{equation}\nonumber
		\|T_\delta-T_\nu\|_{\mathscr{B}(\mathcal{H}_t)}\leq\|T_\delta-T_\nu\|_{HS}\leq C_{h}\tau m^{-\frac{1}{2}}.
	\end{equation}
	with $ \|\cdot\|_{HS} $ the Hilbert-Schmidt norm and $ C_h $ a constant independent of $ m,\tau $. Moreover, \cite{fischer2020sobolev} proves that
	\begin{equation}\nonumber
		\|f_{P,\lambda}-f^*\|_{\mathcal{H}_t}\leq \|f^*\|_2\lambda^\frac{1}{2}.
	\end{equation}
	Combining the above yields
	\begin{equation}\label{upper bound of J_1, proof of I_2 upper bound lemma}
		J_1\lesssim m^{-\frac{1}{2}}.
	\end{equation}
	As for $ J_2 $, we have
	\begin{equation}\nonumber
		\begin{aligned}
			J_2
			&=\|(T_\nu+\lambda I)^{-\frac{1}{2}}T_\nu(f_{P,\lambda}-f^*)\|_{\mathcal{H}_t}\\
			&=\|(T_\nu+\lambda I)^{-\frac{1}{2}}I_\nu^*I_\nu(f_{P,\lambda}-f^*)\|_{\mathcal{H}_t}\\
			&\leq\|(T_\nu+\lambda I)^{-\frac{1}{2}}I_\nu^*\|_{\mathscr{B}(L^2(\nu),\mathcal{H}_t)}\cdot\|f_{P,\lambda}-f^*\|_{L^2(\nu)}.\\
		\end{aligned}
	\end{equation}
	Notic that
	\begin{equation}\nonumber
		\|(T_\nu+\lambda I)^{-\frac{1}{2}}I_\nu^*\|_{\mathscr{B}(L^2(\nu),\mathcal{H}_t)}=\sup_{k\geq0}\left(\frac{pe^{-t\mu_k}}{pe^{-t\mu_k}+\lambda}\right)^\frac{1}{2}\leq1,
	\end{equation}
	and \cite{fischer2020sobolev} shows that
	\begin{equation}\nonumber
		\|f_{P,\lambda}-f^*\|_{L^2(\nu)}\leq \|f^*\|_2\lambda.
	\end{equation}
	Therefore, 
	\begin{equation}\label{upper bound of J_2, proof of I_2 upper bound lemma}
		J_2\lesssim\lambda.
	\end{equation}
	Finally, combining \eqref{upper bound for the first term, proof of I_2 upper bound lemma}, \eqref{upper bound of J_1, proof of I_2 upper bound lemma}, and \eqref{upper bound of J_2, proof of I_2 upper bound lemma} with former results \eqref{upper bound of alpha-to-H_t norm, proof I_1 upper bound lemma}, \eqref{upper bound of empirical norm between T_nu and T_delta, proof I_1 upper bound lemma}, and \eqref{upper bound of g_lambda(T_delta), proof I_1 upper bound lemma}, we derive that
	\begin{equation}\nonumber
		\|f_{D,\lambda}-\hat{f}_{D,\lambda}\|_\alpha\lesssim\lambda^{-\frac{\alpha}{2}}\cdot\left(m^{-\frac{1}{2}}+\lambda\right)\lesssim\lambda^{1-\frac{\alpha}{2}}\lesssim m^{-\frac{1}{2}+\frac{\alpha}{4}}(\log m)^{\frac{d}{4}-\frac{\alpha d}{8}},
	\end{equation}
	which completes the proof.
\end{proof}

Utilizing Theorem 1 in \cite{xia2024spectral} with $ \beta=2 $ and $ \gamma=\alpha $ (Similarly, we have to extend the original result from $ \beta\leq1 $ to $ \beta=2 $, same as in the derivation of \eqref{upper bound for the first term, proof of I_2 upper bound lemma}), we have

\begin{lemma}\label{Lemma: upper bound for I_3 in truncation error analysis}
	Suppose that $ 0<\alpha\leq1 $, $ \tau\geq1 $, $\lambda $ chosen as in \eqref{definition of lambda}, $ m $ sufficiently large. Then, with probability at least $ 1-4e^{-\tau} $, it holds
	\begin{equation}\nonumber
		I_3\leq K_3m^{-\frac{1}{2}+\frac{\alpha}{4}}(\log m)^{\frac{d}{4}-\frac{\alpha d}{8}},
	\end{equation}
	where $K_3 $ is a constant independent of $ m,\alpha $.
\end{lemma}

Finally, a combination of above Lemma \ref{Lemma: upper bound for I_1 in truncation error analysis}, Lemma \ref{Lemma: upper bound for I_2 in truncation error analysis}, and Lemma \ref{Lemma: upper bound for I_3 in truncation error analysis} yields the following result:
\begin{equation}\nonumber
	\begin{aligned}
		\|f_{D,\lambda}^T-f^*\|_\alpha
		&\lesssim m^{-\frac{C_{low}}{C_{up}}\cdot\frac{1}{2}+\frac{\alpha}{4}}(\log m)^{\frac{1}{2}+\frac{d}{4}-\frac{\alpha d}{8}}+m^{-\frac{1}{2}+\frac{\alpha}{4}}(\log m)^{\frac{d}{4}-\frac{\alpha d}{8}}\\
		&\lesssim m^{-\frac{C_{low}}{C_{up}}\cdot\frac{1}{2}+\frac{\alpha}{4}}(\log m)^{\frac{1}{2}+\frac{d}{4}-\frac{\alpha d}{8}},
	\end{aligned}
\end{equation}
which is valid for all $ 0<\alpha\leq1 $. Drawing from \cite{xia2024spectral}, $ \mathcal{H}_t^\alpha\hookrightarrow L^\infty(\nu) $ holds for all $ \alpha>0 $. Therefore, the following Theorem \ref{Thm: upper bound for truncation error} is a direct corollary of above results.

\begin{theorem}\label{Thm: upper bound for truncation error}
	Suppose that Assumption \ref{Assumption: source condition on regression function}, Assumption \ref{Assumption: boundedness condition on regression function}, and Assumption \ref{Assumption: qualification condition on filter function} hold. We choose $ q,\lambda $ as in \eqref{definition of truncated q} and \eqref{definition of lambda} separately. Then, for any $ 0<\alpha\leq1 $ and $ \tau\geq1 $, when $ m $ large enough, there exists an event $ E_3 $ with probability larger than $ 1-20e^{-\tau} $ such that, on $ E_3 $  it holds
	\begin{equation}\nonumber
		\|f_{D,\lambda}^T-f^*\|_\infty\leq K_tm^{-\frac{C_{low}}{C_{up}}\cdot\frac{1}{2}+\frac{\alpha}{4}}(\log m)^{\frac{1}{2}+\frac{d}{4}-\frac{\alpha d}{8}}.
	\end{equation}
	Here, $ K_t $ is a constant independent of $ m,\alpha $.
\end{theorem}

\subsection{Approximation Error Analysis}\label{subsection: approximation error analysis}

In this subsection, we establish the point-wise upper bound between the truncated spectral algorithm estimator $  f_{D,\lambda}^T $, as defined in \eqref{definition of truncated spectral algorithm estimator}, and our diffusion-based spectral algorithm estimator $ \tilde{f}_{D,\lambda} $. For any $ 1\leq j\leq N $, we begin by decomposing the error into distinct components:
\begin{equation}\nonumber
	\begin{aligned}
		|f_{D,\lambda}^T(x_j)-\tilde{f}_{D,\lambda}(j)|
		&=\left|\sum_{k=1}^q g_\lambda(\lambda_k)\langle g_D,\phi_k\rangle_{\mathcal{H}_t}\phi_k(x_j)-\sum_{k=1}^q g_\lambda(\tilde{\lambda}_k)\tilde{g}_{D,k}\tilde{\phi}_k(j)\right|\\
		&\leq\left|\sum_{k=1}^q \left(g_\lambda(\lambda_k)-g_\lambda(\tilde{\lambda}_k)\right)\langle g_D,\phi_k\rangle_{\mathcal{H}_t}\phi_k(x_j)\right|\\
		&\quad+\left|\sum_{k=1}^q g_\lambda(\tilde{\lambda}_k)\left(\langle g_D,\phi_k\rangle_{\mathcal{H}_t}\phi_k(x_j)-\tilde{g}_{D,k}\tilde{\phi}_k(j)\right)\right|\\
		&\triangleq I_1+I_2.
	\end{aligned}
\end{equation}
We will first address the upper bound for $ I_1 $, our approach and findings are presented in the following Lemma \ref{Lemma: upper bound for I_1 in approximation error analysis}.

\begin{lemma}\label{Lemma: upper bound for I_1 in approximation error analysis}
	Suppose that Assumption \ref{Assumption: source condition on regression function}, Assumption \ref{Assumption: boundedness condition on regression function}, and Assumption \ref{Assumption: qualification condition on filter function} hold. Suppose that the conditions \eqref{conditions on heat kernel estimation Thm} holds with $ m $ sufficiently large, $ N>K\gg m $ sufficiently large and $ \epsilon $ sufficiently small. We choose $ q,\lambda $ as in \eqref{definition of truncated q} and \eqref{definition of lambda} separately. Denote $ E_1 $ the  event given in Theorem \ref{Thm: L-infty eigen-system approxiamtion of graph Laplacian}, $ E_2 $ the event given in Theorem \ref{Thm: eigen-system consistency of heat kernel matrix} and $ E_3 $ the event given in Theorem \ref{Thm: upper bound for truncation error}. Then, on $ E_1\bigcap E_2\bigcap E_3 $, for any $ 0<\alpha\leq1$ and $ 1\leq j\leq N $, it holds
	\begin{equation}\nonumber
			I_1\leq K_1'\left(m^{-\frac{\xi}{2}+\frac{1}{2r+1}(\xi+\frac{1}{2})+\frac{\alpha}{4}}\cdot(\log m)^{\frac{d\xi}{4}-\frac{d\alpha}{8}}+m^{\frac{1}{2r+1}\cdot\frac{3}{2}+\frac{\alpha}{4}}\cdot(\log m)^{-\frac{d\alpha}{8}}\cdot\left(\frac{1}{K}+\epsilon^{\frac{1}{4}}\right)\right),
	\end{equation}
	where $ K_1' $ is a constant independent of $ m,K,N,\epsilon,\alpha$.
\end{lemma}

\begin{proof}
	Recalling the empirical eigen-system as defined in \eqref{eigen-system of T_delta}, we obtain the spectral decomposition for $ T_\delta $:
	\begin{equation}\nonumber
		T_\delta = \sum_{k=1}^\infty \lambda_k\langle \cdot,\phi_k\rangle_{\mathcal{H}_t}\phi_k.
	\end{equation} 
	Subsequently, we introduce
	\begin{equation}\nonumber
		\tilde{T}_\delta=\sum_{k=1}^\infty \tilde{\lambda}_k\langle \cdot,\phi_k\rangle_{\mathcal{H}_t}\phi_k,
	\end{equation}
	where $ \tilde{\lambda}_k=0 $ for $ k>m $. Then we have
	\begin{equation}\nonumber
			I_1=\left|g_\lambda(T_\delta)g_D^T(x_j)-g_\lambda(\tilde{T}_\delta)g_D^T(x_j)\right|\leq\left\|\left(g_\lambda(T_\delta)-g_\lambda(\tilde{T}_\delta)\right)g_D^T\right\|_\infty.
	\end{equation}
	Adopting a strategy parallel to the one adopted in Subsection \ref{subsection: truncation error analysis}, we commence our analysis by examining the $ \|\cdot\|_\alpha $-norm bound, subsequently deriving the $ \|\cdot\|_\infty $-norm bound through embedding results. Initially, we break down the error into
	\begin{equation}\nonumber
		\begin{aligned}
			\left\|\left(g_\lambda(T_\delta)-g_\lambda(\tilde{T}_\delta)\right)g_D^T\right\|_\alpha
			&= \left\|T_\nu^\frac{1-\alpha}{2}\left(g_\lambda(T_\delta)-g_\lambda(\tilde{T}_\delta)\right)g_D^T\right\|_{\mathcal{H}_t}\\
			&\leq\|T_\nu^\frac{1-\alpha}{2}(T_\nu+\lambda I)^{-\frac{1}{2}}\|_{\mathscr{B}(\mathcal{H}_t)}\\
			&\quad\cdot\left\|(T_\nu+\lambda I)^{\frac{1}{2}}\left(g_\lambda(T_\delta)-g_\lambda(\tilde{T}_\delta)\right)g_D^T\right\|_{\mathcal{H}_t},
		\end{aligned}
	\end{equation}
	where the first term is bounded by $ \lambda^{-\frac{\alpha}{2}} $ as in \eqref{upper bound of alpha-to-H_t norm, proof I_1 upper bound lemma}. Now, let's introduce several abbreviations:
	\begin{equation}\nonumber
		\begin{aligned}
			f_1&=g_\lambda(T_\delta)g_D^T,\\
			f_2&=g_\lambda(\tilde{T}_\delta)g_D^T,\\
			h_\lambda(t)&=1-tg_\lambda(t).
		\end{aligned}
	\end{equation}
	Therefore,
	\begin{equation}\nonumber
		\begin{aligned}
			f_1-f_2
			&=f_1-\left(h_\lambda(T_\delta)+T_\delta g_\lambda(T_\delta)\right)f_2\\
			&=g_\lambda(T_\delta)(g_D^T-T_\delta f_2)-h_\lambda(T_\delta)f_2.
		\end{aligned}
	\end{equation}
	As a consequence, we further decompose the remaining term as
	\begin{equation}\nonumber
		\begin{aligned}
			\left\|(T_\nu+\lambda I)^{\frac{1}{2}}\left(g_\lambda(T_\delta)-g_\lambda(\tilde{T}_\delta)\right)g_D^T\right\|_{\mathcal{H}_t}
			&\leq\left\|(T_\nu+\lambda I)^{\frac{1}{2}}g_\lambda(T_\delta)(g_D^T-T_\delta f_2)\right\|_{\mathcal{H}_t}\\
			&\quad+\left\|(T_\nu+\lambda I)^{\frac{1}{2}}h_\lambda(T_\delta)f_2\right\|_{\mathcal{H}_t}\\
			&\triangleq J_1+J_2.
		\end{aligned}
	\end{equation}
	We consider term $ J_2 $ first. Since $ g_\lambda $ has qualification $ \xi\geq1 $, 
	\begin{equation}\label{calculation on h_lambda, proof I_1 upper bound lemma, approxiamtion error analysis}
		(t+\lambda)^\xi h_\lambda(t)\leq 2^{\xi-1}\left(t^\xi h_\lambda(t)+\lambda^\xi h_\lambda(t)\right)\leq 2^\xi\lambda^\xi.
	\end{equation}  
	Hence 
	\begin{equation}\nonumber
		\begin{aligned}
			J_2
			&\leq\left\|(T_\nu+\lambda I)^{\frac{1}{2}}(T_\delta+\lambda I)^{-\frac{1}{2}}\right\|_{\mathscr{B}(\mathcal{H}_t)}\\
			&\quad\cdot\left\|(T_\delta+\lambda I)^{\frac{1}{2}}h_\lambda(T_\delta)(T_\delta+\lambda I)^{\xi-\frac{1}{2}}\right\|_{\mathscr{B}(\mathcal{H}_t)}\cdot\left\|(T_\delta+\lambda I)^{\frac{1}{2}-\xi}f_2\right\|_{\mathcal{H}_t},
		\end{aligned}
	\end{equation}
	where the first term above has been bounded in \eqref{upper bound of empirical norm between T_nu and T_delta, proof I_1 upper bound lemma} and the second term above is bounded by above calculation \eqref{calculation on h_lambda, proof I_1 upper bound lemma, approxiamtion error analysis}. 
	For the third term, we have
	\begin{equation}\nonumber
		\begin{aligned}
			\left\|(T_\delta+\lambda I)^{\frac{1}{2}-\xi}f_2\right\|_{\mathcal{H}_t}
			&=\left\|(T_\delta+\lambda I)^{\frac{1}{2}-\xi}g_\lambda(\tilde{T}_\delta)g_D^T\right\|_{\mathcal{H}_t}\\
			&=\left\|\sum_{k=1}^q (\lambda+\lambda_k)^{\frac{1}{2}-\xi}g_\lambda(\tilde{\lambda}_k)\langle g_D,\phi_k\rangle_{\mathcal{H}_t}\phi_k\right\|_{\mathcal{H}_t}\\
			&=\left(\sum_{k=1}^q (\lambda+\lambda_k)^{1-2\xi}g_\lambda(\tilde{\lambda}_k)^2|\langle g_D,\phi_k\rangle_{\mathcal{H}_t}|^2\right)^\frac{1}{2}\\
			&\leq\left(\sum_{k=1}^q \lambda_k^{1-2\xi}\tilde{\lambda}_k^{-2}|\langle g_D,\phi_k\rangle_{\mathcal{H}_t}|^2\right)^\frac{1}{2}.
		\end{aligned}
	\end{equation}
	Since \eqref{relation between eigen-system of T_delta and heat kernel matrix} and \eqref{eigenvalue consistency of heat kernel matrix} yields
	\begin{equation}\nonumber
		|\tilde{\lambda}_k-\lambda_k|\lesssim \frac{1}{K}+\epsilon^{\frac{1}{4}},
	\end{equation}
	whose convergence is independent of $ m $. Therefore, the above term can be bounded as 
	\begin{equation}\nonumber
		\begin{aligned}
			\left\|(T_\delta+\lambda I)^{\frac{1}{2}-\xi}f_2\right\|_{\mathcal{H}_t}
			&\lesssim\left(\sum_{k=1}^q \lambda_k^{1-2\xi}\lambda_k^{-2}|\langle g_D,\phi_k\rangle_{\mathcal{H}_t}|^2\right)^\frac{1}{2}\\
			&\leq \lambda_q^{-\frac{1+2\xi}{2}}\left(\sum_{k=1}^q| \langle g_D,\phi_k\rangle_{\mathcal{H}_t}|^2\right)^\frac{1}{2}\\
			&\lesssim m^{\frac{1}{2r+1}\cdot\frac{1+2\xi}{2}},
		\end{aligned}
	\end{equation}
	where the last inequality follows from the definition of $ g_D $ as provided in \eqref{defintion of g_D}, in conjunction with the boundedness condition and the lower bound result \eqref{lower bound of hat lambda_k, proof of eigenfunction estimation lemma} on $ \lambda_q $.
	Combining the above yields
\begin{equation}\label{upper bound for J_2, proof of I_1 upper bound lemma, approximation error analysis}
	\begin{split}
		J_2&\lesssim \lambda^\xi\cdot\lambda_q^{-\frac{1+2\xi}{2}}\\
		&\lesssim \left(\frac{(\log m)^\frac{d}{2}}{m}\right)^\frac{\xi}{2}\cdot m^{\frac{1}{2r+1}\cdot\frac{1+2\xi}{2}}\\
		&\lesssim m^{-\frac{\xi}{2}+\frac{1}{2r+1}(\xi+\frac{1}{2})}\cdot(\log m)^\frac{d\xi}{4}.
	\end{split}
\end{equation}
	Note that $ -\frac{\xi}{2}+\frac{1}{2r+1}(\xi+\frac{1}{2})=-\frac{1}{2r+1}((r-\frac{1}{2})\xi-\frac{1}{2})<0 $ for $ \xi\geq1 $ and $ r>1 $. 
	Then we return to term $ J_1 $. Notice that
	\begin{equation}\nonumber
		\begin{aligned}
			J_1
			&=\left\|(T_\nu+\lambda I)^{\frac{1}{2}}g_\lambda(T_\delta)(g_D^T-T_\delta g_\lambda(\tilde{T}_\delta)g_D^T)\right\|_{\mathcal{H}_t}\\
			&=\left\|(T_\nu+\lambda I)^{\frac{1}{2}}g_\lambda(T_\delta)\left(I-T_\delta g_\lambda(\tilde{T}_\delta)\right)g_D^T\right\|_{\mathcal{H}_t}\\
			&\leq \left\|(T_\nu+\lambda I)^{\frac{1}{2}}(T_\delta+\lambda I)^{-\frac{1}{2}}\right\|_{\mathscr{B}(\mathcal{H}_t)}\cdot\left\|(T_\delta+\lambda I)^{\frac{1}{2}}g_\lambda(T_\delta)(T_\delta+\lambda I)^{\frac{1}{2}}\right\|_{\mathscr{B}(\mathcal{H}_t)}\\
			&\quad\cdot\left\|(T_\delta+\lambda I)^{-\frac{1}{2}}\left(I-T_\delta g_\lambda(\tilde{T}_\delta)\right)g_D^T\right\|_{\mathcal{H}_t},
		\end{aligned}
	\end{equation}
	with the first and the second terms already bounded in \eqref{upper bound of empirical norm between T_nu and T_delta, proof I_1 upper bound lemma} and \eqref{upper bound of g_lambda(T_delta), proof I_1 upper bound lemma}. For the remaining term, since
	\begin{equation}\nonumber
		\begin{aligned}
			I-T_\delta g_\lambda(\tilde{T}_\delta)
			&=I-\tilde{T}_\delta g_\lambda(\tilde{T}_\delta)+\tilde{T}_\delta g_\lambda(\tilde{T}_\delta)-T_\delta g_\lambda(\tilde{T}_\delta)\\
			&=h_\lambda(\tilde{T}_\delta)+(\tilde{T}_\delta-T_\delta)g_\lambda(\tilde{T}_\delta),
		\end{aligned}
	\end{equation}
	a further decomposition it derived as 
	\begin{equation}\nonumber
		\begin{aligned}
			\left\|(T_\delta+\lambda I)^{-\frac{1}{2}}\left(I-T_\delta g_\lambda(\tilde{T}_\delta)\right)g_D^T\right\|_{\mathcal{H}_t}
			&\leq\left\|(T_\delta+\lambda I)^{-\frac{1}{2}}h_\lambda(\tilde{T}_\delta)g_D^T\right\|_{\mathcal{H}_t}\\
			&\quad+\left\|(T_\delta+\lambda I)^{-\frac{1}{2}}(\tilde{T}_\delta-T_\delta)g_\lambda(\tilde{T}_\delta)g_D^T\right\|_{\mathcal{H}_t}\\
			&\triangleq L_1+L_2.
		\end{aligned}
	\end{equation}
	For $ L_1 $ we have
	\begin{equation}\nonumber
		L_1\leq\left\|(T_\delta+\lambda I)^{-\frac{1}{2}}(\tilde{T}_\delta+\lambda I)^{\frac{1}{2}}\right\|_{\mathscr{B}(\mathcal{H}_t)}\cdot\left\|(\tilde{T}_\delta+\lambda I)^{-\frac{1}{2}}h_\lambda(\tilde{T}_\delta)g_D^T\right\|_{\mathcal{H}_t}.
	\end{equation}
	It's remarkable that our analysis of the operator norm can be specifically confined to the subspace $ L_q=\text{Span}\{\phi_k:1\leq k\leq q\} $ , given that $ g_D^T\in L_q $ and all operators under consideration exhibit $ L_q $-invariance. Therefore,
	\begin{equation}\nonumber
		\begin{aligned}
			\left\|(T_\delta+\lambda I)^{-\frac{1}{2}}(\tilde{T}_\delta+\lambda I)^{\frac{1}{2}}\right\|_{\mathscr{B}(L_q)}^2
			&=\sup_{1\leq k\leq q}\frac{\lambda_k+\lambda}{\tilde{\lambda}_k+\lambda}\\
			&=1 + \sup_{1\leq k\leq q}\frac{\lambda_k-\tilde{\lambda}_k}{\tilde{\lambda}_k+\lambda}\\
			&\leq 1+\tilde{\lambda}_q^{-1}\cdot\sup_{1\leq k\leq q}|\lambda_k-\tilde{\lambda}_k|\\
			&\lesssim 1 + m^\frac{1}{2r+1}\left(\frac{1}{K}+\epsilon^{\frac{1}{4}}\right)\lesssim 2.
		\end{aligned}
	\end{equation}
	For the remaining $ \mathcal{H}_t $-norm, a similar approach as in the upper bound analysis of $ J_2 $ yields
	\begin{equation}\nonumber
		\begin{aligned}
			\left\|(\tilde{T}_\delta+\lambda I)^{-\frac{1}{2}}h_\lambda(\tilde{T}_\delta)g_D^T\right\|_{\mathcal{H}_t}
			&\leq\left\|(\tilde{T}_\delta+\lambda I)^{-\frac{1}{2}}h_\lambda(\tilde{T}_\delta)(\tilde{T}_\delta+\lambda I)^{\frac{1}{2}+\xi}\right\|_{\mathscr{B}(\mathcal{H}_t)}\\
			&\quad\cdot \left\|(\tilde{T}_\delta+\lambda I)^{-\frac{1}{2}-\xi}g_D^T\right\|_{\mathcal{H}_t},\\
		\end{aligned}
	\end{equation}
	with the first term also bounded by calculation \eqref{calculation on h_lambda, proof I_1 upper bound lemma, approxiamtion error analysis}. For the second term, we have
	\begin{equation}\nonumber
		\begin{aligned}
			\left\|(\tilde{T}_\delta+\lambda I)^{-\frac{1}{2}-\xi}g_D^T\right\|_{\mathcal{H}_t}
			&=\left(\sum_{k=1}^q (\lambda+\tilde{\lambda}_k)^{-1-2\xi}|\langle g_D,\phi_k\rangle_{\mathcal{H}_t}|^2\right)^\frac{1}{2}\\
			&\lesssim \lambda_q^{-\frac{1+2\xi}{2}}\left(\sum_{k=1}^q |\langle g_D,\phi_k\rangle_{\mathcal{H}_t}|^2\right)^\frac{1}{2}\\
			&\lesssim m^{\frac{1}{2r+1}\cdot\frac{1+2\xi}{2}}.
		\end{aligned}
	\end{equation}
	Combining the above yields
	\begin{equation}\label{upper bound for L_1, proof of I_1 upper bound lemma, approximation error analysis}
		L_1\lesssim\lambda^\xi \lambda_q^{-\frac{1+2\xi}{2}}\lesssim m^{-\frac{\xi}{2}+\frac{1}{2r+1}(\xi+\frac{1}{2})}\cdot(\log m)^\frac{d\xi}{4}.
	\end{equation}
	Now we consider $ L_2 $. It can be decomposed as
	\begin{equation}\nonumber
		L_2\leq\left\|(T_\delta+\lambda I)^{-\frac{1}{2}}
		\right\|_{\mathscr{B}(\mathcal{H}_t)}\cdot\left\|\tilde{T}_\delta-T_\delta\right\|_{\mathscr{B}(\mathcal{H}_t)}\cdot\left\|g_\lambda(\tilde{T}_\delta)g_D^T\right\|_{\mathcal{H}_t}.
	\end{equation} 
	Similarly, we can restrict all operator norms to subspace $ L_q $. Therefore, we have
	\begin{equation}\nonumber
		\left\|(T_\delta+\lambda I)^{-\frac{1}{2}}
		\right\|_{\mathscr{B}(L_q)}=\sup_{1\leq k\leq q}\frac{1}{\sqrt{\lambda_k+\lambda}}\leq \lambda_q^{-\frac{1}{2}}\lesssim m^{\frac{1}{2r+1}\cdot\frac{1}{2}},
	\end{equation}
	and 
	\begin{equation}\nonumber
		\left\|\tilde{T}_\delta-T_\delta\right\|_{\mathscr{B}(L_q)}=\sup_{1\leq k\leq q}|\tilde{\lambda}_k-\lambda_k|\lesssim \frac{1}{K}+\epsilon^{\frac{1}{4}}.
	\end{equation}
	For the third term of $ L_2 $, a direct calculation shows
	\begin{equation}\nonumber
		\begin{aligned}
			\left\|g_\lambda(\tilde{T}_\delta)g_D^T\right\|_{\mathcal{H}_t}
			&=\left(\sum_{k=1}^q g_\lambda(\tilde{\lambda}_k)^2|\langle g_D,\phi_k\rangle_{\mathcal{H}_t}|^2\right)^\frac{1}{2}\\
			&\lesssim \lambda_q^{-1}\left(\sum_{k=1}^q |\langle g_D,\phi_k\rangle_{\mathcal{H}_t}|^2\right)^\frac{1}{2}\\
			&\lesssim m^{\frac{1}{2r+1}}.
		\end{aligned}
	\end{equation}
	Combining the above yields
	\begin{equation}\label{upper bound for L_2, proof of I_1 upper bound lemma, approximation error analysis}
		L_2\lesssim m^{\frac{1}{2r+1}\cdot\frac{3}{2}}\left(\frac{1}{K}+\epsilon^{\frac{1}{4}}\right).
	\end{equation}
	By combining \eqref{upper bound for L_1, proof of I_1 upper bound lemma, approximation error analysis} and \eqref{upper bound for L_2, proof of I_1 upper bound lemma, approximation error analysis} with the previous upper bounds \eqref{upper bound of empirical norm between T_nu and T_delta, proof I_1 upper bound lemma} and \eqref{upper bound of g_lambda(T_delta), proof I_1 upper bound lemma}, we derive that
	\begin{equation}\label{upper bound for J_1, proof of I_1 upper bound lemma, approximation error analysis}
		J_1\lesssim m^{-\frac{\xi}{2}+\frac{1}{2r+1}(\xi+\frac{1}{2})}\cdot(\log m)^\frac{d\xi}{4}+m^{\frac{1}{2r+1}\cdot\frac{3}{2}}\left(\frac{1}{K}+\epsilon^{\frac{1}{4}}\right).
	\end{equation}
	Therefore, a combination of \eqref{upper bound for J_2, proof of I_1 upper bound lemma, approximation error analysis} and \eqref{upper bound for J_1, proof of I_1 upper bound lemma, approximation error analysis} with previous upper bounds \eqref{upper bound of alpha-to-H_t norm, proof I_1 upper bound lemma} results in
	\begin{equation}\nonumber
		\begin{aligned}
			\left\|\left(g_\lambda(T_\delta)-g_\lambda(\tilde{T}_\delta)\right)g_D^T\right\|_\alpha
			&\lesssim \lambda^{-\frac{\alpha}{2}}\left(m^{-\frac{\xi}{2}+\frac{1}{2r+1}(\xi+\frac{1}{2})}\cdot(\log m)^\frac{d\xi}{4}+m^{\frac{1}{2r+1}\cdot\frac{3}{2}}\left(\frac{1}{K}+\epsilon^{\frac{1}{4}}\right)\right)\\
			&\lesssim m^{-\frac{\xi}{2}+\frac{1}{2r+1}(\xi+\frac{1}{2})+\frac{\alpha}{4}}\cdot(\log m)^{\frac{d\xi}{4}-\frac{d\alpha}{8}}\\
			&\quad+m^{\frac{1}{2r+1}\cdot\frac{3}{2}+\frac{\alpha}{4}}\cdot(\log m)^{-\frac{d\alpha}{8}}\cdot\left(\frac{1}{K}+\epsilon^{\frac{1}{4}}\right).
		\end{aligned}
	\end{equation}
	Finally, the embedding property $ \mathcal{H}_t^\alpha\hookrightarrow L^\infty(\nu) $ yields
	\begin{equation}\nonumber
			I_1\lesssim m^{-\frac{\xi}{2}+\frac{1}{2r+1}(\xi+\frac{1}{2})+\frac{\alpha}{4}}\cdot(\log m)^{\frac{d\xi}{4}-\frac{d\alpha}{8}}+m^{\frac{1}{2r+1}\cdot\frac{3}{2}+\frac{\alpha}{4}}\cdot(\log m)^{-\frac{d\alpha}{8}}\cdot\left(\frac{1}{K}+\epsilon^{\frac{1}{4}}\right),
	\end{equation}
	which valid for all $ \alpha>0 $. This completes the proof.
\end{proof}

We now focus our attention on term $ I_2 $, the upper bound result of which is summarized in the following Lemma \ref{Lemma: upper bound for I_2 in approximation error analysis}.

\begin{lemma}\label{Lemma: upper bound for I_2 in approximation error analysis}
	Suppose that Assumption \ref{Assumption: source condition on regression function} and Assumption \ref{Assumption: boundedness condition on regression function} hold. Suppose that the conditions \eqref{conditions on heat kernel estimation Thm} holds with $ m $ sufficiently large, $ N>K\gg m $ sufficiently large and $ \epsilon $ sufficiently small. We choose $ q $ as in \eqref{definition of truncated q}. Denote $ E_1 $ the event given in Theorem \ref{Thm: L-infty eigen-system approxiamtion of graph Laplacian} and $ E_2 $ the good event proposed in Theorem \ref{Thm: eigen-system consistency of heat kernel matrix}. Then, on $ E_1\bigcap E_2 $, for any $ 1\leq j\leq N $, it holds
	\begin{equation}\nonumber
		I_2\leq K_2' m^{\frac{2r+4}{2r+1}}(\log m)^\frac{d}{2}\cdot\left(\frac{1}{K}+\epsilon^{\frac{1}{4}}\right),
	\end{equation}
	where $ K_2' $ is a constant independent of $ m,K,N,\epsilon$.
\end{lemma}

\begin{proof}
	We first decompose $ I_2 $ as 
	\begin{equation}\nonumber
		\begin{aligned}
			I_2
			&=\left|\sum_{k=1}^q g_\lambda(\tilde{\lambda}_k)\left(\langle g_D,\phi_k\rangle_{\mathcal{H}_t}\phi_k(x_j)-\tilde{g}_{D,k}\tilde{\phi}_k(j)\right)\right|\\
			&\leq\sum_{k=1}^q \left|g_\lambda(\tilde{\lambda}_k)\right|\cdot\left|\langle g_D,\phi_k\rangle_{\mathcal{H}_t}-\tilde{g}_{D,k}\right|\cdot\left|\phi_k(x_j)\right|\\
			&\quad+\sum_{k=1}^q \left|g_\lambda(\tilde{\lambda}_k)\right|\cdot\left|\tilde{g}_{D,k}\right|\cdot\left|\phi_k(x_j)-\tilde{\phi}_k(j)\right|.
		\end{aligned}
	\end{equation}
	Notice that
	\begin{equation}\nonumber
		\left|g_\lambda(\tilde{\lambda}_k)\right|\leq\tilde{\lambda}_k^{-1}\lesssim m^{\frac{1}{2r+1}},
	\end{equation}
	which is already established in \eqref{lower bound of hat lambda_k, proof of eigenfunction estimation lemma}. Recall the definition \eqref{estimator of gD} of $ \tilde{g}_{D,k} $, combining with the reproducing property, we have
	\begin{equation}\nonumber
		\begin{aligned}
			\left|\langle g_D,\phi_k\rangle_{\mathcal{H}_t}-\tilde{g}_{D,k}\right|
			&=\left|\frac{1}{m}\sum_{i=1}^my_i\phi_k(x_i)-\frac{1}{m}\sum_{i=1}^m y_i\tilde{\phi}_k(i)\right|\\
			&\leq\frac{1}{m}\sum_{i=1}^m|y_i|\cdot|\phi_k(x_i)-\tilde{\phi}_k(i)|.
		\end{aligned}
	\end{equation}
	By the eigenfunction estimation result \eqref{estimation error of empirical eigenfunction estimator}, we have
	\begin{equation}\nonumber
			\left|\langle g_D,\phi_k\rangle_{\mathcal{H}_t}-\tilde{g}_{D,k}\right|\lesssim m^{\frac{r+2}{2r+1}}\left(\frac{1}{K}+\epsilon^{\frac{1}{4}}\right).
	\end{equation}
	From the definition of $ \phi_k $ as in \eqref{relation between eigen-system of T_delta and heat kernel matrix}, we have 
	\begin{equation}\label{upper bound of phi_k(x_j), proof of I_2 upper bound lemma, approximation error}
		|\phi_k(x_j)|=\left|\frac{1}{\sqrt{m\hat{\lambda}_k}}\sum_{i=1}^m\hat{u}_k(i)H_t(x_i,x_j)\right|\lesssim m^\frac{r+1}{2r+1},
	\end{equation}
	where the inequality follows from \eqref{upper bound of sqrt mlambda_k, proof of eigenfunction estimation lemma}, that $ \hat{u}_k $ is a unit vector, and the boundedness condition on heat kernel. Combining the above yields

	\begin{equation}\label{upper bound of first term, proof of I_2 upper bound lemma, approximation error analysis}
		\begin{split}
			\sum_{k=1}^q \left|g_\lambda(\tilde{\lambda}_k)\right|\cdot\left|\langle g_D,\phi_k\rangle_{\mathcal{H}_t}-\tilde{g}_{D,k}\right|\cdot\left|\phi_k(x_j)\right|
			&\lesssim q\cdot m^{\frac{3}{2r+1}+1}\left(\frac{1}{K}+\epsilon^{\frac{1}{4}}\right)\\
			&\lesssim m^{\frac{3}{2r+1}+1}(\log m)^\frac{d}{2}\cdot\left(\frac{1}{K}+\epsilon^{\frac{1}{4}}\right).
		\end{split}
	\end{equation}

	For the second term, we have
	\begin{equation}\nonumber
		\left|\tilde{g}_{D,k}\right|\leq\left|\langle g_D,\phi_k\rangle_{\mathcal{H}_t}\right|+\left|\langle g_D,\phi_k\rangle_{\mathcal{H}_t}-\tilde{g}_{D,k}\right|.
	\end{equation}
	Notice that
	\begin{equation}\label{upper bound of <g_D,phi_k>, proof of I_2 upper bound lemma, approximation error analysis}
		\left|\langle g_D,\phi_k\rangle_{\mathcal{H}_t}\right|\leq\|g_D\|_{\mathcal{H}_t}\cdot\|\phi_k\|_{\mathcal{H}_t}\leq M\kappa^2,
	\end{equation}
	since $ \phi_k $ an orthonormal base of $ \mathcal{H}_t $ and $ \|H_t(x_i,\cdot)\|_{\mathcal{H}_t}=H_t(x_i,x_i)\leq\kappa^2 $. Therefore,
	\begin{equation}\nonumber
		\left|\tilde{g}_{D,k}\right|\lesssim M\kappa^2+m^{\frac{r+2}{2r+1}}\left(\frac{1}{K}+\epsilon^{\frac{1}{4}}\right)\lesssim M\kappa^2.
	\end{equation}
	Combining this upper bound with \eqref{estimation error of empirical eigenfunction estimator} and \eqref{lower bound of hat lambda_k, proof of eigenfunction estimation lemma}, we derive that

	\begin{equation}\label{upper bound of second term, proof of I_2 upper bound lemma, approximation error analysis}
		\begin{split}
			\sum_{k=1}^q \left|g_\lambda(\tilde{\lambda}_k)\right|\cdot\left|\tilde{g}_{D,k}\right|\cdot\left|\phi_k(x_j)-\tilde{\phi}_k(j)\right|
			&\lesssim q\cdot m^{\frac{r+3}{2r+1}}\left(\frac{1}{K}+\epsilon^{\frac{1}{4}}\right)\\
			&\lesssim m^{\frac{r+3}{2r+1}}(\log m)^\frac{d}{2}\cdot\left(\frac{1}{K}+\epsilon^{\frac{1}{4}}\right).
		\end{split}
	\end{equation}

	Finally, combining \eqref{upper bound of first term, proof of I_2 upper bound lemma, approximation error analysis} and \eqref{upper bound of second term, proof of I_2 upper bound lemma, approximation error analysis}, we derive that
	\begin{equation}\nonumber
		I_2\lesssim m^{\frac{3}{2r+1}+1}(\log m)^\frac{d}{2}\cdot\left(\frac{1}{K}+\epsilon^{\frac{1}{4}}\right),
	\end{equation}
	which completes the proof.
\end{proof}

Combining Lemma \ref{Lemma: upper bound for I_1 in approximation error analysis} and Lemma \ref{Lemma: upper bound for I_2 in approximation error analysis}, we finally derive the following Theorem \ref{Thm: upper bound for approximation error}.

\begin{theorem}\label{Thm: upper bound for approximation error}
	Suppose that Assumption \ref{Assumption: source condition on regression function}, Assumption \ref{Assumption: boundedness condition on regression function}, and Assumption \ref{Assumption: qualification condition on filter function} hold.
	Suppose that the conditions \eqref{conditions on heat kernel estimation Thm} holds with $ m $ sufficiently large, $ N>K\gg m $ sufficiently large and $ \epsilon $ sufficiently small.  We choose $ q,\lambda $ as in \eqref{definition of truncated q} and \eqref{definition of lambda} separately. Denote $ E_1 $ the event given in Theorem \ref{Thm: L-infty eigen-system approxiamtion of graph Laplacian}, $ E_2 $ the event given in Theorem \ref{Thm: eigen-system consistency of heat kernel matrix}, and $ E_3 $ the event given in Theorem \ref{Thm: upper bound for truncation error}. Then, for any $ 0<\alpha\leq1 $ and $ 1\leq j\leq N $, on $ E_1\bigcap E_2\bigcap E_3 $, it holds
	\begin{equation}\nonumber
		\begin{aligned}
			|f_{D,\lambda}^T(x_j)-\tilde{f}_{D,\lambda}(j)|
			&\leq K_a\left(m^{-\frac{\xi}{2}+\frac{\xi+\frac{1}{2}}{2r+1}+\frac{\alpha}{4}}\cdot(\log m)^{\frac{d\xi}{4}-\frac{d\alpha}{8}}+m^{\frac{3}{2r+1}+1}(\log m)^\frac{d}{2}\cdot\left(\frac{1}{K}+\epsilon^{\frac{1}{4}}\right)\right)
		\end{aligned}
	\end{equation}
	with $ K_a $  a constant independent of $ m,K,N,\epsilon,\alpha$.
\end{theorem}

\subsection{Proof of Theorem \ref{Thm: convergence analysis}}\label{subsection: proof of convergence analysis thm}

\emph{Proof of Theorem \ref{Thm: convergence analysis}}.
	Since for any $ m+1\leq j\leq m+n $, 
	\begin{equation}\nonumber
		\begin{aligned}
			|\tilde{f}_{D,\lambda}(j)-f^*(x_j)|
			&\leq |\tilde{f}_{D,\lambda}(j)-f_{D,\lambda}^T(x_j)|+|f_{D,\lambda}^T(x_j)-f^*(x_j)|\\
			&\leq|\tilde{f}_{D,\lambda}(j)-f_{D,\lambda}^T(x_j)|+\|f_{D,\lambda}^T-f^*\|_\infty,
		\end{aligned}
	\end{equation}
	a direct combination of Theorem \ref{Thm: upper bound for truncation error} and Theorem \ref{Thm: upper bound for approximation error} yields
	\begin{equation}\nonumber
		\begin{aligned}
			|\tilde{f}_{D,\lambda}(j)-f^*(x_j)|
			&\leq K_tm^{-\frac{C_{low}}{C_{up}}\cdot\frac{1}{2}+\frac{\alpha}{4}}(\log m)^{\frac{1}{2}+\frac{d}{4}-\frac{\alpha d}{8}}\\
			&\quad +K_a\left(m^{-\frac{\xi}{2}+\frac{\xi+\frac{1}{2}}{2r+1}+\frac{\alpha}{4}}\cdot(\log m)^{\frac{d\xi}{4}-\frac{d\alpha}{8}}+m^{\frac{3}{2r+1}+1}(\log m)^\frac{d}{2}\cdot\left(\frac{1}{K}+\epsilon^{\frac{1}{4}}\right)\right)\\
			&\leq C m^{-\theta_1}\cdot(\log m)^{\theta_2}+C' m^{\frac{3}{2r+1}+1}(\log m)^\frac{d}{2}\cdot\left(\frac{1}{K}+\epsilon^{\frac{1}{4}}\right).
		\end{aligned}
	\end{equation}
	Here, $ \theta_1$ and $\theta_2 $ are as defined in \eqref{definition of theta_1 and theta_2}, and $ C,C' $ are constants independent of $ m,K,N,\epsilon,\alpha$. This completes the proof.
\qed

\bibliographystyle{plain}
\bibliography{Reference}

\appendix

\section{Proof of Theorem \ref{Thm: L-infty eigen-system approxiamtion of graph Laplacian}}\label{appendix: proof of L-infty eigen-system approxiamtion of graph Laplacian}

\begin{proof}
	Based on the eigen-convergence results presented in \cite{cheng2022eigen}, we first re-normalize the eigenvector $ \tilde{v}_k $ as $ \tilde{u}_k=\frac{1}{\sqrt{pN}}\tilde{v}_k $ to establish an orthonormal basis of $ \mathbb{R}^N $. Therefore, with probability at least $ 1-c_1N^{-8} $, for each $ 1\leq k\leq K $, the following approximation holds:
	\begin{equation}\label{cited eigenvalue approximation error}
		\left|\tilde{\mu}_k-\mu_k\right|\lesssim\left(\frac{\log N}{N}\right)^\frac{1}{\frac{d}{2}+2},
	\end{equation}
	\begin{equation}\label{cited L^2 eigenvector approximation error}
		\left\|\tilde{u}_k-\alpha_k\rho_X\left(\frac{1}{\sqrt{pN}}\psi_k\right)\right\|_{l^2}\lesssim\left(\frac{\log N}{N}\right)^\frac{1}{d+4}.
	\end{equation}
	Here, $ |\alpha_k|=1 $ represents the direction, and $ \|\cdot\|_{l^2} $ denotes the standard Euclidean $2$-norm on $ \mathbb{R}^N $.
	We may assume without loss of generality that $ \alpha_k=1 $ for all $ 1\leq k\leq K $. Our targeted eigenvalue upper bound \eqref{graph laplacian eigenvalue approximation error} directly follows from \eqref{cited eigenvalue approximation error}. The focus then shifts to translating the $ l^2 $-norm bound \eqref{cited L^2 eigenvector approximation error} to our targeted $ l^\infty $-norm bound \eqref{graph laplacian L^infty eigenvector approximation error}. This translation involves defining several normalized norms on $ \mathbb{R}^N $ as
	\begin{equation}\nonumber
		\begin{aligned}
			&\|u\|_2^2\triangleq\frac{1}{N}\sum_{i=1}^N u_i^2,\\
			&\|u\|_1\triangleq\frac{1}{N}\sum_{i=1}^N |u_i|,\\
		\end{aligned}
	\end{equation} 
	leading to the relationship
	\begin{align}
		\|\tilde{v}_k-\rho_X(\psi_k)\|_{2}
		&=\frac{1}{\sqrt{N}}\|\tilde{v}_k-\rho_X(\psi_k)\|_{l^2}\nonumber\\
		&=\sqrt{p}\left\|\tilde{u}_k-\rho_X\left(\frac{1}{\sqrt{pN}}\psi_k\right)\right\|_{l^2}\lesssim\left(\frac{\log N}{N}\right)^\frac{1}{d+4}.\label{2-norm eigenvector approximation error}
	\end{align}
	Given the Gaussian kernel basis for our un-normalized graph Laplacian $ L_{un} $, we further achieve a point-wise convergence rate on sample points with a probability of at least $ 1-2N^{-9} $ as
	\begin{equation}\label{point-wise rate for Gaussian Lun}
		\|L_{un}{\rho_X(f)}-\rho_X(\Delta f)\|_{\infty}\lesssim\left(\frac{\log N}{N}\right)^\frac{1}{d+4}.
	\end{equation}
	Such a result could also be found in \cite{calder2022improved}.
	Now we introduce another un-normalized graph Laplacian $ L_{un}^I $, constructed with the indicator kernel $ K_I $, defined as
	\begin{equation}\nonumber
		K_I(x,z)=\epsilon^{-\frac{d}{2}}\mathbbm{1}_{\{\|x-z\|\leq\sqrt{\epsilon}\}}.
	\end{equation}
	Let $ \varepsilon_k = \tilde{v}_k-\rho_X(\psi_k)$, and we assume $ \varepsilon_k\neq0$ (Otherwise, there is nothing to be proved). We further define
	\begin{equation}\label{definition of lambda_k}
		\lambda_k = \frac{\|L_{un}^I\varepsilon_k\|_\infty}{\|\varepsilon_k\|_\infty}.
	\end{equation}
	Therefore, we have
	\begin{align}
		L_{un}^I\varepsilon_k
		&=L_{un}^I\left(\tilde{v}_k-\rho_X(\psi_k)\right)\nonumber\\
		&=(L_{un}\tilde{v}_k-\rho_X(\Delta\psi_k))+(L_{un}^I\tilde{v}_k-L_{un}\tilde{v}_k)+(\rho_X(\Delta\psi_k)-L_{un}^I(\rho_X(\psi_k)))\label{decomposition of L_Iun epsilon_k}.
	\end{align}
	For the first term in \eqref{decomposition of L_Iun epsilon_k} we have
	\begin{equation}\nonumber
		\begin{aligned}
			L_{un}\tilde{v}_k-\rho_X(\Delta\psi_k)
			& = \tilde{\mu}_k\tilde{v}_k-\rho_X(\mu_k\psi_k)\\
			& = \tilde{\mu}_k(\tilde{v}_k-\rho_X(\psi_k))+(\tilde{\mu}_k-\mu_k)\rho_X(\psi_k).\\
		\end{aligned}
	\end{equation}
	Hence
	\begin{equation}\nonumber
		\begin{aligned}
			\|L_{un}\tilde{v}_k-\rho_X(\Delta\psi_k)\|_\infty
			&\leq\tilde{\mu}_k\|\varepsilon_k\|_\infty+|\tilde{\mu}_k-\mu_k|\cdot\|\rho_X(\psi_k)\|_\infty\\
			&\leq\tilde{\mu}_k\|\varepsilon_k\|_\infty+|\tilde{\mu}_k-\mu_k|\cdot\|\psi_k\|_{L^\infty(\mathcal{M})}.
		\end{aligned}
	\end{equation}
	Since $ \mu_k\leq\mu_K $ and $ K $ is a fixed constant, combining \eqref{population bounds for eigenfunction of laplacian} and \eqref{cited eigenvalue approximation error} we have
	\begin{equation}\label{upper bound of the first term, proof of graph Laplacian Thm}
			\|L_{un}\tilde{v}_k-\rho_X(\Delta\psi_k)\|_\infty
			\leq\tilde{\mu}_k\|\varepsilon_k\|_\infty+C\left(\frac{\log N}{N}\right)^\frac{1}{\frac{d}{2}+2},
	\end{equation}
	where constant $ C $ is independent of $ N, \epsilon $. For the second term in \eqref{decomposition of L_Iun epsilon_k} we have
	\begin{equation}\nonumber
		\|L_{un}^I\tilde{v}_k-L_{un}\tilde{v}_k\|_\infty\leq\|L_{un}^I\tilde{v}_k-\rho_X(\Delta f_k)\|_\infty+\|L_{un}\tilde{v}_k-\rho_X(\Delta f_k)\|_\infty
	\end{equation}
	with some function $ f_k\in C^\infty(\mathcal{M}) $ satisfying $ \tilde{v}_k=\rho_X(f_k) $. For the un-normalized graph Laplacian $ L_{un}^I $, a point-wise convergence result parallel to that of $ L_{un} $ as in \eqref{point-wise rate for Gaussian Lun} is also available, as shown in \cite{calder2022improved}:
	\begin{equation}\label{point-wise rate for indicator Lun}
		\|L_{un}^I{\rho_X(f)}-\rho_X(\Delta f)\|_{\infty}\lesssim \left(\frac{\log N}{N}\right)^\frac{1}{d+4}.
	\end{equation}
	As a consequence, a combination of \eqref{point-wise rate for Gaussian Lun} and \eqref{point-wise rate for indicator Lun} yields
	\begin{equation}\label{upper bound of the second term, proof of graph Laplacian Thm}
		\|L_{un}^I\tilde{v}_k-L_{un}\tilde{v}_k\|_\infty\lesssim\left(\frac{\log N}{N}\right)^\frac{1}{d+4}.
	\end{equation}
	For the remained third term in \eqref{decomposition of L_Iun epsilon_k}, it follows directly from \eqref{point-wise rate for indicator Lun} that
	\begin{equation}\label{upper bound of the third term, proof of graph Laplacian Thm}
		\|\rho_X(\Delta\psi_k)-L_{un}^I(\rho_X(\psi_k))\|_\infty\lesssim\left(\frac{\log N}{N}\right)^\frac{1}{d+4}.
	\end{equation}
	Finally, by combining \eqref{upper bound of the first term, proof of graph Laplacian Thm}, \eqref{upper bound of the second term, proof of graph Laplacian Thm}, and \eqref{upper bound of the third term, proof of graph Laplacian Thm} we have
	\begin{equation}\nonumber
		\|L_{un}^I\varepsilon_k\|_\infty\leq\tilde{\mu}_k\|\varepsilon_k\|_\infty+C'\left(\frac{\log N}{N}\right)^\frac{1}{d+4},
	\end{equation}
	where constant $ C' $ is independent of $ N, \epsilon $.
	Therefore, from definition \eqref{definition of lambda_k} we have
	\begin{equation}\nonumber
		\lambda_k\leq\tilde{\mu}_k+C'
		\frac{\left(\frac{\log N}{N}\right)^\frac{1}{d+4}}{\|\varepsilon_k\|_\infty}.
	\end{equation}
	Now we consider two possible cases.

	\noindent
	Case 1: $ \lambda_k\geq\tilde{\mu}_k+1 $, which implies
	\begin{equation}\nonumber
		\|\varepsilon_k\|_\infty\leq C'\left(\frac{\log N}{N}\right)^\frac{1}{d+4}.
	\end{equation}
	Case 2: $ \lambda_k<\tilde{\mu}_k+1$. By exploiting techniques in \cite{calder2022lipschitz} we can derive that, with probability at least $1-c_2N^{-8}$, it holds
	\begin{equation}\nonumber
		\|\varepsilon_k\|_\infty\lesssim(\lambda_k+1)^{d+1}\|\varepsilon_k\|_1.
	\end{equation}
	A direct calculation using Cauchy-Schwarz inequality combined with \eqref{2-norm eigenvector approximation error} leads to
	\begin{equation}\nonumber
			\|\varepsilon_k\|_\infty
			\lesssim(\tilde{\mu}_k+2)^{d+1}\|\varepsilon_k\|_2
			\lesssim\left(\frac{\log N}{N}\right)^\frac{1}{d+4},
	\end{equation}
	where the last inequality also follows from the fact that $ k\leq K $ and $ K $ is a fixed constant. Finally, combining above two cases yields the desired $l^ \infty $ upper bound
	\begin{equation}\nonumber
		\|\varepsilon_k\|_\infty\lesssim\left(\frac{\log N}{N}\right)^\frac{1}{d+4}.
	\end{equation}
	The proof is then finished.
\end{proof}

\section{Proof of Theorem \ref{Thm: heat kernel estimation}}\label{appendix: proof of heat kernel estimation}

\begin{proof}
	For any $ i,j\in\{1,\cdots,N\} $, we divide the point-wise error into two parts
	\begin{equation}\nonumber
		\begin{aligned}
			\left|H_t(x_i,x_j)-\tilde{H}_{t,K}(i,j)\right|
			&\leq\left|\sum_{k=1}^K\left(e^{-\tilde{\mu}_kt}\tilde{v}_k(i)\tilde{v}_k(j)-e^{-\mu_kt}\psi_k(x_i)\psi_k(x_j)\right)\right|\\
			&\quad+\left|\sum_{k=K+1}^\infty e^{-\mu_kt}\psi_k(x_i)\psi_k(x_j)\right|\\
			&\triangleq I_1+I_2.
		\end{aligned}
	\end{equation}
	Initially, we concentrate on the second term $ I_2 $. An application of the Cauchy-Schwarz inequality allows us to derive
	\begin{equation}\nonumber
		I_2 = \left|\sum_{k=K+1}^\infty e^{-\mu_kt}\psi_k(x_i)\psi_k(x_j)\right|
		\leq \sup_{x\in\mathcal{M}}\sum_{k=K+1}^\infty e^{-\mu_kt}|\psi_k(x)|^2.
	\end{equation}
	Utilizing techniques akin to those detailed by \cite{berard1994embedding} (on page 393), the right-hand side of the above inequality can be bounded by the following integral
	\begin{equation}\nonumber
		\sup_{x\in\mathcal{M}}\sum_{k=K+1}^\infty e^{-\mu_kt}|\psi_k(x)|^2\leq D_2 t^{-\frac{d}{2}}\int_{t\mu_{K+1}}^\infty s^{\frac{d}{2}}e^{-s}ds,
	\end{equation}
	where $ D_2 $ is an absolute constant that depends solely on $ \mathcal{M} $. Combined with \eqref{population bounds for eigenvalue of laplacian}, we have
	\begin{equation}\label{first bound for I_2, proof of heat kernel estimation Thm}
		I_2\leq D_2 t^{-\frac{d}{2}}\int_{C_{low}(K+1)^\frac{2}{d}t}^\infty s^{\frac{d}{2}}e^{-s}ds.
	\end{equation}
	It is noteworthy that
	\begin{equation}\nonumber
		\Gamma(a,x)=\int_x^\infty s^{a-1}e^{-s}ds
	\end{equation}
	represents the incomplete Gamma function. According to \cite{borwein2009uniform}, we obtain the following estimation:
	\begin{equation}\nonumber
		\left|\Gamma(a,x)\right|\leq x^{a-1}e^{-x}\cdot \frac{c}{c-1},
	\end{equation}
	which holds for any $ c>1,a\geq1 $ and $ |x|\geq c(a-1) $.
	Returning to our estimation \eqref{first bound for I_2, proof of heat kernel estimation Thm}, if we set $ c=2 $, then whenever $ C_{low}(K+1)^\frac{2}{d}t\geq d $, or equivalently,
	\begin{equation}\nonumber
		K\geq\left(\frac{d}{C_{low}t}\right)^\frac{d}{2}-1,
	\end{equation} 
	we can deduce an upper bound for $ I_2 $ as:
	\begin{equation}\label{upper bound for I_2, proof of heat kernel estimation Thm}
		\begin{aligned}
			I_2	&\leq D_2 t^{-\frac{d}{2}}\cdot 2\left(C_{low}(K+1)^\frac{2}{d}t\right)^\frac{d}{2}e^{-C_{low}(K+1)^\frac{2}{d}t}\\
			&\leq2D_2C_{low}^\frac{d}{2}(K+1)e^{-C_{low}(K+1)^\frac{2}{d}t}.
		\end{aligned}
	\end{equation}
	Now we return to $ I_1 $. We further decompose it as 
	\begin{equation}\nonumber
		\begin{aligned}
			\left|\sum_{k=1}^K\left(e^{-\tilde{\mu}_kt}\tilde{v}_k(i)\tilde{v}_k(j)-e^{-\mu_kt}\psi_k(x_i)\psi_k(x_j)\right)\right|
			&\leq\left|\sum_{k=1}^K\left(e^{-\tilde{\mu}_kt}-e^{-\mu_kt}\right)\psi_k(x_i)\psi_k(x_j)\right|\\
			&\quad+\left|\sum_{k=1}^K\left(\tilde{v}_k(i)\tilde{v}_k(j)-\psi_k(x_i)\psi_k(x_j)\right)e^{-\tilde{\mu}_kt}\right|.
		\end{aligned}
	\end{equation}
	For the first term in the aforementioned inequality, by applying \eqref{graph laplacian eigenvalue approximation error} we obtain
	\begin{equation}\nonumber
		\begin{aligned}
			\left|e^{-\tilde{\mu}_kt}-e^{-\mu_kt}\right|
			&\leq e^{-\mu_kt}\left|e^{-(\tilde{\mu}_k-\mu_k)t}-1\right|\\
			&\leq e^{-\mu_kt}|\tilde{\mu}_k-\mu_k|t\\
			&\leq C_1\epsilon te^{-\mu_1t}\leq C_1\epsilon t.
		\end{aligned}
	\end{equation}
	Here, $ \epsilon $ serves as shorthand notation for $ \left(\frac{\log N}{N}\right)^\frac{1}{\frac{d}{2}+2} $, consistent with the parameterization used in Theorem \ref{Thm: L-infty eigen-system approxiamtion of graph Laplacian}.
	Combining with \eqref{population bounds for eigenfunction of laplacian} we have
	\begin{equation}\label{first term bound for I_1, proof of heat kernel estimation Thm}
		\begin{aligned}
			\left|\sum_{k=1}^K\left(e^{-\tilde{\mu}_kt}-e^{-\mu_kt}\right)\psi_k(x_i)\psi_k(x_j)\right|
			&\leq\sum_{k=1}^KC_1\epsilon tD_1^2\mu_k^\frac{d-1}{2}\\
			&\leq D_1^2C_1t\epsilon K\mu_K^\frac{d-1}{2}.
		\end{aligned}
	\end{equation}
	For the second term, for any $ 1\leq k\leq K $, the combination of \eqref{graph laplacian L^infty eigenvector approximation error} with \eqref{population bounds for eigenfunction of laplacian} yields:
	\begin{equation}\nonumber
		\begin{aligned}
			\left|\tilde{v}_k(i)\tilde{v}_k(j)-\psi_k(x_i)\psi_k(x_j)\right|
			&\leq\left|\tilde{v}_k(i)\left(\tilde{v}_k(j)-\psi_k(x_j)\right)+\psi_k(x_j)\left(\tilde{v}_k(i)-\psi_k(x_i)\right)\right|\\
			&\leq\max_{1\leq i\leq N}|\tilde{v}_k(i)-\psi_k(x_i)|\cdot\left(\max_{1\leq i\leq N}|\tilde{v}_k(i)|+\max_{1\leq i\leq N}|\psi_k(x_i)|\right)\\
			&\leq C_2\sqrt{\epsilon}\cdot\left(\max_{1\leq i\leq N}|\tilde{v}_k(i)-\psi_k(x_i)|+2\max_{1\leq i\leq N}|\psi_k(x_i)|\right)\\
			&\leq C_2\sqrt{\epsilon}\cdot\left(C_2\sqrt{\epsilon}+2D_1\mu_k^\frac{d-1}{4}\right)\\
			&\leq 3C_2D_1\mu_k^\frac{d-1}{4}\sqrt{\epsilon}.
		\end{aligned}
	\end{equation}
	The last inequality in this formulation is derived under the condition that $ \epsilon\ll1 $.	
	Given that $ |\tilde{\mu}_k-\mu_k|\leq C_1\epsilon\leq\frac{1}{2}\mu_k $ for $ 2\leq k\leq K $ and $ \tilde{\mu}_1=\mu_1=0 $, we can deduce
	\begin{equation}\nonumber
		|e^{-\tilde{\mu}_kt}|\leq|e^{-\frac{1}{2}\mu_kt}|\leq|e^{-\frac{1}{2}\mu_1t}|\leq1,
	\end{equation} 
	which is valid for all $ 1\leq k\leq K $. Consequently, we obtain
	\begin{equation}\label{second term bound for I_1, proof of heat kernel estimation Thm}
		\left|\sum_{k=1}^K\left(\tilde{v}_k(i)\tilde{v}_k(j)-\psi_k(x_i)\psi_k(x_j)\right)e^{-\tilde{\mu}_kt}\right|\leq 3C_2D_1K\mu_k^\frac{d-1}{4}\sqrt{\epsilon}.
	\end{equation} 
	By combining \eqref{first term bound for I_1, proof of heat kernel estimation Thm} and \eqref{second term bound for I_1, proof of heat kernel estimation Thm} we have
	\begin{equation}\label{upper bound for I_1, proof of heat kernel estimation Thm}
		\begin{aligned}
			I_1
			&\leq D_1^2C_1t\epsilon K\mu_K^\frac{d-1}{2}+3C_2D_1K\mu_k^\frac{d-1}{4}\sqrt{\epsilon}\\
			&\leq 4D_1^2C_1C_2t\sqrt{\epsilon} K\mu_K^\frac{d-1}{2}.
		\end{aligned}
	\end{equation}
	Finally, by integrating the upper bounds \eqref{upper bound for I_1, proof of heat kernel estimation Thm} and \eqref{upper bound for I_2, proof of heat kernel estimation Thm} we ascertain that	
	\begin{equation}\nonumber
		\left|H_t(x_i,x_j)-\tilde{H}_{t,K}(i,j)\right|\leq4D_1^2C_1C_2t\sqrt{\epsilon} K\mu_K^\frac{d-1}{2}+2D_2C_{low}^\frac{d}{2}(K+1)e^{-C_{low}(K+1)^\frac{2}{d}t}.
	\end{equation}
	Therefore, given our selection of $ K,\epsilon $, which adhere to the conditions
	\begin{equation}\nonumber
		(K+1)e^{-C_{low}(K+1)^\frac{2}{d}t}\leq\frac{1}{K},
	\end{equation}
	\begin{equation}\nonumber
		\epsilon^{\frac{1}{4}}\leq \frac{1}{C_1C_2K\mu_K^\frac{d-1}{2}},
	\end{equation}
	we ultimately attain the desired upper bound
	\begin{equation}\nonumber
		\left|H_t(x_i,x_j)-\tilde{H}_{t,K}(i,j)\right|\leq4D_1^2t\epsilon^{\frac{1}{4}}+2D_2C_{low}^\frac{d}{2}\frac{1}{K}.
	\end{equation}
	This completes the proof.
\end{proof}

\section{Proof of Theorem \ref{Thm: eigen-system consistency of heat kernel matrix}}\label{appendix: proof of eigen-system consistency of heat kernel matrix}

\begin{proof}
	Drawing on the classical Weyl's perturbation inequality, as discussed in \cite{weyl1912asymptotische,franklin2012matrix}, we obtain for any $ 1\leq k\leq m $,
	\begin{equation}\label{Weyl's eigenvalue bound}
		|\tilde{\lambda}_k-\hat{\lambda}_k|\leq\|\tilde{H}_{t,m}-H_t\|_2,
	\end{equation} 
	where $ \|\cdot\|_2 $ denotes the operator $ 2 $-norm. Let $ \bm{\varepsilon} =\tilde{H}_{t,m}-H_t $ represent an $ m\times m $ matrix with entries:
	\begin{equation}\nonumber
		|\bm{\varepsilon}_{ij}|\leq \frac{1}{m}\left(\frac{C_a}{K}+C_b\epsilon^{\frac{1}{4}}\right)
	\end{equation}
	derived from \eqref{point-wise error of heat kernel matrxi estimator}. 
	Given that $ \|\bm{\varepsilon}\|_2=\max\{x^T\bm{\varepsilon}y: \|x\|_2=\|y\|_2=1\} $, it follows that
	\begin{equation}\label{2-norm bound for epsilon matrix, proof of heat kernel matrix consistency Thm}
		\|\bm{\varepsilon}\|_2= x^T\bm{\varepsilon}y\leq \sup\limits_{i,j}|\bm{\varepsilon}_{ij}|\cdot\sum_{i,j=1}^m|x_iy_j|\leq \sup\limits_{i,j}|\bm{\varepsilon}_{ij}|\cdot m\leq\frac{C_a}{K}+C_b\epsilon^{\frac{1}{4}}.
	\end{equation}
	This insight, when integrated with \eqref{Weyl's eigenvalue bound}, facilitates the derivation of the desired eigenvalue estimation \eqref{eigenvalue consistency of heat kernel matrix}. 
	Progressing to eigenvector estimations, Davis-Kahan's theorem, which is often referred to as the $ \sin\theta $ theorem, as discussed in \cite{davis1970rotation,parlett1998symmetric}, allows for a comparison between the eigen-pairs $ (\tilde{\lambda}_k,\tilde{u}_k) $ of the matrix $ \tilde{H}_{t,m} $ and a corresponding unit eigenvector $ \hat{u}_k $ of the matrix $ H_t $, associated with the eigenvalue $ \hat{\lambda}_k $, leading to the relationship
	\begin{equation}\nonumber
		|\sin\left\langle\tilde{u}_k,\hat{u}_k\right\rangle|\leq\frac{\|\bm{\varepsilon}\|_2}{\delta_k}.
	\end{equation}
	Here, $\left\langle\tilde{u}_k,\hat{u}_k\right\rangle$ signifies the angle between $ \tilde{u}_k $ and $ \hat{u}_k $, and $ \delta_k $ represents the eigen-gap of $ H_t $ with respect to the eigenvalue $ \hat{\lambda}_k $, defined as $ \delta_k =\min\{|\hat{\lambda}_j-\hat{\lambda}_k|: \hat{\lambda}_j\neq\hat{\lambda}_k\} $.
	Given that:
	\begin{equation}\nonumber
		\|\tilde{u}_k-\hat{u}_k\|_\infty\leq\|\tilde{u}_k-\hat{u}_k\|_2\leq|\left\langle\tilde{u}_k,\hat{u}_k\right\rangle|\leq\frac{\pi}{2}|\sin\left\langle\tilde{u}_k,\hat{u}_k\right\rangle|,
	\end{equation}
	where the second inequality is derived from the geometric principle that an arc's length on a unit circle is always greater than the length of its chord. Consequently, this leads to the relationship
	\begin{equation}\label{first bound for L-infty eigenvector consistency, proof of heat kernel matrix consistency Thm}
		\|\tilde{u}_k-\hat{u}_k\|_\infty\leq\frac{\pi}{2}\frac{\|\bm{\varepsilon}\|_2}{\delta_k}.
	\end{equation}
	To establish a lower bound for the eigen-gap $ \delta_k $, it's pivotal to consider that the matrix $ H_t $ represents the heat kernel matrix with entries $ \frac{1}{m}H_t(x_i,x_j) $. The integral operator of the heat kernel on $ L^2(\nu) $ is $ L_\nu $, which can be expressed through spectral decomposition \eqref{spectral decomposition of L_nu}. According to \cite{koltchinskii2000random,braun2006accurate}, for any $ 1\leq k\leq m $, and for a fixed integer $ 1\leq l\leq m $, with a probability of at least $ 1-2e^{-\tau} $, it holds
	\begin{equation}\label{heat kernel matrix asytoptic bound, proof of heat kernel matrix consistency Thm}
		|\hat{\lambda}_k-pe^{-\mu_kt}|=O\left( pe^{-\mu_kt}(pe^{-\mu_lt})^{-\frac{1}{2}}l^2m^{-\frac{1}{2}}+\Lambda_{>l}+\sqrt{\Lambda_{>l}}m^{-\frac{1}{2}}\right),
	\end{equation}
	where $ \Lambda_{>l}=\sum_{j=l+1}^\infty pe^{-\mu_jt} $. Setting $ l=\left(\frac{\log m}{3C_{low}t}\right)^\frac{d}{2} $, we can derive 
	\begin{equation}\label{caculation of Lambda_>l, proof of heat kernel matrix consistency Thm}
		\begin{aligned}
			\Lambda_{>l}=\sum_{j=l+1}^\infty pe^{-\mu_jt}
			&\leq\sum_{j=l+1}^\infty pe^{-C_{low}tj^\frac{2}{d}}\\
			&\lesssim\int_{l+1}^\infty e^{-C_{low}tx^\frac{2}{d}}dx\\
			&\lesssim\int_{C_{low}tl^\frac{2}{d}+1}^\infty u^{\frac{d}{2}-1}e^{-u}du\\
			&\lesssim (C_{low}tl^\frac{2}{d}+1)^{\frac{d}{2}-1}e^{-(C_{low}tl^\frac{2}{d}+1)}\\
			&\lesssim le^{-C_{low}tl^\frac{2}{d}}\\
			&\lesssim (\log m)^\frac{d}{2}m^{-\frac{1}{3}}.
		\end{aligned}
	\end{equation}
	The fourth inequality above again emerges from the estimation involving the incomplete Gamma function as in the Appendix \ref{appendix: proof of heat kernel estimation}. Consequently, this leads to the establishment of
	\begin{equation}\nonumber
		\sqrt{\Lambda_{>l}}m^{-\frac{1}{2}}\lesssim\sqrt{(\log m)^{\frac{d}{2}}m^{-\frac{1}{3}}}m^{-\frac{1}{2}}\lesssim(\log m)^{\frac{d}{4}}m^{-\frac{2}{3}}.
	\end{equation}
	Notice that $ \mu_k\geq\mu_1=0 $, we have
	\begin{equation}\nonumber
		\begin{aligned}
			pe^{-\mu_kt}(pe^{-\mu_lt})^{-\frac{1}{2}}l^2m^{-\frac{1}{2}}
			&\lesssim \left(e^{-C_{low}tl^\frac{2}{d}}\right)^{-\frac{1}{2}}l^2m^{-\frac{1}{2}}\\
			&\lesssim m^\frac{1}{6}(\log m)^dm^{-\frac{1}{2}}\\
			&\lesssim (\log m)^dm^{-\frac{1}{3}}.
		\end{aligned}
	\end{equation}
	By integrating these estimations into \eqref{heat kernel matrix asytoptic bound, proof of heat kernel matrix consistency Thm}, we ascertain that, for any $ 1\leq k\leq m $,
	\begin{equation}\label{heat kernel matrix eigenvalue estimation, proof of heat kernel matrix consistency Thm}
		|\hat{\lambda}_k-pe^{-\mu_kt}|=O\left((\log m)^dm^{-\frac{1}{3}}\right).
	\end{equation} 
	For any $ 1\leq k\leq q $, there exists some $ j $ (dependent on $ k $) such that
	\begin{equation}\nonumber
		\begin{aligned}
			\delta_k
			&=|\hat{\lambda}_j-\hat{\lambda}_k|\\
			&\geq|pe^{-\mu_jt}-pe^{-\mu_kt}|-O\left((\log m)^dm^{-\frac{1}{3}}\right)\\
			&=pe^{-\mu_kt}|1-e^{(\mu_k-\mu_j)t}|-O\left((\log m)^dm^{-\frac{1}{3}}\right)\\
			&\geq pe^{-\mu_kt}t|\mu_k-\mu_j|-O\left((\log m)^dm^{-\frac{1}{3}}\right)\\
			&\geq pe^{-\mu_qt}t\gamma_q-O\left((\log m)^dm^{-\frac{1}{3}}\right),\\
		\end{aligned}
	\end{equation}
	where $ \gamma_q=\min\{|\mu_j-\mu_k|: \mu_j\neq\mu_k,1\leq j,k\leq q\} $ represents the smallest eigen-gap among the first $ q $ eigenvalues of the Laplacian $ \Delta $. Given the asymptotic behavior $ \mu_k\sim k^{\frac{2}{d}} $ as $ k\to\infty $, it follows that the eigen-gap between $ \mu_k $ and $ \mu_{k+1} $  also increases to infinite. Consequently, as $ q\sim(\log m)^\frac{d}{2}\to\infty $, $\gamma_q$ becomes a constant not dependent on $ m $. Furthermore, given that $ r>1 $,  \eqref{definition of truncated q} yields
	\begin{equation}\label{lower bound of delta_k, proof of heat kernel matrix consistency Thm}
		\begin{aligned}
			\delta_k
			&\geq p e^{-C_{up}tq^\frac{2}{d}}t\gamma_q-O\left((\log m)^dm^{-\frac{1}{3}}\right)\\
			&\geq pt\gamma_qm^{-\frac{1}{2r+1}}-O\left((\log m)^dm^{-\frac{1}{3}}\right)\\
			&\gtrsim m^{-\frac{1}{2r+1}}.
		\end{aligned}
	\end{equation}
	Finally, combining \eqref{2-norm bound for epsilon matrix, proof of heat kernel matrix consistency Thm}, \eqref{first bound for L-infty eigenvector consistency, proof of heat kernel matrix consistency Thm}, and \eqref{lower bound of delta_k, proof of heat kernel matrix consistency Thm} yields
	\begin{equation}\nonumber
		\|\tilde{u}_k-\hat{u}_k\|_\infty\lesssim m^{\frac{1}{2r+1}}\left(\frac{1}{K}+\epsilon^{\frac{1}{4}}\right),
	\end{equation}
	which is the desired result. The proof is then finished.
\end{proof}

\section{Proof of Lemma \ref{Lemma: eigenfunction estimation}}\label{appendix: proof of eigenfunction estimation}

\begin{proof}
	For any $ 1\leq k\leq q $ and any $ 1\leq j\leq N $ we have
	\begin{equation}\nonumber
		\begin{aligned}
			|\phi_k(x_j)-\tilde{\phi}_k(j)|
			&=\left|\frac{1}{\sqrt{m\hat{\lambda}_k}}\sum_{i=1}^m\hat{u}_k(i)H_t(x_i,x_j)-\frac{1}{\sqrt{m\tilde{\lambda}_k}}\sum_{i=1}^m\tilde{u}_k(i)\tilde{H}_{t,K}(i,j)\right|\\
			&\leq\frac{1}{\sqrt{m\hat{\lambda}_k}}\left|\sum_{i=1}^m\hat{u}_k(i)H_t(x_i,x_j)-\sum_{i=1}^m\tilde{u}_k(i)\tilde{H}_{t,K}(i,j)\right|\\
			&\quad+\left|\frac{1}{\sqrt{m\hat{\lambda}_k}}-\frac{1}{\sqrt{m\tilde{\lambda}_k}}\right|\cdot\left|\sum_{i=1}^m\tilde{u}_k(i)\tilde{H}_{t,K}(i,j)\right|\\
			&\triangleq I_1+I_2.
		\end{aligned}
	\end{equation}
	Addressing $ I_1 $, \eqref{heat kernel matrix eigenvalue estimation, proof of heat kernel matrix consistency Thm} yields
	\begin{equation}\label{lower bound of empirical eigenvalue, proof of eigenfunction estimation lemma}
		\begin{aligned}
			\hat{\lambda}_k
			&\geq pe^{-\mu_kt}-O\left((\log m)^dm^{-\frac{1}{3}}\right)\\
			&\geq pe^{-\mu_qt}-O\left((\log m)^dm^{-\frac{1}{3}}\right)\\
			&\gtrsim m^{-\frac{1}{2r+1}}-O\left((\log m)^dm^{-\frac{1}{3}}\right)\\
			&\gtrsim m^{-\frac{1}{2r+1}}.
		\end{aligned}
	\end{equation}
	This leads us to conclude that:
	\begin{equation}\label{upper bound of sqrt mlambda_k, proof of eigenfunction estimation lemma}
		\frac{1}{\sqrt{m\hat{\lambda}_k}}\lesssim m^{-\frac{r}{2r+1}}.
	\end{equation}
	Furthermore, by leveraging \eqref{point-wise error of heat kernel estimator} and \eqref{L-infty eigenvector consistency of heat kernel matrix} we obtain
	\begin{equation}\nonumber
		\begin{aligned}
			\left|\sum_{i=1}^m\hat{u}_k(i)H_t(x_i,x_j)-\sum_{i=1}^m\tilde{u}_k(i)\tilde{H}_{t,K}(i,j)\right|
			&\leq\left|\sum_{i=1}^m\left(\hat{u}_k(i)-\tilde{u}_k(i)\right)H_t(x_i,x_j)\right|\\
			&\quad+\left|\sum_{i=1}^m\tilde{u}_k(i)\left(H_t(x_i,x_j)-\tilde{H}_{t,K}(i,j)\right)\right|\\
			&\leq m\kappa^2C_c m^{\frac{1}{2r+1}}\left(\frac{1}{K}+\epsilon^{\frac{1}{4}}\right)+m\left(\frac{C_a}{K}+C_b\epsilon^{\frac{1}{4}}\right)\\
			&\lesssim m^{1+\frac{1}{2r+1}}\left(\frac{1}{K}+\epsilon^{\frac{1}{4}}\right),
		\end{aligned}
	\end{equation}
	where the second inequality in this calculation is a result of the boundedness of $ H_t $ and the unit vector property of $ \tilde{u}_k $. Consequently, we arrive at the following bound:
	\begin{equation}\label{upper bound of I_1, proof of eigenfunction estimation lemma}
		I_1\lesssim m^{\frac{r+2}{2r+1}}\left(\frac{1}{K}+\epsilon^{\frac{1}{4}}\right).
	\end{equation}
	For bounding $ I_2 $, a similar approach yields
	\begin{equation}\nonumber
		\begin{aligned}
			\left|\sum_{i=1}^m\tilde{u}_k(i)\tilde{H}_{t,K}(i,j)\right|
			&\leq\sum_{i=1}^m\left|\tilde{H}_{t,K}(i,j)\right|\\
			&\leq\sum_{i=1}^m\left|\tilde{H}_{t,K}(i,j)-H_t(x_i,x_j)\right|+\sum_{i=1}^m\left|H_t(x_i,x_j)\right|\\
			&\leq m\left(\frac{C_a}{K}+C_b\epsilon^{\frac{1}{4}}\right)+m\kappa^2\lesssim m.
		\end{aligned}
	\end{equation}
	Additionally, we have
	\begin{equation}\nonumber
		\begin{aligned}
			\left|\frac{1}{\sqrt{m\hat{\lambda}_k}}-\frac{1}{\sqrt{m\tilde{\lambda}_k}}\right|
			&=m^{-\frac{1}{2}}\left|\frac{1}{\sqrt{\hat{\lambda}_k}}-\frac{1}{\sqrt{\tilde{\lambda}_k}}\right|\\
			&=m^{-\frac{1}{2}}\left|\frac{\sqrt{\tilde{\lambda}_k}-\sqrt{\hat{\lambda}_k}}{\sqrt{\hat{\lambda}_k}\sqrt{\tilde{\lambda}_k}}\right|\\
			&=m^{-\frac{1}{2}}\left|\frac{\tilde{\lambda}_k-\hat{\lambda}_k}{\sqrt{\hat{\lambda}_k}\sqrt{\tilde{\lambda}_k}\left(\sqrt{\tilde{\lambda}_k}+\sqrt{\hat{\lambda}_k}\right)}\right|.\\
		\end{aligned}
	\end{equation}
	Drawing from \eqref{eigenvalue consistency of heat kernel matrix} and \eqref{lower bound of empirical eigenvalue, proof of eigenfunction estimation lemma}, we obtain
	\begin{equation}\label{lower bound of hat lambda_k, proof of eigenfunction estimation lemma}
		\tilde{\lambda}_k\gtrsim m^{-\frac{1}{2r+1}}-\left(\frac{C_a}{K}+C_b\epsilon^{\frac{1}{4}}\right)\gtrsim m^{-\frac{1}{2r+1}},
	\end{equation}
	and hence
	\begin{equation}\nonumber
		\left|\frac{1}{\sqrt{m\hat{\lambda}_k}}-\frac{1}{\sqrt{m\tilde{\lambda}_k}}\right|\lesssim m^{-\frac{1}{2}}m^{\frac{3}{2}\cdot\frac{1}{2r+1}}\left(\frac{C_a}{K}+C_b\epsilon^{\frac{1}{4}}\right).
	\end{equation}
	This leads us to conclude
	\begin{equation}\label{upper bound for I_2, proof of eigenfunction estimation lemma}
		I_2\lesssim m^{\frac{1}{2}+\frac{3}{2}\cdot\frac{1}{2r+1}}\left(\frac{1}{K}+\epsilon^{\frac{1}{4}}\right)\lesssim m^{\frac{r+2}{2r+1}}\left(\frac{1}{K}+\epsilon^{\frac{1}{4}}\right).
	\end{equation}
	Finally, by merging \eqref{upper bound of I_1, proof of eigenfunction estimation lemma} and \eqref{upper bound for I_2, proof of eigenfunction estimation lemma}, we achieve the desired result
	\begin{equation}\nonumber
		|\phi_k(x_j)-\tilde{\phi}_k(j)|\lesssim m^{\frac{r+2}{2r+1}}\left(\frac{1}{K}+\epsilon^{\frac{1}{4}}\right).
	\end{equation}
	The proof is then finished.
\end{proof}

\section{Complementary Discussions of Proof for Lemma \ref{Lemma: upper bound for I_2 in truncation error analysis}}\label{appendix: complementary discussion of proof in main theorem}

\subsection{Modification of Lemma 10 in \cite{xia2024spectral}}

In this section, we present a detailed proof for the assertion mentioned as \eqref{cited lemma 10, proof of I_2 upper bound lemma}. It is important to note that, in the groundwork laid out for Lemma 10 in \cite{xia2024spectral}, the majority of supplementary lemmas are readily adaptable to a scenario where $ \beta=2 $, with only one exception, the Lemma 8. Consequently, our primary focus will be on expanding the scope of Lemma 8 in \cite{xia2024spectral} to incorporate $ \beta=2 $. 
In other words, we aim to establish the following result:
\begin{equation}\label{appdenix, Lemma 8 in Xia's paper}
	\left\|(T_\delta+\lambda I)^{\frac{1}{2}}h_\lambda(T_\delta) f_{P,\lambda}\right\|_{\mathcal{H}_t}\lesssim \lambda.
\end{equation}
Notice that
\begin{equation}\nonumber
	\left\|(T_\delta+\lambda I)^{\frac{1}{2}}h_\lambda(T_\delta) f_{P,\lambda}\right\|_{\mathcal{H}_t}
	=\left\|(T_\delta+\lambda I)^{\frac{1}{2}}h_\lambda(T_\delta) g_\lambda(T_\nu)I_\nu^*f^*\right\|_{\mathcal{H}_t}.
\end{equation}
Since Assumption \ref{Assumption: source condition on regression function} implies that $ f\in\mathcal{H}_t $, there exists another function $ g_0\in L^2(\nu) $ such that $ f^*=L_\nu g_0 $. Consequently,
\begin{equation}\nonumber
	\begin{aligned}
		\left\|(T_\delta+\lambda I)^{\frac{1}{2}}h_\lambda(T_\delta) f_{P,\lambda}\right\|_{\mathcal{H}_t}
		&=\left\|(T_\delta+\lambda I)^{\frac{1}{2}}h_\lambda(T_\delta) g_\lambda(T_\nu)I_\nu^*L_\nu g_0\right\|_{\mathcal{H}_t}\\
		&=\left\|(T_\delta+\lambda I)^{\frac{1}{2}}h_\lambda(T_\delta) g_\lambda(T_\nu)T_\nu I_\nu^*g_0\right\|_{\mathcal{H}_t}\\
		&\leq\left\|(T_\delta+\lambda I)^{\frac{1}{2}}h_\lambda(T_\delta) g_\lambda(T_\nu)T_\nu I_\nu^*\right\|_{\mathscr{B}(L^2,\mathcal{H}_t)} \cdot\|g_0\|_{L_2(\nu)}\\
		&=\left\|(T_\delta+\lambda I)^{\frac{1}{2}}h_\lambda(T_\delta) g_\lambda(T_\nu)T_\nu T_\nu^\frac{1}{2}\right\|_{\mathscr{B}(\mathcal{H}_t)} \cdot\|g_0\|_{L_2(\nu)},\\
	\end{aligned}
\end{equation}
where the last equation follows from the fact that $ I_\nu^* $ and $ T_\nu ^\frac{1}{2}$ are self-adjoint operators with same eigen-system. Then we further decompose the above expression as
\begin{equation}\nonumber
	\begin{aligned}
		\left\|(T_\delta+\lambda I)^{\frac{1}{2}}h_\lambda(T_\delta) g_\lambda(T_\nu)T_\nu T_\nu^\frac{1}{2}\right\|_{\mathscr{B}(\mathcal{H}_t)}
		&\leq \left\|(T_\delta+\lambda I)^{\frac{1}{2}}h_\lambda(T_\delta) (T_\delta+\lambda I)^{\frac{1}{2}}\right\|_{\mathscr{B}(\mathcal{H}_t)}\\
		&\quad\cdot\left\| (T_\delta+\lambda I)^{-\frac{1}{2}}(T_\nu+\lambda I)^\frac{1}{2}\right\|_{\mathscr{B}(\mathcal{H}_t)}\\
		&\quad\cdot\left\| (T_\nu+\lambda I)^{-\frac{1}{2}}T_\nu^\frac{1}{2}\right\|_{\mathscr{B}(\mathcal{H}_t)}\cdot\left\| g_\lambda(T_\nu)T_\nu\right\|_{\mathscr{B}(\mathcal{H}_t)}.
	\end{aligned}
\end{equation}
From \eqref{calculation on h_lambda, proof I_1 upper bound lemma, approxiamtion error analysis} and \eqref{upper bound of empirical norm between T_nu and T_delta, proof I_1 upper bound lemma} we have
\begin{equation}\nonumber
	\left\|(T_\delta+\lambda I)^{\frac{1}{2}}h_\lambda(T_\delta) (T_\delta+\lambda I)^{\frac{1}{2}}\right\|_{\mathscr{B}(\mathcal{H}_t)}=\left\|(T_\delta+\lambda I)h_\lambda(T_\delta)\right\|_{\mathscr{B}(\mathcal{H}_t)}\leq2\lambda,
\end{equation}
and
\begin{equation}\nonumber
	\left\| (T_\delta+\lambda I)^{-\frac{1}{2}}(T_\nu+\lambda I)^\frac{1}{2}\right\|_{\mathscr{B}(\mathcal{H}_t)}\leq\sqrt{3}.
\end{equation}
By a direct calculation, we derive that
\begin{equation}\nonumber
	\left\| (T_\nu+\lambda I)^{-\frac{1}{2}}T_\nu^\frac{1}{2}\right\|_{\mathscr{B}(\mathcal{H}_t)}=\sup\limits_{k\in\mathbb{N}}\left(\frac{e^{-\mu_kt}}{e^{-\mu_kt}+\lambda}\right)^\frac{1}{2}\leq1.
\end{equation}
Finally, property \eqref{property of regularization family} yields
\begin{equation}\nonumber
	\left\| g_\lambda(T_\nu)T_\nu\right\|_{\mathscr{B}(\mathcal{H}_t)}\leq 1.
\end{equation}
Our targeted \eqref{appdenix, Lemma 8 in Xia's paper} then follows from above pieces. The proof is then finished.
\qed

\subsection{Modification of Theorem 1 in \cite{xia2024spectral}}

In this section, we extend the original proof in \cite{xia2024spectral} for Theorem 1 to our setting with $ \beta=2 $ in details. We begin with \eqref{cited lemma 10, proof of I_2 upper bound lemma}:
\begin{equation}\nonumber
	\left\|(T_\nu+\lambda I)^{-\frac{1}{2}}\left(g_D-T_\delta f_{P,\lambda}\right)\right\|_{\mathcal{H}_t}^2\lesssim\frac{1}{m}\left(N_\nu(\lambda)+\lambda^{2-\alpha}+\frac{L^2_\lambda}{m\lambda^\alpha}\right)+\lambda^2.
\end{equation}
From Lemma 4 in \cite{xia2024spectral}, it holds
\begin{equation}\nonumber
	N_\nu(\lambda)\lesssim (\log \lambda^{-1})^\frac{d}{2}.
\end{equation}
Since 
\begin{equation}\nonumber
	\lambda \sim \left(\frac{(\log m)^\frac{d}{2}}{m}\right)^\frac{1}{2},
\end{equation}
we have $ \lambda\to0 $ as $ m\to\infty $. Therefore, 
\begin{equation}\nonumber
	N_\nu(\lambda)+\lambda^{2-\alpha}\lesssim(\log \lambda^{-1})^\frac{d}{2}.
\end{equation}
For the third term in the bracket, we have
\begin{equation}\nonumber
	\frac{L^2_\lambda}{m\lambda^\alpha}=\frac{1}{m\lambda^\alpha}\cdot\max\{M,\|f_{P,\lambda}-f^*\|_{\infty}\}.
\end{equation}
Exploiting Lemma 1 in \cite{xia2024spectral} with $ \gamma=\alpha $, $ \beta=2 $ combining with the embedding property $ \mathcal{H}_t^\alpha\hookrightarrow L^\infty(\nu) $, we derive that
\begin{equation}\nonumber
	\|f_{P,\lambda}-f^*\|_{\infty}\lesssim \lambda^{2-\alpha}.
\end{equation}
Hence 
\begin{equation}\nonumber
	\frac{L^2_\lambda}{m\lambda^\alpha}\lesssim\frac{1}{m\lambda^\alpha},
\end{equation}
which also tends to $ 0 $ as $ m\to\infty $. Therefore, 
\begin{equation}\nonumber
	\left\|(T_\nu+\lambda I)^{-\frac{1}{2}}\left(g_D-T_\delta f_{P,\lambda}\right)\right\|_{\mathcal{H}_t}^2\lesssim\frac{1}{m}(\log \lambda^{-1})^\frac{d}{2}+\lambda^2.
\end{equation}
A direct calculation shows 
\begin{equation}\nonumber
	\frac{1}{m}(\log \lambda^{-1})^\frac{d}{2}\sim \frac{\left(\frac{1}{2}(\log m -\frac{d}{2}\log\log m)\right)^\frac{d}{2}}{m}\lesssim \frac{(\log m)^\frac{d}{2}}{m}.
\end{equation}
Combining with the fact that
\begin{equation}\nonumber
	\lambda^2\sim\frac{(\log m)^\frac{d}{2}}{m},
\end{equation}
we finally derive 
\begin{equation}\label{upper bound in appdenix D}
	\left\|(T_\nu+\lambda I)^{-\frac{1}{2}}\left(g_D-T_\delta f_{P,\lambda}\right)\right\|_{\mathcal{H}_t}^2\lesssim\frac{(\log m)^\frac{d}{2}}{m}\sim\lambda^2.
\end{equation}
The upper bound delineated in \eqref{upper bound in appdenix D} is exactly \eqref{upper bound for the first term, proof of I_2 upper bound lemma}, which forms the essence of the proof for Theorem 1 in \cite{xia2024spectral}. Specifically, the synergy between \eqref{upper bound in appdenix D} and \eqref{appdenix, Lemma 8 in Xia's paper}, when integrated with the integral operator techniques adopted in Section \ref{section: proof of convergence result}, supports the \textit{estimation} error analysis, while the original \textit{approximation} error analysis can be extended to $ \beta=2 $ directly. Therefore, our proof ends with the same approach as in \cite{xia2024spectral}. The proof is then finished.
\qed

\section{Convergence Rate for Tikhonov Regularization}\label{appendix: convergence rate for Tikhonov regularization}

Before the detailed proof of the convergence rate for Tikhonov regularization, we note that with $ g_\lambda(t)=\frac{1}{\lambda+t} $, this filter function exhibits Lipschitz continuity, characterized by a Lipschitz constant of $ \lambda^{-2} $. This continuity property provides a solid foundation for establishing enhanced convergence results, particularly for the term $ I_1 $ as discussed in Lemma \ref{Lemma: upper bound for I_1 in approximation error analysis}.
Specifically, we have
\begin{equation}\nonumber
	\begin{aligned}
		I_1
		&=\left|\sum_{k=1}^q \left(g_\lambda(\lambda_k)-g_\lambda(\tilde{\lambda}_k)\right)\langle g_D,\phi_k\rangle_{\mathcal{H}_t}\phi_k(x_j)\right|\\
		&\leq\sum_{k=1}^q \left|g_\lambda(\lambda_k)-g_\lambda(\tilde{\lambda}_k)\right|\cdot\left|\langle g_D,\phi_k\rangle_{\mathcal{H}_t}\right|\cdot\left|\phi_k(x_j)\right|.
	\end{aligned}
\end{equation}
By Lipschitz property of $ g_\lambda $, we have
\begin{equation}\nonumber
	\left|g_\lambda(\lambda_k)-g_\lambda(\tilde{\lambda}_k)\right|\leq\lambda^{-2}|\lambda_k-\tilde{\lambda}_k|\lesssim \lambda^{-2}\left(\frac{1}{K}+\epsilon^{\frac{1}{4}}\right),
\end{equation}
and the remaining term of $ I_1 $ are already bounded in \eqref{upper bound of <g_D,phi_k>, proof of I_2 upper bound lemma, approximation error analysis} and \eqref{upper bound of phi_k(x_j), proof of I_2 upper bound lemma, approximation error}. Combining these pieces yields
\begin{equation}\nonumber
	\begin{aligned}
		I_1
		&\lesssim q\cdot\lambda^{-2}\left(\frac{1}{K}+\epsilon^{\frac{1}{4}}\right)\cdot m^\frac{r+1}{2r+1}\\
		&\lesssim m^{\frac{r+1}{2r+1}+1}\cdot\left(\frac{1}{K}+\epsilon^{\frac{1}{4}}\right).
	\end{aligned}
\end{equation}
Using this improved upper bound for $ I_1 $, we can then enhance the overall convergence rate for Tikhonov regularization to become
\begin{equation}\nonumber
	|\tilde{f}_{D,\lambda}(j)-f^*(x_j)|\lesssim  m^{-\frac{C_{low}}{C_{up}}\cdot\frac{1}{2}+\frac{\alpha}{4}}(\log m)^{\frac{1}{2}+\frac{d}{4}-\frac{\alpha d}{8}}+ m^{\frac{\max\{3,r+1\}}{2r+1}+1}(\log m)^\frac{d}{2}\cdot\left(\frac{1}{K}+\epsilon^{\frac{1}{4}}\right).
\end{equation}
Therefore, a similar approach as in section \ref{subsection: convergence analysis} will give the desired convergence rate \eqref{convergence rate of diffusion-based spectral algorithm, special case}. This completes the proof.
\qed

\end{document}